\documentclass{article}

\usepackage{microtype}
\usepackage{graphicx}
\usepackage{subcaption}
\usepackage{booktabs} % for professional tables

\usepackage{hyperref}
\usepackage{verbatim}
\usepackage[accepted]{arxiv_one_column}

\icmltitlerunning{Model-Targeted Poisoning Attacks with Provable Convergence}

\usepackage{amsthm}
\usepackage{todo}
\usepackage{amsmath,amssymb, mathtools}
\usepackage{multirow}

\newcommand{\cD}{\mathcal{D}}
\newcommand{\cX}{\mathcal{X}}
\newcommand{\cY}{\mathcal{Y}}
\newcommand{\cL}{\mathcal{L}}
\newcommand{\Regret}{\mathsf{Regret}}
\newcommand{\set}[1]{\{#1\}}

\newcommand{\MNIST}{{MNIST~1\nobreakdash--7}}

\newcommand{\State}{\STATE}
\newcommand{\While}[1]{\WHILE{#1}}
\newcommand{\EndWhile}[1]{\ENDWHILE{#1}}
  \newcommand{\Return}[1]{\textbf{return} #1}

\theoremstyle{definition} % amsthm only
\newtheorem{theorem}{Theorem}
\newtheorem{definition}{Definition}
\newtheorem{remark}{Remark}
\newtheorem{corollary}{Corollary}
\newtheorem{lemma}{Lemma}
\newtheorem{setting}{Setting}
\newcommand\shortsection[1]{\vspace{6pt}{\noindent\bf #1.}}

\newcommand\shortsectionnp[1]{\vspace{6pt}{\noindent\bf #1}}
\usepackage{multirow}
\usepackage{amsmath,amsfonts,bm}

\def\eqref#1{equation~\ref{#1}}
\def\1{\bm{1}}

\DeclareMathAlphabet{\mathsfit}{\encodingdefault}{\sfdefault}{m}{sl}
\SetMathAlphabet{\mathsfit}{bold}{\encodingdefault}{\sfdefault}{bx}{n}

\newcommand{\R}{\mathbb{R}}

\DeclareMathOperator*{\argmax}{arg\,max}
\DeclareMathOperator*{\argmin}{arg\,min}

\DeclareMathOperator{\sign}{sign}

\begin{document}

% Suya modified to onecolumn
\onecolumn{
\icmltitle{Model-Targeted Poisoning Attacks with Provable Convergence}
}


\icmlsetsymbol{equal}{*}

\begin{icmlauthorlist}
%\icmlauthor{Fnu Suya}{equal,va}
\icmlauthor{Fnu Suya}{va}
\icmlauthor{Saeed Mahloujifar}{pri}
\icmlauthor{Anshuman Suri}{va}
\icmlauthor{David Evans}{va}
\icmlauthor{Yuan Tian}{va}
% \icmlauthor{Tateu H.~Yasehe}{ed,to,goo}
% \icmlauthor{Aaoeu Iasoh}{goo}
% \icmlauthor{Buiui Eueu}{ed}
% \icmlauthor{Aeuia Zzzz}{ed}
% \icmlauthor{Bieea C.~Yyyy}{to,goo}
% \icmlauthor{Teoau Xxxx}{ed}
% \icmlauthor{Eee Pppp}{ed}
\end{icmlauthorlist}

\icmlaffiliation{va}{University of Virginia}
\icmlaffiliation{pri}{Princeton University}
% \icmlaffiliation{ed}{School of Computation, University of Edenborrow, Edenborrow, United Kingdom}

% \icmlcorrespondingauthor{Cieua Vvvvv}{c.vvvvv@googol.com}
% \icmlcorrespondingauthor{Eee Pppp}{ep@eden.co.uk}
\icmlcorrespondingauthor{Fnu Suya}{suya@virginia.edu}
\icmlcorrespondingauthor{Saeed Mahloujifar}{sfar@princeton.edu}

\icmlkeywords{Machine Learning, ICML}

\vskip 0.3in
%]


\printAffiliationsAndNotice{}  % leave blank if no need to mention equal contribution

\begin{abstract}
In a poisoning attack, an adversary with control over a small fraction of the training data attempts to select that data in a way that induces a corrupted model that misbehaves in favor of the adversary. We consider poisoning attacks against convex machine learning models and propose an efficient poisoning attack designed to induce a specified model. Unlike previous model-targeted poisoning attacks, our attack comes with provable convergence to {\it any} attainable target classifier. The distance from the induced classifier to the target classifier is inversely proportional to the square root of the number of poisoning points. We also provide a lower bound on the minimum number of poisoning points needed to achieve a given target classifier. Our method uses online convex optimization, so finds poisoning points incrementally. This provides more flexibility than previous attacks which require a priori assumption about the number of poisoning points.
Our attack is the first model-targeted poisoning attack that provides provable convergence for convex models, and in our experiments, it either exceeds or matches state-of-the-art attacks in terms of attack success rate and distance to the target model. 
\end{abstract}
\section{Introduction}
%Machine learning models are widely deployed in different domains to automate the tasks. However, 
Machine learning often requires a large amount of labeled training data, which is collected from untrusted sources. A typical application is email spam filtering, where a spam detector filters out spam messages based on features (e.g., presence of certain words) and periodically updates the model based on newly received emails labeled by users. In such a setting, spammers can generate spam messages that inject benign words likely to occur in legitimate emails, and when models are trained on these spam messages, the filtering accuracy drops significantly~\citep{nelson2008exploiting,huang2011adversarial}. Such attacks are known as \textit{poisoning attacks}, and a training process that uses labels or data from untrusted sources is potentially vulnerable to them.

Poisoning attacks can be categorized as {\it objective-driven} %~\citep{biggio2011support,biggio2012poisoning,xiao2012adversarial,steinhardt2017certified,shafahi2018poison,zhu2019transferable,jagielski2019subpop} 
or {\it model-targeted}. %~\citep{mei2015using,koh2018stronger} 
%depending on whether a target model is considered in the attack process.
Objective-driven poisoning attacks have a specified attacker objective (such as reducing the overall accuracy of the victim model) and aim to induce a model that maximizes that objective. Model-targeted attacks have a specific target model in mind and aim to induce a victim model as close as possible to that target model. 
%\shortsection{Objective-Driven Poisoning Attacks} 
Objective-driven attacks are most commonly studied in the existing literature, and indeed, it is natural to think about attacks in terms of the goals of an adversary. We argue, though, that breaking poisoning attacks into the two steps of first finding a model to target and then selecting poisoning points to induce that model has significant advantages. This view leads to improvements in our understanding of poisoning attacks and simplifies the task of designing effective attacks for a variety of different objectives. Importantly, it can also lead to more effective poisoning attacks. 
%{\color{green} I want to say it leads to more effective attacks, but I'm not sure if we can support that?}

Attacker objectives for realistic attacks are diverse, and designing a unified and effective attack strategy for different attacker objectives is hard.  Most work has considered one of two extremal attacker objectives: \emph{indiscriminate} attacks, where the adversary's goal is simply to decrease the overall accuracy of the model~\citep{biggio2012poisoning,xiao2012adversarial,mei2015using,steinhardt2017certified,koh2018stronger}; and \emph{instance-targeted} attacks, where the goal is to induce a classifier that misclassifies a particular known input~\citep{shafahi2018poison,zhu2019transferable,koh2017understanding,geiping2020witches,huang2020metapoison}.
Recently,~\citet{jagielski2019subpop} introduced a more realistic attacker objective known as a \emph{subpopulation} attack, where the goal is to increase the error rate or obtain a particular output for a defined subpopulation of the data distribution. Gradient-based local optimization is most commonly used to construct poisoning points for a particular attacker objective~\citep{biggio2012poisoning,xiao2012adversarial,mei2015using,koh2017understanding,shafahi2018poison,zhu2019transferable}. These attacks can be modified to fit other attacker objectives, but since they are based on local optimization techniques they often get stuck into bad local optima and fail to find effective sets of poisoning points~\citep{steinhardt2017certified,koh2018stronger}. To circumvent the issue of local optima, ~\citet{steinhardt2017certified} formulate the indiscriminate attack as a min-max optimization and solve it efficiently using online convex optimization techniques. However, this attack only applies to the indiscriminate setting. 

In contrast, \emph{model-targeted attacks} incorporate the attacker objective into a target model and hence, the target model can reflect any attacker objective. Thus, the same model-targeted attack methods can be directly applied to a range of indiscriminate and subpopulation attacks just by finding a suitable target model. 
%{\color{red} One additional benefit of model-targeted poisoning is, stronger adversaries may generate a target model using other poisoning attacks with an attack objective, and induce similar target models that satisfy the objective using fewer poisoning points with model-targeted poisoning.}
\citet{mei2015using} first introduced a target model into a poisoning attack and then utilized the KKT condition to transform the problem into a tractable form, but their attack is still based on gradient-based local optimization techniques and suffers from bad local optima~\citep{steinhardt2017certified,koh2018stronger}. \citet{koh2018stronger} proposed the KKT attack, which converts the complicated bi-level optimization into a simple convex optimization problem utilizing the KKT condition and the  Carath\'eodory number of the set of scaled gradients, avoiding the local optima issues. However, their attack only works for margin-based losses and does not provide any guarantee on the number of poisoning points required to converge to the target classifier. Additionally, these attacks require knowing the number of poisoning points before running the attack, which is often not available in practical applications. 
% I don't think we need this - and it doesn't really work as an example - if the adversary doesn't know if a given point will get through or not, that's a problem for us also
%For example, focused adversaries aim to block certain legitimate emails (e.g., for corporates, these emails can be messages from their competitors) from the victim's inbox~\citep{nelson2008exploiting} and it is hard to know the number of poisoning points needed to achieve this attacker goal (with a proper target model) beforehand. 

We study poisoning attacks on simple convex models because poisoning attacks are still not fully understood in these settings. In addition, many important industrial applications continue to rely on simple models due to their easiness in model debugging, low computational cost, and for many applications, such simple convex models also have either comparable or better performances than the complex deep neural networks~\citep{dacrema2019we,tramer2020differentially}.

%and practitioners might lack motivation to use complex models since the overhead may not surplus the benefit.
%In particular, we find both theoretical and empirical bounds on the sufficient (and necessary) number of poisoning points to get close to a specific target classifier~\footnote{Similar to previous works, in this paper, we focus on designing a model-targeted attack that works for any attainable target model and leave the exploration of finding better target classifiers as the future work.}.

\shortsection{Contributions}
Our main contribution is a principled and general model-targeted poisoning method, along with proof that the model it induces converges to the target model. 
%{\color{green} I don't think this (next sentence, now commented out) is correct - it needs both the target model, and the clean training set (and assumes knowledge of the victim's architecture and training method) - I'm also bothered by the "set" - we emphasize online nature, so the order of the points matters - but the way Algo 1 is written, it is a set. I don't think it should be (either here or there) - isn't it outputting an ordered list of points?} %Our poisoning method takes as input a target model, and produces a set of poisoning points. 
In this work, we focus on effectiveness in inducing a given target model and defer to future work a full exploration of how to select good target models for particular attacker objectives. Our focus also aligns with the goal of previous model-targeted poisoning attacks~\citep{koh2018stronger,mei2015using}.

We prove, for settings where the loss function is convex and proper regularization is adopted in training, that the model induced by training on the original training data with these points added, converges to the target classifier as the number of poison points increases (Theorem~\ref{theorem:convergence_main}). Previous model-targeted attacks lack such convergence guarantees. We then prove a lower bound on the minimum number of poisoning points needed to reach the target model (Theorem~\ref{theorem:lower_bound}).
% remove this here - I think adding the settings at beginning of paragraph is enough
%, given that the loss function for empirical risk minimization is convex. 
{Such a lower bound can be used to estimate the optimality of model-targeted poisoning attacks and also indicate the intrinsic hardness of attacking different targets.} 
Our attack applies to incremental poisoning scenarios as it works in an online fashion to find effective poisoning points without a predetermined poisoning rate. Previous model-targeted attacks assume a priori number of poisoning points. 

We evaluate our attack and compare it to the state-of-the-art model-targeted attack~\citep{koh2018stronger}. We evaluate the convergence of our attack to the target model and find that for the same number of poisoning points, our attack is able to induce models closer to the target model, for all target classifiers we tried. The success rate of our attack exceeds that of the state-of-the-art attack in subpopulation attack scenarios and is comparable for indiscriminate attacks (Section~\ref{sec:experiments}).  
%\dnote{hard to reconcile the last two sentence - first says our attack outperforms, second says it is comparable performance; need to be more clear what we are claiming here (and not wordsmith things in a way that will feel deceptive to readers - be up front about the limitations and benefits}

% \dnote{I don't think we should need this paragraph - the intro and contributions should be enough to make this clear, and most of what is here now is redundant} 
% \shortsection{Comparison to the KKT attack}
% Our model-targeted attack has several advantages over the state-of-the-art model-targeted attack, the KKT attack~\citep{koh2018stronger}. Our attack comes with a provable guarantee that it converges to the target classifier, and also provides a certified lower bound on number of poisoning points needed to reach a given target model. The KKT attack does not provide either of these guarantees. The convergence property of our attack is applicable to the broader family of convex loss functions while the validity of the KKT attack is limited to margin based loss functions. From the practical side, our attack works in an online fashion without needing to know in advance the number of poisoning points available, and hence is very efficient in incremental poisoning scenario. 
% Empirically, we also show that our attack has better performance than the KKT attack in the subpopulation settings, and comparable performance in
% indiscriminate settings.

\section{Problem Setup}

The poisoning attack proposed in this paper applies to multi-class prediction tasks or regression problems (by treating the response variable as an additional data feature), but for simplicity of presentation we consider a binary prediction task, $h:\mathcal{X}\rightarrow \mathcal{Y}$, where $X\subseteq \R^{d}$ and $\mathcal{Y} = \{+1,-1\}$. The prediction model $h$ is characterized by parameters $\theta \in \Theta\subseteq \R^{d}$. We define the non-negative convex loss on an individual point, $(x,y)$, as $l(\theta;x,y)$ (e.g., hinge loss for SVM model). We also define the empirical loss over a set of points $A$ as $L(\theta;A)=\sum_{(x,y)\in A}l(\theta;x,y)$. %We assume a true data distribution, $p^{*}$, over $\mathcal{X}\times\mathcal{Y}$. \dnote{is there a reason to use $p$ for the data distribution? confusing to me to use this here as well as in the subscripts, }

We adopt the game-theoretic formalization of the poisoning attack process from \citet{steinhardt2017certified} to describe our model-targeted attack scenario:
\begin{enumerate}
%\begin{itemize}
    \item $N$ data points are drawn uniformly at random from the true data distribution over $\mathcal{X}\times\mathcal{Y}$ and form the clean training set, $\cD_c$.
    \item The adversary, with knowledge of $\cD_c$, the model training process and the model space $\Theta$, generates a target classifier $\theta_{p}\in \Theta$ that satisfies the attack goal. 
    
    \item The adversary produces a set of poisoning points, $\cD_p$, with the knowledge of $\cD_c$, model training process, $\Theta$ and $\theta_p$. 
    \item Model builder trains the model on $\cD_c\cup\cD_p$ and produces a classifier, $\theta_{\mathit{atk}}$. 
  %  \end{itemize}
\end{enumerate}

The adversary's goal is that the induced classifier, $\theta_{\mathit{atk}}$, is close to the desired target classifier, $\theta_p$ (Section~\ref{sec:converge_proof} discusses how this distance is measured). Step 2 corresponds to the target classifier generation process. Our attack works for any target classifier, and in the paper we do not focus on the question of how to find the best target classifier to achieve a particular adversarial goal but simply adopt the heuristic target classifier generation process from \citet{koh2018stronger}. 
%\dnote{this is confusing to me - either we are focused on presenting a model-agnostic attack; or, we are considering the problem of how to select a target classifier (I reworded the first clause to match the focus of the paper until here, but doesn't match with the second clause). We need to decide if the paper is about how to select the target classifier or not; if it is about this, the introduction and contributions should make this clear, or at least mention it - otherwise, it is difficult to have the first mention that we are considering the research problem of how to select a good target classifier. If the latter, then it is hard to see what we would only do this for the indiscriminate case, which we argue isn't so important or interesting, and not for the discriminate case. I think the former makes more sense, so the next two sentences shouldn't be here (but should be mentioned later, and not emphasized much, as we should make it clear that this is an opportunity for further improvements and area for future work, but not the focus of this paper.}
Step 3 corresponds to our model-targeted poisoning attack and is also the main contribution of the paper.  

We assume the model builder trains a model through empirical risk minimization (ERM) and the training process details
%\dnote{here, we are showing the training objective - are we assuming more about the training process is known to the attackers (e.g., the optimizer being used, hyperparameters, etc.)}
are known to the attacker:
\begin{equation}
       \theta_c = \argmin_{\theta\in \Theta}\frac{1}{|\cD_c|} L(\theta;\cD_c) + C_{R} \cdot R(\theta)
 \label{eq:erm}
 \end{equation}

where $R(\theta)$ is the nonnegative regularization function (e.g., $\frac{1}{2}\|\theta\|^{2}_2$ for SVM model). 

\shortsection{Threat Model} 
We assume an adversary with full knowledge of training data, model space, and training process. Although this may be unrealistic for many scenarios, this setting allows us to focus on a particular aspect of poisoning attacks and is the setting used in many prior works~\citep{biggio2011support,mei2015using,steinhardt2017certified,koh2018stronger,shafahi2018poison}. 
We assume an addition-only attack where the attacker only adds poisoning points into the clean training set. A stronger attacker may be able to modify or remove existing points, but this typically requires administrative access to the system. The added points are unconstrained, other than being value elements of the input space. They can have arbitrary features and labels, which enables us to perform the worst-case analysis on the robustness of models against addition-only poisoning attacks. Although some previous works also allow arbitrary selection of the poisoning points~\citep{biggio2011support,mei2015using,steinhardt2017certified,koh2018stronger}, others put different restrictions on the poisoning points.  A clean-label attack assumes adversaries can only perturb the features of the data, but the label is given by an oracle labeler~\citep{koh2017understanding,shafahi2018poison,zhu2019transferable,huang2020metapoison}. In label-flipping attacks, adversaries are only allowed to change the labels~\citep{biggio2011support,xiao2012adversarial,xiao2015support,jagielski2019subpop}. These restricted attacks are weaker than the poisoning attacks without restrictions~\citep{koh2018stronger,hong2020effectiveness}. %\dnote{do we really know they are "much weaker"? I'm not sure what is in the papers we are citing here, but this is a strong claim} 

\section{Related Work}
\label{sec:related}
%\paragraph{Gradient based Attacks} 
The most commonly used poisoning strategy is gradient-based attack. Gradient-based attacks iteratively modify a candidate poisoning point $(\hat{x},\hat{y})$ in the set $\cD_p$ based on the test loss defined on $\hat{x}$ (keeping $\hat{y}$ fixed). This kind of attack was first studied on SVM models~\citep{biggio2012poisoning,demontis2019adversarial}, and later extended to linear and logistic regression~\citep{mei2015using,demontis2019adversarial}, and recently to larger neural network models~\citep{koh2017understanding,shafahi2018poison,zhu2019transferable,huang2020metapoison}. \citet{jagielski2018manipulating} studied gradient attacks and principled defenses on linear regression tasks. 
%Their work studies linear regression while in this paper, we mainly focus on binary classification, although our attack can also be extended to regression tasks. More importantly, our attack aims to induce a target model by generating poisoning points while \citet{jagielski2018manipulating}'s attack tries to increase the Mean Squared Error of the linear regression task with a fixed poisoning budget. 
In addition to classification and regression tasks, gradient-based poisoning attacks are also applied to topic modeling~\citep{mei2015security}, collaborative filtering~\citep{li2016data} and algorithmic fairness~\citep{solans2020poisoning}. 

Besides the gradient-based attacks, researchers also utilize generative adversarial networks to craft poisoning points efficiently for larger neural networks, but with limited effectiveness~\citep{yang2017generative,munoz2019poisoning}. The strongest attacks so far are the KKT attack~\citep{koh2018stronger} and the min-max attack~\citep{steinhardt2017certified,koh2018stronger}. {However, the KKT attack cannot scale well for multi-class classification and is limited to margin-based losses~\citep{koh2018stronger}}. The min-max attack only works for indiscriminate attack setting, but additionally provides a certificate on worst-case test loss for a fixed number of poisoning points.
%and our theoretical lower bound on number of poisons can treated as its dual problem. 
We are also inspired by~\citet{steinhardt2017certified} to adopt online convex optimization to instantiate our model-targeted attack, but now dealing with a more general attack scenario. We also distinguish ourselves from the poisoning attack against online learning~\citep{wang2018data}. The attack against online learning considers a setting where training data arrives in a streaming manner while we consider the offline setting with training data being fixed. Another line of work studies ``targeted'' poisoning attacks where an adversary guarantees to increase the probability of an arbitrary ``bad'' property~\citep{mahloujifar2019curse, mahloujifar2017learning, mahloujifar2019universal}, as long as that property has some non-negligible chance of naturally happening. These attacks cannot be applied in model-targeted setting as the probability of naturally producing a specific target model is often negligible. 
% In the model targeted setting one could try to achieve any bad property as long as there is a target model that has that bad property. So, our attack model is more general and can achieve its goal more efficiently.
% We are distinct from these works in that our attack tries to enforce a specific classifier, and not a bad property. 
% Our attack could achieve its goal much more efficiently than the the attacks in this line of work.
%, while still having provable guarantees
%  \footnote{Although the attacks of~\citep{mahloujifar2019curse, mahloujifar2019universal} are shown to be polynomial time~\citep{mahloujifar2018can,etesami2020computational}, they are still far from being practical.}. 
%\dnote{I think we need to be more explicit what "provable guarantee means - it matters what you are able to prove about them (both here and in the footnote). There are plenty of (vacuous) things I can provably guarantee about any attack.}
Related to our Theorem~\ref{theorem:lower_bound},~\citet{ma2019data} also derived a lower bound on the number of poisoning points (to induce a target model), but their lower bound only applies when differential privacy is deployed during the model training process (and hence hurts model utility), which is different from our problem setting.   

%Defenses against data poisoning attacks also gained attention from the community. Techniques from robust statistics are commonly adopted and has been applied to linear regression~\citep{liu2017robust,jagielski2018manipulating}, classification problems~\citep{rubinstein2009antidote} or both~\citep{diakonikolas2019sever}. Sample other defenses are based on game theoretic formulation~\citep{zhang2017game}, randomized smoothing~\citep{rosenfeld2020certified} and differential privacy~\citep{ma2019data}. We note that Ma et al.~\citep{ma2019data} also derived a lower bound on number of required poisoning points, which is related to our Theorem~\ref{theorem:lower_bound}. However, their lower bound only applies when differential privacy is deployed during the model training process (and hence hurts model utility), which is different from our problem setting. 

% \todo{add something about bagging defenses: \url{https://arxiv.org/abs/2008.04495}}

\section{Poisoning Attack with a Target Model}
\label{sec:model-targeted_poison}
Our new poisoning attack determines a target model and selects poisoning points to achieve that target model. The target model generation is not our focus and we adopt the heuristic approach proposed by~\citet{koh2018stronger}. For the new poisoning attack, we first show how the algorithm generates the poisoning points (Section~\ref{sec:online_lr}). Then, we prove that the generated poisoning points, once added to the clean data, can produce a classifier that asymptotically converges to the target classifier (Section~\ref{sec:converge_proof}).

\subsection{Model-Targeted Poisoning with Online Learning}
\label{sec:online_lr}
The main idea of our model-targeted poisoning attack, as outlined in Algorithm~\ref{algorithm}, is to sequentially add a point into the training set that has maximum loss-difference between the intermediate model obtained so far and the target model. By training models on the updated training set we actually minimize the gap in the loss of the intermediate classifier and the target classifier. Repeating the process then eventually generates classifiers that have similar loss distribution as the target classifier. We show in Section~\ref{sec:converge_proof} why similar loss distribution implies convergence.

%, to eventually make the target classifier $\theta_{p}$ attainable.
%and the lower bound is updated by checking the theoretical lower bound (given by Theorem ~\ref{theorem:min_point}) of different pairs of model and choosing their maximum.
\begin{center}
\begin{minipage}{0.45\textwidth}
\vspace{-1.5em}
\begin{algorithm}[H]
    \caption{$\mathsf{ModelTargeted Poisoning}$\label{algorithm}}
    \textbf{Input:} $\cD_c$, the loss functions ($L$ and $l$), $\theta_p$ \\
    \textbf{Output:} $\cD_p$
    \begin{algorithmic}[1]
        \State $\cD_p=\emptyset$
        %\For{t $\gets 1$ to $T$}
        \While{stop criteria not met}
        \State $\theta_{t}=\argmin L(\theta;\cD_c \cup \cD_p)$
        \State $(x^{*},y^{*})=\argmax_{\cX\times\cY} l(\theta_{t}; x,y) - l(\theta_p; x, y)$\label{line:max_loss}
        % \If{$L(\theta_{t};\mathcal{D}_{c})<L(\theta_{p};\mathcal{D}_{c})$} 
        %     \State $c_{\theta_{p}}$ = $\frac{L(\theta_{p};\mathcal{D}_{c})-L(\theta_{t};\mathcal{D}_{c})}{l(\theta_{t};x^{*},y^{*})-l(\theta_{p};x^{*},y^{*}))}$ 
        %     \If {$c_{\theta_{p}}^{*} < c_{\theta_{p}}$}
        %         \State $c_{\theta_{p}}^{*}$=$c_{\theta_{p}}$
        %     \EndIf
        % \EndIf
        \State $\cD_p=\cD_p\cup \set{(x^*,y^*)}$ 
        %\EndFor\\
        \EndWhile\\
        \Return $\cD_p$
    \end{algorithmic}
\end{algorithm}
\end{minipage}
\end{center}

Algorithm 1 requires the input of clean training set $\cD_c$, the Loss function ($L$ for a set of points and $l$ for individual point), and the target model $\theta_p$. The output from Algorithm $1$ will be the set of poisoning points $\cD_p$. The algorithm is simple: first, adversaries train the intermediate model $\theta_t$ on the mixture of clean and poisoning points $\cD_c\cup\cD_p$ with $\cD_p$ an empty set in the first iteration (Line 3). The adversary then searches for the point that maximizes the loss difference between $\theta_t$ and $\theta_p$ (Line 4). After the point of maximum loss difference is found, it is added to the poisoning set $\cD_p$ (Line 5). The whole process repeats until the stop condition is satisfied in Line 2. The stop condition is flexible and it can take various forms: 1) adversary has a budget $T$ on the number of poisoning points, and the algorithm halts when the algorithm runs for $T$ iterations; 2) the intermediate classifier $\theta_t$ is closer to the target classifier (than a preset threshold $\epsilon$) in terms of the maximum loss difference, and more details regarding this distance metric will be introduced in Section~\ref{sec:converge_proof}; 3) adversary has some requirement on the accuracy and the algorithm terminates when $\theta_t$ satisfies the accuracy requirement. Since we focus on producing a classifier close to the target model, we adopt the second stop criterion that measures the distance with respect to the maximum loss difference, and report results based on this criterion in Section~\ref{sec:experiments}. 

A nice property of Algorithm~\ref{algorithm} is that the classifier $\theta_{atk}$ trained on $\cD_c\cup\cD_p$ is close to the target model $\theta_p$ and asymptotically converges to $\theta_p$. Details of the convergence will be shown in the next section. The algorithm may appear to be slow, particularly for larger models due to the requirement of repeatedly training a model in line 3. However, this is not an issue. First, as will be shown in the next section, the algorithm is an online optimization process and line 3 corresponds to solving the online optimization problem exactly. However, people often use the very efficient online gradient descent method to approximately solve the problem and its asymptotic performance is the same~\citep{shalev2012online}. Second, if we solve the optimization problem exactly, we can add multiple copies of $(x^*,y^*)$ into $\cD_p$ each time. This reduces the overall iteration number, and hence reduces the number of times retraining models. The proof of convergence will be similar. For simplicity in interpreting the results, we do not use this in our experiments and add only one copy of $(x^*,y^*)$ each iteration. However, we also tested the performance by adding two copies of $(x^*,y^*)$ and find that the attack results are nearly the same while the efficiency is improved significantly. For example, for experiments on \MNIST\ dataset, by adding 2 copies of points, with the same number of poisoning points, the attack success rate decreases at most by 0.7\% while the execution time is reduced approximately by half.

\subsection{Convergence of Our Poisoning Attack}
\label{sec:converge_proof}
Before proving the convergence of Algorithm~\ref{algorithm}, we need to measure the distance of the model $\theta_{atk}$ trained on $\cD_c\cup\cD_p$ to the target model $\theta_p$. First, we define a general closeness measure based on their prediction performance which we will use to state our convergence theorem:

\begin{definition}[Loss-based distance and $\epsilon$-close]
\label{def:eps_close}
For two models $\theta_1$ and $\theta_2$, a space $\cX\times \cY$ and a loss $l(\theta; x,y)$, we define \emph{loss-based distance} $D_{l,\cX,\cY}\colon \Theta\times \Theta\to R$ as 
$$D_{l,\cX,\cY}(\theta_1,\theta_2) = \max_{(x,y) \in \cX\times \cY} l(\theta_1; x, y) - l(\theta_2; x,y),$$

and we say model $\theta_1$ is \emph{$\epsilon$-close} to model $\theta_2$ when the loss-based distance from $\theta_1$ to $\theta_2$ is upper bounded by $\epsilon$.
\end{definition}

\paragraph{Measuring model distance}  We use loss-based distance to capture the ``behavioral'' distance between two models. Namely, if $\theta_1$ is $\epsilon$-close (as measured by loss-based distance) to $\theta_2$ and vice versa, then $\theta_1$ and $\theta_2$ would have an almost equal loss on all the points, meaning that they have almost the same behavior across all the space. Note that our general definition of loss-based distance does not have the symmetry property of metrics and hence is not a metric. However, it has some other properties of metrics in the space of attainable\footnote{Attainable models are models that can be obtained by training on some data from the input space. See formal definition in Appendix~\ref{sec:proofs}.} models. For example, if some model  $\theta$ is attainable using ERM, no model could have a negative distance to it. To further show the value of this distance notion, in Appendix \ref{sec:closeness} we demonstrate an $O(\epsilon)$ upper bound on the $\ell_1$-norm of difference between two models that are $\epsilon$-close with respect to loss-based distance for the special case of Hinge loss. For Hinge loss, it also satisfies the {\it bi-directional closeness}, that is if $\theta_1$ is $\epsilon$-close to $\theta_2$, then $\theta_2$ is $O(\epsilon)$-close to $\theta$ (details can be found in Corollary~\ref{cor:bidirection-close}), and the proof details can be found in Appendix~\ref{sec:closeness}. In the rest of the paper, we will use the terms $\epsilon$-close or $\epsilon$-closeness to denote that a model is $\epsilon$ away from another model based on the loss-based distance. 

Our convergence theorem uses the loss-based distance to establish that the attack of Algorithm \ref{algorithm} produces model that converges to the target classifier:

\begin{theorem}
\label{theorem:convergence_main}
    After at most $T$ steps, Algorithm~\ref{algorithm} will produce the poisoning set $\cD_p$ and the classifier trained on $\cD_c\cup\cD_p$ is $\epsilon$-close to $\theta_{p}$, with respect to loss-based distance, $D_{l, \cX, \cY}$, for
    \begin{align*}
        \epsilon =\frac{\alpha(T) + L(\theta_p; D_c) - L(\theta_c;D_c)}{T\cdot \gamma}
    \end{align*}
    %\label{theorem:convergence}
    where, $\gamma$ is a constant for a given $\theta_p$ and classification task, and $\alpha(T)$ is the regret of the online algorithm when the loss function used for training is convex. 
\end{theorem}

\begin{remark}
\label{remark:convergence}
Online learning algorithms with sublinear regret bound can be applied to show the convergence. Here, we adopt results from ~\citet{mcmahan2017survey}. Specifically, $\alpha(T)$ is in the order of $O(\log T)$) and we have $\epsilon\leq O(\frac{\log T}{T})$ when the loss function is additionally Lipschitz continuous and the regularizer $R(\theta)$ is strongly convex, and $\epsilon \rightarrow 0$ when $T\rightarrow +\infty$. $\alpha(T)$ is also in the order of $O(\log T)$ when the loss function used for training is strongly convex and the regularizer is convex. 
\end{remark}

\shortsection{Proof idea} The full proof of Theorem~\ref{theorem:convergence_main} is in Appendix \ref{sec:proofs}. Here, we only summarize the high-level proof idea. The key idea is to frame the poisoning problem as an online learning problem. In this formulation, each step of the online learning problem corresponds to the $i^{th}$ poison point $(x_i,y_i)$. In particular, the loss function at iteration $i$ of the online learning problem is set to $l(\cdot; x_i, y_i)$. Then, we show that by defining the parameters of the online learning problem carefully, the output of the follow-the-leader (FTL) algorithm~\citep{shalev2012online} at iteration $i$ is a model that is identical to training a model on a dataset consisting of the clean points and the first $i-1$ poisoning points. On the other hand, the way the poisoning points are selected, we can show that at the $i^{th}$ iteration the maximum loss difference between the target model and the best induced model so far would be smaller than the regret of the FTL algorithm divided by the number of poisoning points. The convergence bound of Theorem~\ref{theorem:convergence_main} boils down to regret analysis of the algorithm based on the loss function. Since we are assuming the loss function is convex with a strongly convex regularizer (or a strongly convex loss function with a convex regularizer), we can show that the regret is bounded by $O(\log T)$ and hence the loss distance between the induced model and the target model converges to 0.

\paragraph{Implications of Theorem \ref{theorem:convergence_main}}  The theorem says that the loss-based distance of the model trained on $\cD_c\cup\cD_p$ to the target model correlates to the loss difference between the target model and the clean model $\theta_c$ (trained on $\cD_c$) on $\cD_c$, and correlates inversely with the number of poisoning points. Therefore, it implies 1) if the target classifier $\theta_p$ has a lower loss on $\cD_c$, then it is easier to achieve the target model, and 2) with more poisoning points, we get closer to the target classifier and our attack will be more effective. The theorem also justifies the motivation behind the heuristic method in~\citet{koh2018stronger} to select a target classifier with a lower loss on clean data. For the indiscriminate attack scenario, we also improve the heuristic approach by adaptively updating the model and producing target classifiers with a much lower loss on the clean set. This helps to empirically validate our theorem. Details of the original and improved heuristic approach and relevant experiments are in Appendix~\ref{ssec:bettertarget}. 
%Based on the Theorem, we can also derive an upper bound on number of poisoning points $T$ with a fixed $\epsilon$.  

\subsection{Lower Bound on the Number of Poisoning Points}\label{sec:lowerbound}

We first provide the lower bound on the number of poisoning points required for producing the target classifier in the addition-only setting (Theorem~\ref{theorem:lower_bound}) and then explain how the lower bound estimation can be incorporated into Algorithm~\ref{algorithm}. 
%Considering the learning problem defined in~\eqref{eq:erm}, in the theorem below, we show the minimum number of poisoning points required to produce $\theta_{p}$ in the addition only poisoning scenario. 
The intuition behind the theorem below is, when the number of poisoning points added to the clean training set is smaller than the lower bound, there always exists a classifier $\theta$ with lower loss compared to $\theta_p$ and hence the target classifier cannot be attained. The full proof of the theorem can be found in Appendix~\ref{sec:proofs}.

\iffalse
\begin{theorem}[Lower Bound]\label{thm:lowerbounds}
    Given a target classifier $\theta_{p}$, to reproduce $\theta_{p}$ by adding the poisoning set $\mathcal{D}_{p}$ into $\mathcal{D}_{c}$, the number of poisoning points $|\mathcal{D}_{p}|$ cannot be lower than 
    \begin{align*}
    & \sup_{\theta} z(\theta) = \\
    &\frac{L(\theta_{p};\mathcal{D}_{c})-L(\theta;\mathcal{D}_{c}) + NC_{R}(R(\theta_p)-R(\theta))}{\sup_{x,y}\big(l(\theta;x,y)-l(\theta_{p};x,y)\big) + C_{R}(R(\theta)-R(\theta_{p}))}.
    \end{align*}
    \label{theorem:lower_bound}
\end{theorem}

\begin{corollary}
\label{cor:eps_lower_bound}
If we further assume bi-directional closeness in the loss-based distance, we can also derive the lower bound on number of poisoning points needed to induce models that are $\epsilon$-close to the target model.  More precisely, if $\theta_1$ being $\epsilon$-close to $\theta_2$ implies that $\theta_2$ is also $k\cdot \epsilon$ close to $\theta_1$, then we have, 
\begin{align*}
& \sup_{\theta}z'(\theta) = \\ &\frac{L(\theta_{p};\mathcal{D}_{c})-L(\theta;\mathcal{D}_{c}) - NC_{R}\cdot R^{*} -Nk\epsilon}{\sup_{x,y}\big(l(\theta;x,y)-l(\theta_{p};x,y)\big) + C_{R}\cdot R^{*} + k\epsilon}.
\end{align*}
where $R^*$ is an upper bound on the regularizer $R(\theta)$.
\end{corollary}
\fi

\begin{theorem}[Lower Bound]\label{thm:lowerbounds}
    Given a target classifier $\theta_{p}$, to reproduce $\theta_{p}$ by adding the poisoning set $\mathcal{D}_{p}$ into $\mathcal{D}_{c}$, the number of poisoning points $|\mathcal{D}_{p}|$ cannot be lower than 
    \begin{align*}
    \sup_{\theta} z(\theta) = \frac{L(\theta_{p};\mathcal{D}_{c})-L(\theta;\mathcal{D}_{c}) + NC_{R}(R(\theta_p)-R(\theta))}{\sup_{x,y}\big(l(\theta;x,y)-l(\theta_{p};x,y)\big) + C_{R}(R(\theta)-R(\theta_{p}))}.
    \end{align*}
    \label{theorem:lower_bound}
\end{theorem}

\begin{corollary}
\label{cor:eps_lower_bound}
If we further assume bi-directional closeness in the loss-based distance, we can also derive the lower bound on number of poisoning points needed to induce models that are $\epsilon$-close to the target model.  More precisely, if $\theta_1$ being $\epsilon$-close to $\theta_2$ implies that $\theta_2$ is also $k\cdot \epsilon$ close to $\theta_1$, then we have, 
\begin{align*}
\sup_{\theta}z'(\theta) = \frac{L(\theta_{p};\mathcal{D}_{c})-L(\theta;\mathcal{D}_{c}) - NC_{R}\cdot R^{*} -Nk\epsilon}{\sup_{x,y}\big(l(\theta;x,y)-l(\theta_{p};x,y)\big) + C_{R}\cdot R^{*} + k\epsilon}.
\end{align*}
where $R^*$ is an upper bound on the regularizer $R(\theta)$.
\end{corollary}

% \begin{equation} \label{eq1}
% \begin{split}
% A & = \frac{\pi r^2}{2} \\
%  & = \frac{1}{2} \pi r^2
% \end{split}
% \end{equation}

% \begin{remark} \rm
%     To obtain the actual value of the lower bound to achieve $\theta_p$, one just needs to replace $\theta$ with a known classifier. One easy choice is to replace $\theta$ with model $\theta_c$ trained on $\cD_c$.
% \end{remark}

The formula for the lower bound in Theorem~\ref{theorem:lower_bound} (and also the lower bound in Corollary~\ref{cor:eps_lower_bound}) can be easily incorporated into Algorithm~\ref{algorithm} to obtain a tighter theoretical lower bound. We simply need to check all of the intermediate classifiers $\theta_t$ produced during the attack process and replace $\theta$ with $\theta_t$, 
%As long as $L(\theta_{p};\mathcal{D}_{c})-L(\theta_{t};\mathcal{D}_{c}) + NC_{R}(R(\theta_p)-R(\theta_{t})) > 0$, 
and the lower bound can be computed for the pair of $\theta_t$ and $\theta_p$. Algorithm~\ref{algorithm} then additionally returns the lower bound, which is the highest lower bound computed from our poisoning procedure. 
%{\color{red} We also note that the lower bound only shows the minimum number of poisoning points required to induce a target classifier, but does not provide an actual poisoning set.}

\section{Experiments}
\label{sec:experiments}

%We first describe our experimental setup regarding the datasets, models, attacks and target classifiers. Next, 

We present the experimental results by showing the convergence of Algorithm~\ref{algorithm}, the comparison of attack success rates to state-of-the-art model-targeted poisoning attack, and the theoretical lower bound for inducing a given target classifier and its gap to the number of poisoning points used by our attack. All of our evaluation code is available at: \url{https://github.com/suyeecav/model-targeted-poisoning}.

\shortsection{Datasets and Subpopulations} We experiment on both the practical subpopulation and the conventional indiscriminate attack scenarios. { 
We selected datasets and models for our experiments based on evaluations of previous poisoning attacks~\citep{biggio2012poisoning,mei2015security,koh2018stronger,steinhardt2017certified,koh2017understanding,jagielski2019subpop}. 
For the subpopulation attack experiments, we use the Adult dataset~\citep{Dua:2019}, which was used for evaluation by~\citep{jagielski2019subpop}.
We downsampled the Adult dataset to make it class-balanced and ended up with 15,682 training and 7,692 test examples. Each example has the dimension of 57 after one-hot encoding the categorical attributes. For the indiscriminate setting, we use the Dogfish~\citep{koh2017understanding} and \MNIST\ datasets~\citep{lecun1998mnist}\footnote{\MNIST\ dataset is a subset of the well-known MNIST dataset that only contains digit 1 and 7.}. The Dogfish dataset contains 1,800 training and 600 test samples. We use the same Inception-v3 features~\citep{szegedy2016rethinking} as in~\citet{koh2017understanding,steinhardt2017certified,koh2018stronger} and each image is represented by a 2,048-dimensional vector. The \MNIST\ dataset contains 13,007 training and 2,163 test samples, and each image is flattened to a 784-dimensional vector. 

We identify the subpopulations for the Adult dataset using $k$-means clustering techniques (ClusterMatch in \citet{jagielski2019subpop}) to obtain different clusters ($k=20$). For each cluster, we select instances with the label ``$\leq$ 50K'' to form the subpopulation (indicating all instances in the subpopulation are in the low-income group). This way of defining subpopulation is rather arbitrary (in contrast to a more likely attack goal which would select subpopulations based on demographic characteristics), but enables us to simplify analyses. From the $20$ subpopulations obtained, we select three subpopulations with the highest test accuracy on the clean model. They all have 100\% test accuracy, indicating all instances in these subpopulations are correctly classified as low income. This enables us to use ``attack success rate'' and ``accuracy'' without any ambiguity on the subpopulation---for each of our subpopulations, all instances are originally classified as low income, and the simulated attacker's goal is to have them classified as high income. 

\shortsection{Models and Attacks}
We conduct experiments on linear SVM and logistic regression (LR) models. Although our theoretical results do not apply to non-convex models, for curiosity we also tested our attack on deep neural networks and report results in Appendix~\ref{sec:dnn_results}.

We use the heuristic approach from~\citet{koh2018stronger} to generate target classifiers for both attack settings. In the subpopulation setting, for each subpopulation, we generate a target model that has 0\% accuracy (100\% attacker success) on the subpopulation, indicating that all subpopulation instances are now classified as high income. In the indiscriminate setting, for \MNIST, we aim to generate three target classifiers with overall test errors of 5\%, 10\%, and 15\%. For SVM, we obtained target models of test accuracies of 94.0\%, 88.8\%, and 83.3\%, and for LR, the target models are of test accuracies of 94.7\%, 89.0\%, and 84.5\%. For Dogfish, we aim to generate target models with overall test errors of 10\%, 20\%, and 30\%. For SVM, we obtained target models of test accuracies of 89.3\%, 78.3\%, and 67.2\% and for logistic regression, we obtained target models of test accuracies of 89.0\% 79.5\%, and 67.3\%. The test accuracy of the clean SVM model is 78.5\% on Adult, 98.9\% on \MNIST\, and is 98.5\% on Dogfish. The test accuracy of clean LR model is 79.9\% on Adult, 99.1\% on \MNIST\, and 98.5\% on Dogfish.

We compare our model-targeted poisoning attack in Algorithm~\ref{algorithm} to the state-of-the-art KKT attack~\citep{koh2018stronger}. We do not include the model-targeted attack from~\citet{mei2015using} because there is no open source implementation and this attack is also reported to underperform the KKT attack \citep{koh2018stronger}. Our main focus here is on comparing to other model-targeted attacks in terms of achieving the target models, but we also do include experiments (in Appendix~\ref{sec:model_vs_objective}) comparing our attack to existing objective-driven attacks, where the target model for our attack is selected to achieve that objective. With a carefully selected target model, our attack can also outperform the state-of-the-art objective-driven attack. 

%We also do not comprehensively compare to objective-driven attacks because our main goal here is to evaluate how well our attack approaches a given target model, across a range of target models. However, model-targeted attacks can be compared to objective-driven attacks with regards to a given attacker objective by choosing the target model in a careful way and we show some preliminary results at the end of this section. 
%We provide some heuristics of choosing such target models and show at the end of this section that, with a carefully selected target model, our attack can also outperform the state-of-the-art objective-driven attack.
}

%We do not include the two other strong attacks in our evaluation. We exclude the influence attack~\citep{koh2017understanding,koh2018stronger} because it slightly underperforms the KKT attack, does not work with a target classifier, and is also slow. We exclude the min-max attack~\citep{steinhardt2017certified,koh2018stronger} because it only works for the indiscriminate scenario and does not use a target classifier in the attack process. The simple random label-flipping attack for subpopulation~\citep{jagielski2019subpop} is not considered either because it consistently underperforms the KKT attack.

%\shortsection{Metrics}
Both our attack and the KKT attack take as input a target model and the original training data, and output a set of poisoning points intended to induce a model as close as possible to the target model when the poisoning points are added to the original training data. We compare the effectiveness of the attacks by testing them using the same target model and measuring convergence of their induced models to the target model.

The KKT attack requires the number of poisoning points as an input, while our attack is more flexible and can produce poisoning points in priority order without a preset number. As a stopping condition for our experiments, we use either a target number of poisoning points or a threshold for $\epsilon$-close distance to the target model. Since we do not know the number of poisoning points needed to reach some attacker goal in advance for the KKT attack, we first run our attack and produce a classifier that satisfies the selected $\epsilon$-close distance threshold. The loss function is hinge loss for SVM and logistic loss for LR. For SVM model, we set $\epsilon$ as 0.01 on Adult, 0.1 on \MNIST\, and 2.0 on Dogfish dataset. For LR model, we set $\epsilon$ as 0.05 on Adult, 0.1 on \MNIST\, and 1.0 on Dogfish. Then, we use the size of the poisoning set returned from our attack (denoted by $n_p$) as the input to the KKT attack for the target number of poisons needed. We also compare the two attacks with varying numbers of poisoning points up to $n_p$. 
For the KKT attack, its entire optimization process must be rerun
whenever the target number of poisoning points changes. Hence, it is infeasible to evaluate the KKT attack on many different poisoning set sizes. In our experiments, we run the KKT attack five poisoning set sizes: $0.2 \cdot n_p$, $0.4 \cdot n_p$, $0.6 \cdot n_p$, $0.8 \cdot n_p$, and $n_p$. For our attack, we simply run iterations up to the maximum number of poisoning points, collecting a data point for each iteration up to $n_p$. In Appendix~\ref{sec:additional_exps}, we also plot the performance of our attack with respect to the number of poisoning points added across iterations.

    \begin{figure*}[tb]
        \centering
        \begin{subfigure}[b]{0.45\textwidth} 
            \centering 
            \includegraphics[width=\textwidth]{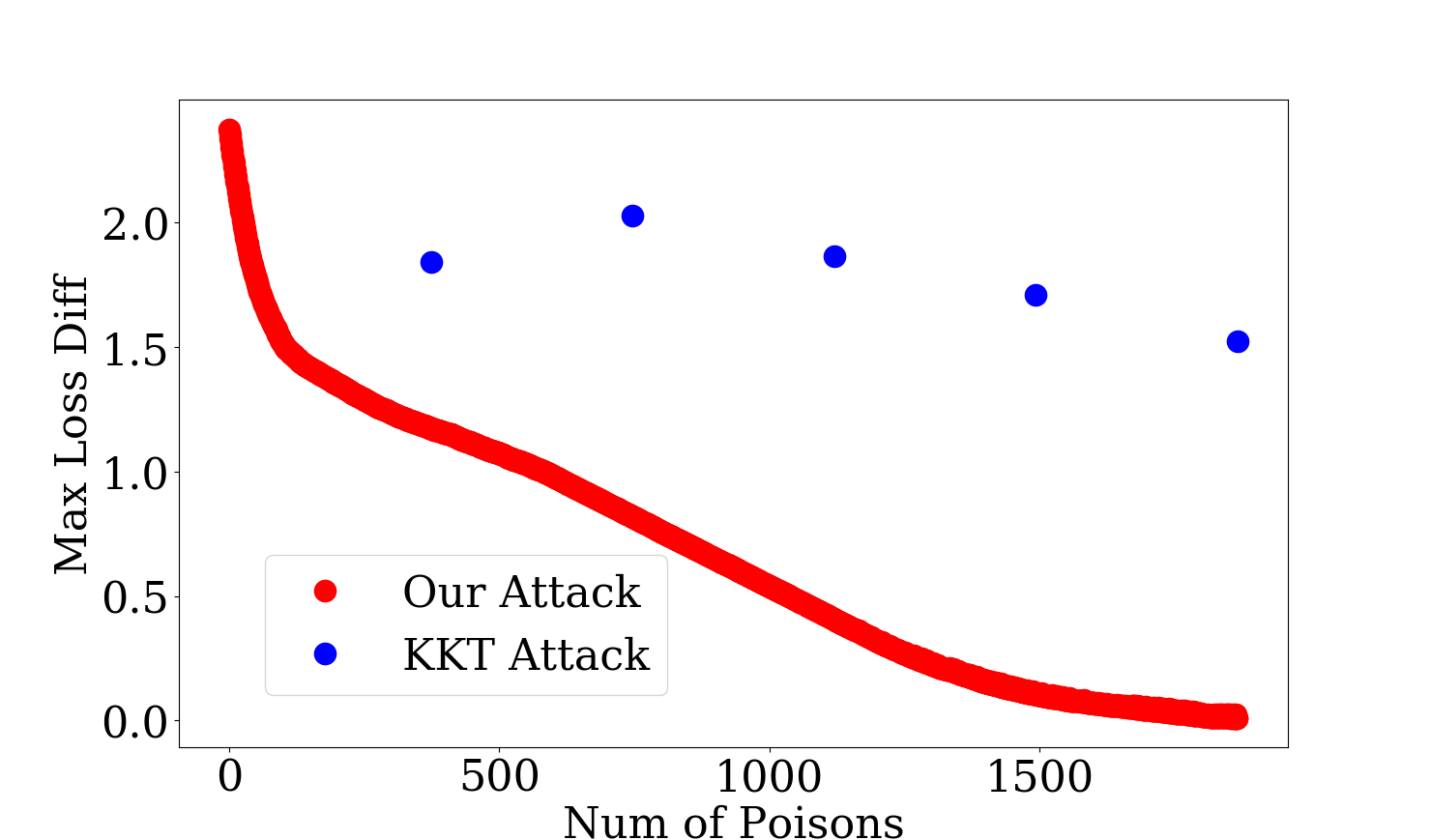}
            \caption[]%
            {Max Loss Difference}    
            \label{fig:adult_svm_subpop0_max_loss_diff}
        \end{subfigure}
        \begin{subfigure}[b]{0.45\textwidth} 
            \centering 
        \includegraphics[width=\textwidth]{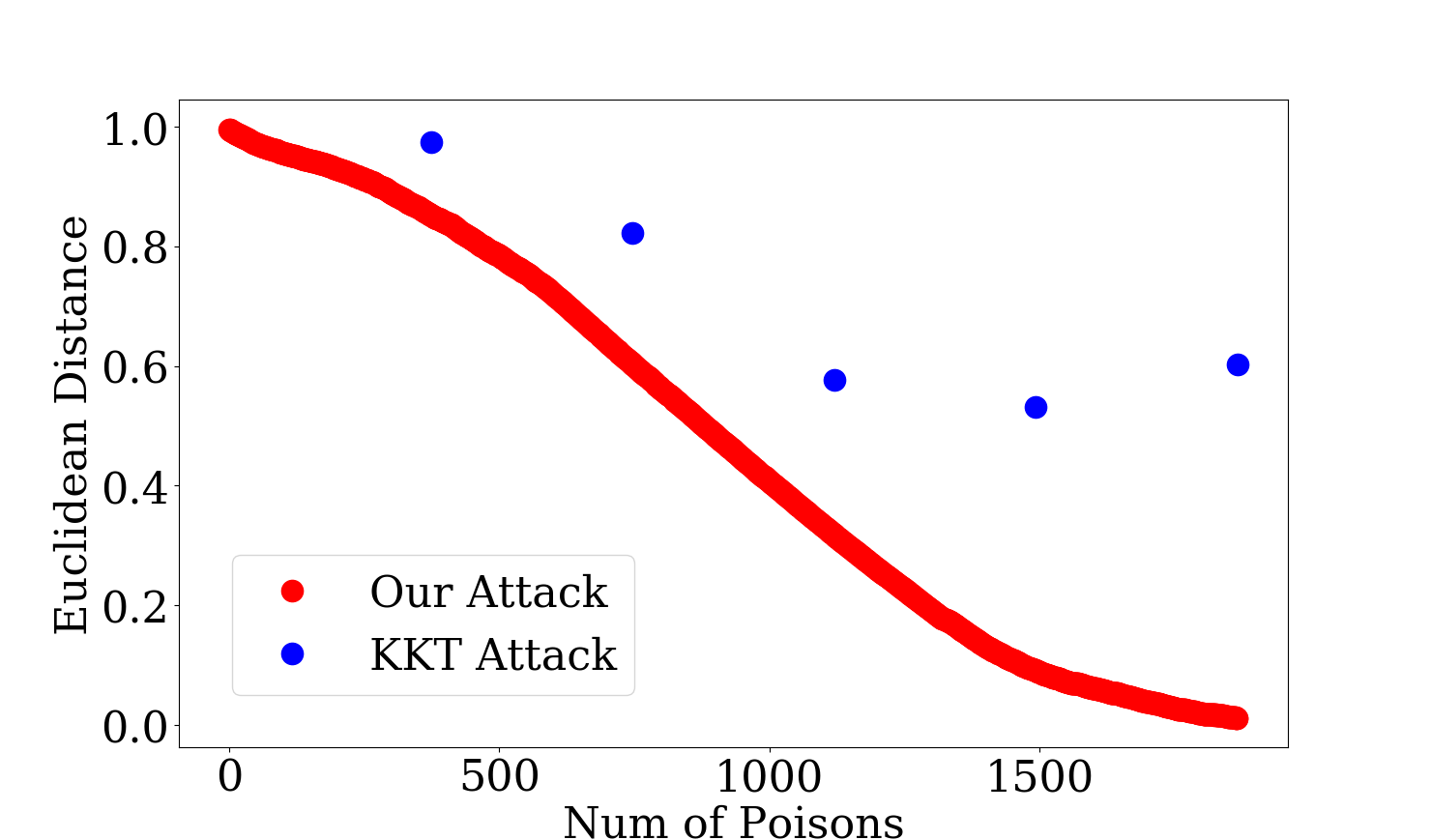}
            \caption[]%
            {Euclidean Distance}    
            \label{fig:adult_svm_subpop0_norm_diff}
        \end{subfigure}
        \caption[]
        {Convergence to the target model. Results shown are for the first subpopulation (Cluster 0) and the model is SVM. %The maximum number of poisons is set by terminating the attack with $0.01$-close threshold to the target classifier.
        }
        \label{fig:adult_svm_subpop0_convergence}
    \end{figure*}

\begin{table*}[h]
\centering
\scalebox{0.73}{
\begin{tabular}{cccccccccccccc}
\toprule
\multirow{2}{*}{\begin{tabular}[c]{@{}c@{}}Model/\\ Dataset\end{tabular}} & \multirow{2}{*}{\begin{tabular}[c]{@{}c@{}}Target\\ Model\end{tabular}} & \multirow{2}{*}{$n_p$} & \multirow{2}{*}{\begin{tabular}[c]{@{}c@{}}Lower\\ Bound\end{tabular}} & \multicolumn{2}{c}{$0.2n_p$} & \multicolumn{2}{c}{$0.4n_p$} & \multicolumn{2}{c}{$0.6n_p$} & \multicolumn{2}{c}{$0.8n_p$} & \multicolumn{2}{c}{$n_p$} \\
 &  &  &  & KKT & Ours & KKT & Ours & KKT & Ours & KKT & Ours & KKT & Ours \\\midrule
\multirow{3}{*}{\begin{tabular}[c]{@{}c@{}}SVM/\\ Adult\end{tabular}} & Cluster 0 & 1,866 & 1,667$\pm$2.4 & 96.8 & 98.4 & 65.4 & \textbf{51.6} & 14.9 & 35.6 & 1.1 & 2.7 & 15.4 & \textbf{0.5} \\
 & Cluster 1 & 2,097 & 1,831.4$\pm$5.0 & 72.2 & 77.1 & 41.0 & \textbf{23.6} & 2.8 & \textbf{0.7} & 1.4 & \textbf{0.7} & 6.9 & \textbf{0.0} \\
 & Cluster 2 & 2,163.3$\pm$2.5 & 1,863.0$\pm$9.2 & 94.9 $\pm$ 0.7 & \textbf{24.3$\pm$ 0.3} & 15.9 & 20.3 & 34.3 $\pm$ 0.2 & \textbf{12.1} & 21.6 $\pm$ 0.1  & \textbf{0.3} & 20.3 $\pm$ 0.7 & \textbf{0.3} \\\midrule
\multirow{3}{*}{\begin{tabular}[c]{@{}c@{}}LR/\\ Adult\end{tabular}} & Cluster 0 & 2,005 & N/A & 82.1 $\pm$ 1.0 & \textbf{75.6} & 71.7 $\pm$ 0.4 & \textbf{42.2} & 46.7 & \textbf{15.9} & 36.6 $\pm$ 2.1 & \textbf{1.9} & 24.3 $\pm$ 0.8 & \textbf{0.4} \\
 & Cluster 1 & 1,630$\pm$1.1 & N/A & 98.1 & \textbf{94.9} & 97.2 & \textbf{79.0} & 96.7 & \textbf{34.1} & 95.8 & \textbf{6.1} & 95.8 & \textbf{0.5} \\
 & Cluster 2 & 2,428 & N/A & 97.9 & \textbf{94.5} & 93.9 & \textbf{45.8} & 89.8 $\pm$ 0.7 & \textbf{5.8} & 79.2 $\pm$ 4.6 & \textbf{0.6} & 60.2 $\pm$ 1.7 & \textbf{0.6}\\
 \bottomrule
\end{tabular}
}
\caption{Subpopulation attack on Adult: comparison of test accuracies on subpopulations (\%). Target models for Adult dataset consist of models with 0\% accuracy on the selected subpopulations (Cluster 0 - Cluster 2). $n_p$ denotes the maximum number of poisoning points used by our attack, and $xn_p$ denotes comparing the two attacks at $xn_p$ poisoning points. $n_p$ is set by running our attack till the induced model becomes $0.01$-close to the target model. All results are averaged over 4 runs and standard error is 0 (exact same number of misclassifications across the runs), except where reported. We do not show lower bound for LR because we can only compute an approximate maximum loss difference and the lower bound will no longer be valid.
}
\label{tab:model-targeted-comparison-subpop}
\end{table*}

%\begin{table*}[h]
\begin{table*}[!tbh]
\centering
\scalebox{0.78}{
\begin{tabular}{cccccccccccccc}
\toprule
 &  &  &  & \multicolumn{2}{c}{$0.2n_p$} & \multicolumn{2}{c}{$0.4n_p$} & \multicolumn{2}{c}{$0.6n_p$} & \multicolumn{2}{c}{$0.8n_p$} & \multicolumn{2}{c}{$n_p$} \\
\multirow{-2}{*}{\begin{tabular}[c]{@{}c@{}}Model/\\ Dataset\end{tabular}} & \multirow{-2}{*}{\begin{tabular}[c]{@{}c@{}}Target\\ Model\end{tabular}} & \multirow{-2}{*}{$n_p$} & \multirow{-2}{*}{\begin{tabular}[c]{@{}c@{}}Lower\\ Bound\end{tabular}} & KKT & Ours & KKT & Ours & KKT & Ours & KKT & Ours & KKT & Ours \\ \midrule
 & 5\% Error & 1,737 & 874 & 97.3 & \textbf{97.1} & 96.4 & \textbf{96.1} & 95.7 & 95.7 & 94.9 & 94.9 & 94.3 & 94.6 \\
 & 10\% Error & 5,458 & 3,850.4$\pm$0.8 & 95.8 & \textbf{95.5} & 93.4 & \textbf{92.1} & 92.7 & \textbf{90.9} & 91.1 & \textbf{90.7} & 90.2 & 90.2 \\ 
\multirow{-3}{*}{\begin{tabular}[c]{@{}c@{}}SVM/\\ MNIST 1-7\end{tabular}} & 15\% Error & 6,192 & 4,904 & 98.3 & \textbf{97.8} & 96.3 & 98.1 & 97.2 & 97.3 & 98.3 & \textbf{92.7} & 82.7 & 85.9 \\ \midrule
 & 10\% Error & 32 & 15 & 97.0 & \textbf{95.8} & 94.0 & \textbf{93.3} & 92.2 & \textbf{91.2} & 90.7 & \textbf{90.2} & 90.3 & \textbf{89.8} \\
 & 20\% Error & 89 & 45 & 95.5 & 95.7 & 92.5 & \textbf{92.2} & 90.3 & \textbf{88.7} & 84.7 & 84.7 & 82.3 & \textbf{82.0} \\ 
\multirow{-3}{*}{\begin{tabular}[c]{@{}c@{}}SVM/\\ Dogfish\end{tabular}} & 30\% Error & 169 & 83 & 95.5 & 95.7 & 93.5 & \textbf{90.7} & 88.3 & \textbf{82.5} & 78.3 & \textbf{75.2} & 71.8 & \textbf{71.7} \\ \midrule
 & 5\% Error & 756 & N/A & 97.5 & \textbf{96.9} & 97.4 & \textbf{96.5} & 97.2 & \textbf{96.0} & 96.9 & \textbf{95.7} & 96.9 & \textbf{95.2} \\
 & 10\% Error & 2,113 & N/A & 97.0 & \textbf{95.7} & 96.9 & \textbf{93.8} & 96.8 & \textbf{92.3} & 96.2 & \textbf{91.4} & 96.4 & \textbf{90.4} \\ 
\multirow{-3}{*}{\begin{tabular}[c]{@{}c@{}}LR/\\ MNIST 1-7\end{tabular}} & 15\% Error & 3,907 & N/A & 96.9 & \textbf{95.4} & 97.0 & \textbf{93.3} & 97.1 & \textbf{90.7} & 97.1 & \textbf{88.3} & 97.1 & \textbf{87.1} \\ \midrule
 & 10\% Error & 62 & N/A & 98.8 & \textbf{93.0} & 98.5 $\pm$ 0.1 & \textbf{89.7 $\pm$ 0.3} & 98.8 & \textbf{89.2} & 98.8 & \textbf{89.2} & 98.8 & \textbf{89.0} \\ 
 & 20\% Error & 120 & N/A & 98.5 & \textbf{93.2} & 99.2 & \textbf{88.2} & 99.3 & \textbf{85.3} & 99.5 & \textbf{83.0} & 99.5 & \textbf{80.7} \\
\multirow{-3}{*}{\begin{tabular}[c]{@{}c@{}}LR/\\ Dogfish\end{tabular}} & 30\% Error & 181 & N/A & 97.8 & \textbf{92.3} & 98.8 & \textbf{85.7 $\pm$ 0.2} &99.2 & \textbf{81.3 $\pm$ 0.3} & 99.5 & \textbf{75.7} & 99.5 & \textbf{72.5 $\pm$ 0.2} \\
\bottomrule
\end{tabular}
}
\caption{Indiscriminate attack on \MNIST\ and Dogfish: comparison of overall test accuracies (\%). The target models are of certain overall test errors. $n_p$ is set by running our attack till the induced model becomes $\epsilon$-close to the target model and we set $\epsilon$ as 0.1 for \MNIST\ and 2.0 for Dogfish dataset. All results are averaged over 4 runs and standard error is 0 (exact same number of misclassifications across the runs), except where reported.
}
\label{tab:model-targeted-comparison-indis}
\end{table*}

\shortsection{Convergence}
Figure~\ref{fig:adult_svm_subpop0_convergence} shows the convergence of Algorithm~\ref{algorithm} using both maximum loss difference and Euclidean distance to the target, and the result is reported on the first subpopulation (Cluster 0) of Adult and the model is SVM. The maximum number of poisoning points ($n_p$) for the experiments is obtained when the classifier from Algorithm~\ref{algorithm} is $0.01$-close to the target classifier. Our attack steadily reduces the maximum loss difference and Euclidean distance to the target model, in contrast to the KKT attack which does not seem to converge towards the target model reliably. 
Concretely, at the maximum number of poisons in Figure~\ref{fig:adult_svm_subpop0_convergence}, both the maximum loss difference and Euclidean distance of our attack (to the target) are less than 2\% of the corresponding distances of the KKT attack. Similar observations are also observed for the indiscriminate and other subpopulation attack settings, see Appendix~\ref{sec:additional_exps}. 

We believe our attack outperforms the KKT attack in terms of convergence to the target model because it approaches the target classifier differently. The foundation of the KKT attack is that for binary classification, for any target classifier generated by training on a set $D_c\cup D_p$ with $|D_p|=n$, the (exact) same classifier can also be obtained by training on set $D_c \cup D'_p$ with $|D'_p|\leq n$ and this poisoning set $D'_p$ only contains two distinct points, one from each class. In practice, the KKT attack often aims to induce the exact same classifier with much fewer poisoning points, which may not be feasible and leads the KKT attack to fail. In contrast, our attack does not try to obtain the exact target model but just selects each poisoning point in turn as the one with the best expected impact. Hence, our attack  gets close to the target model with fewer poisoning points than the number of points used to exactly produce the target model. 
%However, regarding the attack success, this advantage is not always realized (e.g., in the last column of Table~\ref{tab:model-targeted-comparison-indis}, KKT attack is slightly better than our attack on \MNIST\ dataset with target model of 15\% error) because the attack success metric is not closely tied to closeness to the target model. 
%For example, in the indiscriminate attack scenario on MNIST 1--7 dataset (Appendix~\ref{sec:additional_exps}), the KKT attack sometimes slightly outperforms our attack.

\shortsection{Attack Success}
Next, we compare the classifiers induced by the two attacks in terms of the attacker's goal. Table~\ref{tab:model-targeted-comparison-subpop} summarize the results of the subpopulation attacks, where attack success is measured on the targeted cluster. At the maximum number of poisons, our attack is much more successful than the KKT attack, for both the SVM and LR models. For example, on Cluster 1 with LR, the induced classifier from our attack has 0.5\% accuracy compared to the 95.8\% accuracy of KKT.
%---the induced classifiers have 0.5\% accuracy compared to 15.4\% accuracy for KKT on subpopulation 1, 0.0\% compared to 6.9\% on subpoulation 2, and 0.3\% compared to 20.1\% on subpoplation 3. 
Table~\ref{tab:model-targeted-comparison-indis} shows the results of indiscriminate attacks on \MNIST\ and Dogfish, and the attack success is the overall test error. 
% error rate on the targeted population. For the subpopulation attacks on Adult, this is the targeted clusted; for the indiscriminate attacks on \MNIST\ and Dogfish, the attack success is the overall test error. \todo{is this correct? I couldn't find it defined} 
For the indiscriminate attack on SVM, both on \MNIST\ and Dogfish, the two attacks have similar performance while for LR, our attack is much better than the KKT attack. The reason of KKT failing on LR is, its objective function becomes highly non-convex and is very hard to optimize. More details about the formulation can be found in \citet{koh2018stronger}. For logistic loss, our attack also needs to maximize a non-concave maximum loss difference\footnote{We use Adam optimizer~\citep{kingma2014adam} with random restarts to solve this maximization problem approximately.} (Step 4 in Algorithm~\ref{algorithm}). However, this objective is much easier to optimize than that of the KKT attack and our attack is still very effective on LR models.

\shortsection{Optimality of Our Attack} To check the optimality of our attack, we calculate a lower bound on the number of poisoning points needed to induce the model that is induced by the poisoning points  found by our attack. We calculate this lower bound on the number of poisons using Theorem~\ref{thm:lowerbounds} (details in Section~\ref{sec:lowerbound}). Note that Theorem \ref{thm:lowerbounds} provides a valid lower bound based on any intermediate model. To get a lower bound on the number of poisoning points, we only need to use Theorem \ref{thm:lowerbounds} on the encountered intermediate models and report the best one. We do this by running Algorithm~\ref{algorithm} using the induced model (and not the previous target model) as the target model, terminating when the induced classifier is $\epsilon$-close to the given target model. Note that for LR, maximizing the loss difference is not concave and therefore, we cannot obtain the actual maximum loss difference, which is required in the denominator in Theorem~\ref{theorem:lower_bound}. Therefore, we only report results on SVM. For the subpopulation attack on Adult, we set $\epsilon=0.01$ and for the indiscriminate attack on \MNIST\ and Dogfish, we set $\epsilon$ to 0.1 and 2.0 respectively. We then consider all the intermediate classifiers that the algorithm induced across the iterations. 
%The Adult dataset contains many binary features and hence the search process for the point with maximum loss difference is a (non-convex) mixed integer programming problem. To obtain a valid lower bound, we relax the binary features into real valued features in range $[0,1]$ and then compute the lower bound. 
Our calculated lower bound in Table~\ref{tab:model-targeted-comparison-subpop} (Column 3-4) shows that for the Adult dataset, the gap between the lower bound and the number of used poisoning points is relatively small. This means our attack is nearly optimal in terms of minimizing the number of poisoning points needed to induce the target classifier. However, for the \MNIST\ and Dogfish datasets in Table~\ref{tab:model-targeted-comparison-indis}, there still exists some gap between the lower bound and the number of poisoning points used by our attack, indicating there might exist more efficient model-targeted poisoning attacks. 

\section{Conclusion} % and Discussion}

We propose a general poisoning framework with provable guarantees to approach any attainable target classifier, along with a lower and upper bound on the number of poisoning points needed. Our attack is a generic tool that first captures the adversary's goal as a target model and then focuses on the power of attacks to induce that model. This separation enables future work to explore the effectiveness of poisoning attacks corresponding to different adversarial goals. We have not considered defenses in this work, and it is an important and interesting direction to study the effectiveness of our attack against data poisoning defenses. Defenses may be designed to limit the search space of the points with maximum loss difference and hence increase the number of poisoning points needed. We also leave the investigation of the application of our model-targeted attacks in other attacker objectives, e.g. backdoor attacks and privacy attacks, for future work.

\section*{Acknowledgements}

This work was partially funded by awards from the National Science Foundation (NSF) SaTC program (Center for Trustworthy Machine Learning, \#1804603), NSF Office of Advanced Cyberinfrastructure (\#2002985) and Amazon research award.

% \todo{add acks}
% \end{comment}

\bibliography{icml2021}
\bibliographystyle{icml2021}

\clearpage
\appendix
% \onecolumn{\icmltitle{Model-Targeted Poisoning Attacks: Supplemental Material}
%           \vskip 0.3in}
%\section{Appendix}
\section{Proofs}
\label{sec:proofs}
In this section, we provide the proofs of the main theorems shown in this paper. For convenience, we restate all the theorems below while also referencing to the main paper. 

Before proving the main theorem, we introduce two new  definitions and several lemmas to assist with the proof.

\begin{definition}[Attainable models]\label{defn:attainable}
We say $\theta$ is $C_R$-attainable with respect to loss function $l$ and regularization function $R$ if there exists a training set $\cD$ such that 
$$\theta= \argmin_{\theta\in\Theta} \frac{1}{|D|}\cdot L(\theta;\cD) +C_R\cdot R(\theta) $$
\end{definition}

% Attainable parameter means the model parameter can be obtained by training some training data from the space $\mathcal{X}\times\mathcal{Y}$. Note that $\sup_{x,y}\big(l(\theta_2; x,y) - l(\theta_1; x,y\big) \geq 0$ and hence $C\geq 0$. The argument is simple: if we assume the supremium of the loss difference can be negative, then $\theta_1$ is not attainable because no matter what kind of points are selected into the training set, $\theta_2$ always has lower loss compared to $\theta_1$ and $\theta_1$ will be never be the solution to . \snote{this is where the problem is when we talk about ERM or structural ERM}

\begin{lemma}\label{lemma:c>c_r}
Let $\theta_1$ and $\theta_2$ be two $C_R$-attainable parameters for some $C_R>0$ such that $R(\theta_1)>R(\theta_2)$. Then,
$$\sup_{x,y} \big(l(\theta_2; x,y) - l(\theta_1; x,y)\big)/\big(R(\theta_1)-R(\theta_2)\big)> C_R.$$
\end{lemma}

\begin{proof}
 Consider any attainable pairs of $(\theta_1, \theta_2)$ such that $R(\theta_1) > R(\theta_2)$ and let $\cD_1$ to be training set that the training algorithm produces the unique minimizer $\theta_1$. Namely,
 $$\theta_1 = \argmin_{\theta} \frac{1}{|\cD_1|} \cdot L(\theta; \cD_1) + C_R\cdot R(\theta) $$
 
 Since $\theta_1$ minimizes the total loss on $\cD_1$ uniquely, we have
 $$\frac{1}{|\cD_1|} L(\theta_2; \cD_1) + C_R\cdot R(\theta_2) > \frac{1}{|\cD_1|} L(\theta_1; \cD_1) + C_R\cdot R(\theta_1)$$

%  $$\theta_1 = \min_{\theta} \frac{1}{|T_1|} \cdot loss(\theta, T_1) + C_R\cdot R(\theta) $$

By rearranging the above inequality and by an averaging argument, we have
 $$
  \sup_{x,y} \big(l(\theta_2; x,y) - l(\theta_1; x,y)\big) \geq \frac{1}{|\cD_1|} L(\theta_2; \cD_1) - \frac{1}{|\cD_1|} L(\theta_1; \cD_1) > C_R\cdot\big(R(\theta_1) - R(\theta_2)\big).
 $$
 
Now since $R(\theta_1)> R(\theta_2)$ we have
 $$
   \sup_{x,y} \big(l(\theta_2; x,y) - l(\theta_1; x,y)\big)/\big(R(\theta_1)-R(\theta_2)\big) > C_{R}. 
 $$
%  \begin{align*}
%  \frac{1}{|T_2|} \cdot loss(\theta_2, T_2) + C_R\cdot R(\theta_2) \geq \frac{1}{|T_2|} \cdot loss(\theta_1, T_2)\\
%  - (\max_{(x,y)} loss\big(\theta_1:(x,y)) - loss(\theta_2:(x,y))\big)\\
%  + C_R\cdot R(\theta_1) + C_R\cdot(R(\theta_2) - R(\theta_1))
%  \end{align*}
%  Therefore since $\frac{1}{|T_2|} loss(\theta_1, T_2) + C_R\cdot R(\theta_1) > \frac{1}{|T_2|} loss(\theta_2, T_2) + C_R\cdot R(\theta_2)$ we have
%   $$-(\max_{(x,y)} loss\big(\theta_1:(x,y)) - loss(\theta_2:(x,y))\big)
%   + C_R\cdot(R(\theta_2) - R(\theta_1))<0.$$
%   which implies 
%   $$(max_{(x,y)} loss\big(\theta_1:(x,y)) - loss(\theta_2:(x,y))\big)/(R(\theta_2) - R(\theta_1)) > C_R$$
\end{proof}

\begin{lemma}
\label{lemma:max_loss_dff-reg_diff}
    Let $F$ be the family of all $C_R$-attainable models. For any $\theta_1 \in F$, there is a constant $\gamma$ where for all $\theta_2 \in F$ we have $$\sup_{x,y}\big( l(\theta_2;x,y)-l(\theta_1;x,y)\big) +C_R(R(\theta_2)-R(\theta_1)) > \gamma\cdot\sup_{x,y}\big( l(\theta_2; x,y)-l(\theta_1;x,y)\big)$$ 
    {where $\gamma$ is a positive constant related to, $\theta_1$, $C_R$ and other model parameters (fixed for a given classification task)}.
\end{lemma}

\begin{proof}
We prove the lemma for $\gamma = 1 - C_R/C$ for $$C=\left(\inf_{\substack{\theta_2\in F\\
\text{s.t. } R(\theta_1) > R(\theta_2) }}\sup_{x,y} (l(\theta_2;x,y)-l(\theta_1;x,y))/(R(\theta_1) - R(\theta_2))\right).$$  

First, note that by Lemma \ref{lemma:c>c_r} we have 
\begin{equation}\label{eq:00001}
    C>C_R\geq0.
\end{equation}
which implies $\gamma$ is positive.
    Now we consider two subcases based on the sign of $R(\theta_2)-R(\theta_1)$:
    
\shortsectionnp{Case 1:}    ${R(\theta_2)-R(\theta_1)\geq 0}$.\ In this case the inequality is straightforward:
    \begin{align*}
            \sup_{x,y}\big( l(\theta_2;x,y)-l(\theta_1;x,y)\big) +C_R\cdot(R(\theta_2)-R(\theta_1)) & \geq \sup_{x,y}\big( l(\theta_2;x,y)-l(\theta_1;x,y)\big)\\ 
            & > (1-C_R/C)\cdot \sup_{x,y}\big( l(\theta_2;x,y)-l(\theta_1;x,y)\big),     
    \end{align*}
    where the last inequality is based on~\eqref{eq:00001}.    

\shortsectionnp{Case 2:} ${R(\theta_2) - R(\theta_1) < 0}$.\ From the definition of $C$ we have 
    $$R(\theta_1)-R(\theta_2)\leq \frac{\sup_{x,y}\big( l(\theta_2;x,y)-l(\theta_1;x,y)\big)}{C}.$$
    Equivalently, we can say 
    $$R(\theta_2)-R(\theta_1)\geq -\frac{\sup_{x,y}\big( l(\theta_2;x,y)-l(\theta_1;x,y)\big)}{C}.
    $$
    Replacing $R(\theta_2)-R(\theta_1)$ with the lower bound above completes the proof, namely
    $$\sup_{x,y}\big( l(\theta_2;x,y)-l(\theta_1;x,y)\big) +C_R(R(\theta_2)-R(\theta_1))\geq (1-C_R/C)\cdot \sup_{x,y}\big( l(\theta_2;x,y)-l(\theta_1;x,y)\big).$$
\end{proof}
% Proofs of Lemma~\ref{lemma:c>c_r} and Lemma~\ref{lemma:max_loss_dff-reg_diff} can found in Section~\ref{sec:proofs} in the Supplementary material.

With Definition~\ref{defn:attainable} and the lemmas, we are ready to prove Theorem 4.1 (restating Theorem~\ref{theorem:convergence_main}, from Section~\ref{sec:converge_proof}):

%, corresponding to the Theorem~\ref{theorem:convergence_main} in the main paper.

%\subsection*{Proof of Theorem \ref{theorem:convergence_main}}

\shortsection{Theorem \ref{theorem:convergence_main}} {\em
After at most $T$ steps, Algorithm~\ref{algorithm} will produce the poisoning set $\cD_p$ and the classifier trained on $\cD_c\cup\cD_p$ is $\epsilon$-close to $\theta_{p}$, with respect to loss-based distance, $D_{l, \cX, \cY}$, for
    \begin{align*}
        \epsilon =\frac{\alpha(T) + L(\theta_p; D_c) - L(\theta_c;D_c)}{T\cdot \gamma}
    \end{align*}
    %\label{theorem:convergence}
    where, $\gamma$ is a constant for a given $\theta_p$ and classification task, and $\alpha(T)$ is the regret of the online algorithm when the loss function used for training is convex. 
}

% \begin{lemma}
%  Let $C_R$ ($C_R > 0$) be the  regularization constant used in the model training, then $C> C_R$.
%  \label{lemma:C>C_R}
% \end{lemma}

% \begin{lemma}
%     For any $(\theta_1,\theta_2)\in H^{2}$, $\sup_{x,y}\big( l(\theta_1;(x,y))-l(\theta_2;(x,y))\big) +\lambda(R(\theta_1)-R(\theta_2)) > (1-\frac{C_R}{C})\sup_{x,y}\big( l(\theta_1;(x,y))-l(\theta_2;(x,y))\big)$ 
%     \label{lemma:max_diff-reg}
% \end{lemma}

The goal of the adversary is to get $\epsilon$-close to $\theta_p$ (in terms of the loss-based distance) by injecting (potentially few) number of poisoned training data. The algorithm is in essence an online learning problem and we transform Algorithm~\ref{algorithm} into the form of standard online learning problem. Specifically, we adopt the {\it follow the leader} (FTL) framework to describe Algorithm~\ref{algorithm} in the language of standard online learning problem. We first describe the online learning setting considered in this paper and the notion of the regret. 

%However, when we analyze the regret bound, we play a simple trick to transform the FTL framework into the {\it follow the regularized leader} (FTRL), which additionally consider a regularization term in comparison to the FTL framework and is a commonly studied in online convex optimization literature. Describing Algorithm~\ref{algorithm} with the FTL framework instead of the FTL framework helps us to show their connections more concisely. And analyzing the regret of the FTL framework by transforming it into FTL helps us to obtain the desired sublinear regret bound. 

\begin{definition}
Let $\cL$ be a class of loss functions, $\Theta$ set of possible models, $A\colon (\Theta\times\cL)^* \to \Theta$ an online learner and $S \colon (\Theta \times \cL)^*\times \Theta \to \cL$ a strategy for picking loss functions in different rounds of online learning (adversarial environment in the context of online convex optimization). We use $\Regret(A, S, T)$ to denote the regret of $A$ against $S$, in $T$ rounds. Namely,

$$\Regret(A,S,T)=\sum_{j=0}^T l_j(\theta_j) - \min_{\theta \in \Theta} \sum_{j=0}^T l_j(\theta)$$

where 
\begin{align*}
    \theta_i=A\big((\theta_0,l_0),\dots, (\theta_{i-1}, l_{i-1})\big)
\text{~~~and~~~~}
l_i = S\big((\theta_0, l_0),\dots,(\theta_{i-1}, l_{i-1}), \theta_i\big).
\end{align*}
\end{definition}

With the online learning problem set up, we proceed to the main proof which first describes Algorithm~\ref{algorithm} in the FTL framework.
\begin{proof}[Proof of Theorem \ref{theorem:convergence_main}]

The FTL framework proceeds by solving all the functions incurred during the previous online optimization steps, namely, $A_\mathsf{FTL}((\theta_0,l_0),\dots, (\theta_{i},l_{i})) = \argmin_{\theta\in\Theta} \sum_{j=0}^i l_i(\theta)$.

%, where $r(\theta)$ is the regularization term. In this paper, since we consider the typical model training algorithm of the form in Eq.~\ref{}, where the regularization term $\eta_{m}\cdot R(\theta)$ changes as the number of training data changes. Therefore, we need a framework where the the regularizer $r(\theta)$ is chosen adaptively during the online optimization process. Therefore, we instead consider the adaptive version of the FTL framework. Specifically, the adaptive FTL framework $A_\mathsf{FTL}(l_0,\dots, l_i) = \argmin_{\theta\in\Theta} \sum_{j=1}^i (l_i(\theta) + r_i(\theta))$.
% Next, we describe how we design the $i$th loss function $l_i$ in each round of the online optimization. For the first choice, we let $l_0 = \eta_{m}\cdot R(\theta)$. In round $1$, we set $l_1 = L(\theta; \cD_c)$ and and $r_{1} = 0$. Hence, in round $1$, the adaptive FTL solves the problem of training a model based on the clean dataset $\cD_c$. From round $i\geq 2$, the loss function is adaptively chosen as: given the latest model $\theta_i$, $S_{\theta_p}$ first finds $(x*_i,y*_i)$ that maximizes the loss difference between $\theta_i$ and a target model $\theta_p$, namely, $(x^*_i,y^*_i) = \argmax_{(x,y)} l(\theta_i;x,y) - l(\theta_p;x,y)$ and $S_{\theta_p}\big((\theta_0,l_0),\dots,(\theta_{i-1},l_{i-1}),\theta_i\big)=l_{i}(\theta) = l(\theta; x^*_i, y^*_i)$ and $r_i$ is adaptively set as $r_i = (\eta_{m+i} - \eta_{m+i-1})\cdot R(\theta).$

Next, we describe how we design the $i$th loss function $l_i$ in each round of the online optimization. For the first choice, $A_\mathsf{FTL}$ chooses a random model $\theta_0\in\Theta$. In the first round (round 0), $S_{\theta_p}$ uses the clean training set $\cD_c$ and the loss is set as $$S_{\theta_p}(\theta_0)=l_0(\theta) = L(\theta; \cD_c) + N\cdot C_R\cdot R(\theta).$$ According to the FTL framework, $A_\mathsf{FTL}$ returns  model that minimizes the loss on the clean training set $\cD_c$ using the structural empirical risk minimization. For the subsequent iterations ($i\geq 1$), the loss functions is defined as, given the latest model $\theta_i$, $S_{\theta_p}$ first finds $(x^*_i,y^*_i)$ that maximizes the loss difference between $\theta_i$ and a target model $\theta_p$. Namely,
 $$(x^*_i,y^*_i) = \argmax_{(x,y)} l(\theta_i;x,y) - l(\theta_p;x,y)$$
 and then chooses the $i$th loss function as follows:
 $$S_{\theta_p}\big((\theta_0,l_0),\dots,(\theta_{i-1},l_{i-1}),\theta_i\big)=l_{i}(\theta)= l(\theta; x^*_i, y^*_i) + C_R\cdot R(\theta).$$

% $$S_{\theta_p}(\theta_0)=l_0(\theta) = L(\theta; \cD_c) + \eta_m\cdot R(\theta)$$
%  For the subsequent loss functions, given the latest model $\theta_i$, $S_{\theta_p}$ first finds $(x*_i,y*_i)$ that maximizes the loss difference between $\theta_i$ and a target model $\theta_p$. Namely,
%  $$(x^*_i,y^*_i) = \argmax_{(x,y)} l(\theta_i;x,y) - l(\theta_p;x,y)$$
%  and then chooses the $i$th loss function as follows:
%  $$S_{\theta_p}\big((\theta_0,l_0),\dots,(\theta_{i-1},l_{i-1}),\theta_i\big)=l_{i}(\theta)= l(\theta; x^*_i, y^*_i) + (\eta_{m+i} - \eta_{m+i-1})\cdot R(\theta).$$

Now we will see how FTL framework behaves when working on these loss functions at different iterations. We use $D_p^i$ to denote the set $\set{(x^*_1,y^*_1),\dots,(x^*_i,y^*_i)}$. We have
 
 \begin{align*}
 \theta_{i}=A_\mathsf{FTL}((\theta_0,l_0),\dots,(\theta_{i-1},l_{i-1})) &= \argmin_{\theta\in\Theta} \sum_{j=0}^{i-1} l_j(\theta)\\
 &= \argmin_{\theta\in\Theta} L(\theta;\cD_c)+ N\cdot C_R\cdot R(\theta)\\
 &~~~~ +\sum_{j=1}^{i-1}  l(\theta;x^*_i,y^*_i)+ C_R\cdot R(\theta)\\
 &= \argmin_{\theta\in\Theta} L(\theta;\cD_c\cup \cD_p^{i-1})+ (N+i-1)\cdot C_R\cdot R(\theta)\\
 %&=M(\cD_c\cup \cD_p^{i-1})
 &=\argmin_{\theta\in\Theta}\frac{1}{|\cD_c\cup\cD_p^{i-1}|}L(\theta;\cD_c\cup \cD_p^{i-1})+C_R\cdot R(\theta)
 \end{align*}

This means that $A_\mathsf{FTL}$ algorithm, at each step, trains a new model over the combination of clean data and poison data so far ($i-1$ number of poisons). Now we want to see what is the translation of the $\Regret(A_\mathsf{FTL},S_{\theta_p},T)$. If we can prove an upper bound on regret, namely if we show $\Regret(A_\mathsf{FTL},S_{\theta_p},T)\leq \alpha(T)$ for some function $\alpha$, then we have
 
 \begin{align*}
 \sum_{j=0}^T l_j(\theta_j) - \sum_{j=0}^T l_j(\theta_p) \leq \sum_{j=0}^T l_j(\theta_j) - \min_{\theta \in \Theta} \sum_{j=0}^T l_j(\theta) \leq \alpha(T)
 \end{align*}
 
 which implies
 
 \begin{align*}
 \sum_{j=0}^T l_j(\theta_j) - \sum_{j=0}^T l_j(\theta_p) &= L(\theta_c; D_c) - L(\theta_p;D_c) + N\cdot C_R\cdot(R(\theta_c) -R(\theta_p))\\
 &~~~~ +   \sum_{j=1}^T l_j(\theta_j) - \sum_{j=1}^T l_j(\theta_p)\\
%  \text{\Snote{should the second sum start from 1 too?}}\\
 &=L(\theta_c; D_c) - L(\theta_p;D_c)+ N\cdot C_R\cdot(R(\theta_c) -R(\theta_p))\\
%  \text{\Snote{is this right?}}\\
 &~~~~ +   \sum_{j=1}^T \big[\max_{x,y}\big(l(\theta_j; x,y) -  l(\theta_p; x,y)\big)+ C_R\cdot(R(\theta_{j})-R(\theta_{p}))\big]\\
 &\leq \alpha(T)
 \end{align*}
 
%  \begin{align*}
%  \sum_{j=0}^T l_j(\theta_j) - \sum_{j=0}^T l_j(\theta_p) &= L(\theta_c; D_c) - L(\theta_p;D_c) +   \sum_{j=1}^T l(\theta_j; x^*_j,y^*_j) - \sum_{j=0}^T l(\theta_p; x^*_j,y^*_j)\\
%  &=L(\theta_c; D_c) - L(\theta_p;D_c) + \sum_{j=1}^T \max_{x,y}l(\theta_j; x,y) -  l(\theta_p; x,y) \leq \alpha(T)
%  \end{align*} 
 Therefore we have
 \begin{align*}
 \sum_{j=1}^T \big[\max_{x,y} \big(l(\theta_j; x,y) -  l(\theta_p; x,y)\big) + C_R\cdot(R(\theta_{j})-R(\theta_{p})) \big]
 &\leq \alpha(T) + L(\theta_p; D_c) - L(\theta_c;D_c)\\
 &+ N\cdot C_R\cdot (R(\theta_p) - R(\theta_c))
 \end{align*}

%  \begin{align*}
%  \sum_{j=1}^T \max_{x,y} l(\theta_j; x,y) -  l(\theta_p; x,y) 
%  \leq \alpha(T) + L(\theta_p; D_c) - L(\theta_c;D_c)
%  \end{align*}

Based on Lemma~\ref{lemma:max_loss_dff-reg_diff}, we further have
 \begin{align*}
 \sum_{j=1}^T \gamma\cdot\big(\max_{x,y}l(\theta_j; x,y) -  l(\theta_p; x,y)\big) &\leq \alpha(T) + L(\theta_p; D_c) - L(\theta_c;D_c) \\
 &+ N\cdot C_R\cdot (R(\theta_p) - R(\theta_c))
 \end{align*}

Above inequality states that average of the maximum loss difference in all previous rounds is bounded from above. Therefore, we know that among the $T$ iterations, there exist an iteration $j^*\in [T]$ (with lowest maximum loss difference) such that the maximum loss difference of $\theta_{j^*}$ is $\epsilon$-close to $\theta_p$ with respect to the loss-based distance where 

\begin{align*}
    &\epsilon=\frac{\alpha(T) + L(\theta_p; D_c) - L(\theta_c;D_c)+ N\cdot C_R\cdot (R(\theta_p) -R(\theta_c))}{T\cdot\gamma}.
\end{align*}

\end{proof}

Theorem~\ref{theorem:convergence_main} characterizes the dependencies of $\epsilon$ on $\alpha(T)$ and the constant term $L(\theta_p; D_c) - L(\theta_c;D_c)+ N\cdot C_R\cdot (R(\theta_p) -R(\theta_c))$. To show the convergence of Algorithm~\ref{algorithm}, we need to ensure $\epsilon \rightarrow 0$ when $T\rightarrow +\infty$, which implies we need to show $\alpha(T) \leq O(\sqrt{T})$. Following remark (restating 
Remark~\ref{remark:convergence} in Section~\ref{sec:converge_proof}) and its proof shows the desired convergence. 

\shortsection{Remark~\ref{remark:convergence}} {\em 
Online learning algorithms with sublinear regret bound can be applied to show the convergence. Here, we adopt the regret analysis from~\citet{mcmahan2017survey}. Specifically, $\alpha(T)$ is in the order of $O(\log T)$) and we have $\epsilon\leq O(\frac{\log T}{T})$ when the loss function is Lipschitz continuous and the regularizer $R(\theta)$ is strongly convex, and $\epsilon \rightarrow 0$ when $T\rightarrow +\infty$. $\alpha(T)$ is also in the order of $O(\log T)$ when the loss function used for training is strongly convex and the regularizer is convex. 
}

Our FTL framework formulation can utilize the existing logarithmic regret bound of adaptive FTL algorithm when the objective functions are strongly convex with respect to some norm $\|\cdot\|$, as illustrated in Section 3.6 in~\citet{mcmahan2017survey}. For clarity in presentation, we first restate their related results below.

\begin{setting}[Setting 1 in~\citet{mcmahan2017survey}]
\label{setting}
Given a sequence of objective loss functions $f_1, f_2, ..., f_i$ and a sequence of incremental regularization functions $r_0, r_1, ..., r_i$ we consider an algorithm that selects the response point based on 
\begin{align*}
&\theta_1 = \argmin_{\theta\in \R^d}r_0(\theta)\\ &\theta_{i+1} = \argmin_{\theta\in \R^d}\sum_{j=1}^{i}f_{j}(\theta)+r_j(\theta) + r_0(\theta), \textup{for}~i=1,2,...
\end{align*}

We simplify the summation notation with $f_{1:i}(\theta) = \sum_{j=1}^{i}f_j(\theta)$. Assume that $r_i$ is a convex function and satisfy $r_i(\theta) \geq 0~\textup{for}~i \in \{0,1,2,...\}$, against a sequence of convex loss functions $f_i:\R^d \rightarrow R \cup \{\infty\}$. Further, letting $h_{0:i} = r_{0:i} + f_{1:i}$ we assume \textup{dom}~$h_{0:i}$ is non-empty. Recalling $\theta_i = \argmin_{\theta} h_{0:i-1}(\theta)$, we further assume $\partial f_i(\theta_i)$ is non-empty. We denote the dual norm of a norm $\|\cdot\|$ as $\|\cdot\|_{*}$.
\end{setting}

\begin{theorem}[Restatement of Theorem 1 in~\citet{mcmahan2017survey}]
\label{theorem:theorem_ftrl}
Consider Setting~\ref{setting}, and suppose the $r_i$ are chosen such that $r_{0:i} + f_{1:i+1}$ is 1-strongly-convex w.r.t. some norm $\|\cdot\|_{(i)\cdot}$. If we define the regret of the algorithm with respect to a selected point $\theta^*$ as 
$$
\Regret_{T}(\theta^*,f_i) \equiv \sum_{i=1}^{T}f_i(\theta_i) - \sum_{i=1}^{T}f_i(\theta^*). 
$$
Then, for any $\theta^{*}\in \R^{d}$ and for any $T > 0$, with $g_i\in \partial f_i(\theta_i)$, we have 

$$\Regret_{T}(\theta^*,f_i)\leq r_{0:T-1}(\theta^*)+\frac{1}{2}\|g_i\|^{2}_{(i-1),*}$$
\end{theorem}

\begin{corollary}[Formalization of FTL result in Section 3.6 in~\citet{mcmahan2017survey}]\label{corollary:ftl}
In the FTL framework (no individual regularizer is used in the optimization procedure), suppose each loss function $f_i$ is 1-strongly convex w.r.t. a norm $\|\cdot\|$, then we have
$$
\Regret_{T}(\theta^*,f_i)\leq \frac{1}{2}\sum_{i=1}^{T}\frac{1}{i}\|g_i\|^{2}_{*}\leq \frac{G^2}{2}(1+\log T) 
$$
with $\|g_i\|_{*} \leq G$.
%\frac{1}{2}\sum_{i=1}^{T}\|g_i\|^{2}_{(i),*}
\end{corollary}

\begin{proof} {\it The following proof is a restatement of the proof in Section 3.6 in~\citet{mcmahan2017survey}.} The proof follows from Theorem~\ref{theorem:theorem_ftrl}. Since we are considering the FTL framework, let $r_i(\theta) = 0$ for all $i$ and define $\|\theta\|_{(i)} = \sqrt{i}\|\theta\|$. Observe that $h_{0:i}$ (i.e., $f_{1:i}$) is 1-strongly convex with respect to $\|\theta\|_{(i)}$ (Lemma 3 in~\citet{mcmahan2017survey}), and we have $\|\theta\|_{(i),*} = \frac{1}{\sqrt{i}}\|\theta\|_{*}$. Then by applying Theorem~\ref{theorem:theorem_ftrl}, we have
    $$
\Regret_{T}(\theta^*,f_i)\leq \frac{1}{2}\sum_{i=1}^{T}\|g_i\|^{2}_{(i),*} = \frac{1}{2}\sum_{i=1}^{T}\frac{1}{i}\|g_i\|^{2}_{*}
$$

Based on the inequality of $\sum_{i=1}^{T}1/i\leq 1+\log T$ and if we further assume $\|g_i\|_{*}\leq G$, then we can have
$$
\frac{1}{2}\sum_{i=1}^{T}\frac{1}{i}\|g_i\|^{2}_{*}\leq \frac{G^2}{2}(1+\log T) 
$$
\end{proof}

\begin{proof}[Proof of Remark~\ref{remark:convergence}]
We will prove the logarithmic regret bound in Remark~\ref{remark:convergence} utilizing Corollary~\ref{corollary:ftl}. First of all, our online learning process fits into Setting~\ref{setting}. Specifically, we set $r_i(\theta)=0$ for all $i$. For $f_i(\theta)$, when $1\leq i\leq N$, we set $f_i(\theta)=\frac{1}{N}L(\theta;\cD_c)+C_R\cdot R(\theta)$ (evenly distributing the term $L(\theta;\cD_c) + N\cdot C_R \cdot R(\theta)$ across $N$ iterations) and when $i\geq N+1$, we set $f_i(\theta)=l_{i-N}(\theta)$. Details of $l_i$ can be referred from the proof of Theorem~\ref{theorem:convergence_main}. Therefore, $f_i$ is 1-strongly convex with respect to a norm $\|\cdot\|$ (the norm is determined by the regularizer $R(\theta)$ and $C_R$). Further, $l_{0:i}(\theta) = f_{1:N+i}(\theta)$. In addition, the assumption that \textup{dom}~$h_{0:i}$ is non-empty in Setting~\ref{setting} means when if we train a classifier on the poisoned data set, we can always return a model and hence the assumption is satisfied. The assumption of the existence of subgradient $\partial f_i(\theta_i)$ in Setting~\ref{setting} is also satisfied by the poisoning attack scenario. 

The logarithmic regret of $\Regret(A_\mathsf{FTL},S_{\theta_p},T)$ of our algorithm then follows from the result of $\Regret_{T}(\theta^*,f_i)$ in Corollary~\ref{corollary:ftl}. Specifically, $l_{0:i}(\theta) = f_{1:N+i}(\theta)$ is 1-strongly convex to norm $\|\cdot\|_{i} = \sqrt{N+i}\|\cdot\|$ and since we assume the loss function is $G$-Lipschitz, we have $\|g_i\|_{*}\leq G$. Therefore, we have the logarithmic regret bound as:

$$
\Regret(A_\mathsf{FTL},S_{\theta_p},T) \leq \alpha(T) = \frac{1}{2}\sum_{i=1}^{T}\frac{1}{i+N}\|g_i\|_{*}^{2}\leq \frac{1}{2}\sum_{i=1}^{T}\frac{1}{i}\|g_i\|_{*}^{2} \leq  \frac{G^2}{2}(1+\log T) \leq O(\log T).
$$

%based on the assumption, the objective function $l_i$ in our algorithm are 1-strongly convex with respect to a norm $\|\cdot\|$ ($l_0$ is an exception that is $\sqrt{N}$-strongly convex, but does not impact the subsequent results on $l_{0:i}$) and therefore $l_{0:i}$ is 1-strongly convex to norm $\|\cdot\|_{i} = \sqrt{N+i}\|\cdot\|$. 
%$\|\cdot\|$ (for $l_0, \sigma = 2NC_R$, for $l_i$ with $i\geq 1, \sigma = 2C_R$). 
% The norm of the gradient $g_i\in \partial l_i$ is also bounded in practice. For example, in the case of Hinge loss and $\ell_2$-regularizer for the SVM model, $\|g_i\|_2$ is bounded 
% by the Lipschitz constant of $l_i$ w.r.t $\|\cdot\|_2$ (Lemma 2.6 in~\citet{shalev2012online}), and $l_i$ is Lipschitz continuous when the data point $x$ has bounded norm. Therefore, utilizing Corollary~\ref{corollary:ftl}, we have the logarithmic regret bound of our algorithm as 
% $$
% \Regret(A_\mathsf{FTL},S_{\theta_p},T) \leq \alpha(T) = \frac{G^2}{2}(1+\log T) \leq O(\log T).
% $$
\end{proof}

We next provide the proof of the certified lower bound (restating Theorem~\ref{thm:lowerbounds} from Section~\ref{sec:lowerbound}): 

\paragraph{Theorem \ref{thm:lowerbounds}.}{\em
    Given a target classifier $\theta_{p}$, to reproduce $\theta_{p}$ by adding the poisoning set $\mathcal{D}_{p}$ into $\mathcal{D}_{c}$, the number of poisoning points $|\mathcal{D}_{p}|$ cannot be lower than
    }
    $$\sup_{\theta}z(\theta) = \frac{L(\theta_{p};\mathcal{D}_{c})-L(\theta;\mathcal{D}_{c}) + NC_{R}(R(\theta_p)-R(\theta))}{\sup_{x,y}\big(l(\theta;x,y)-l(\theta_{p};x,y)\big) + C_{R}(R(\theta)-R(\theta_{p}))}.$$

The main intuition behind the theorem is, when the the number of poisoning points added to the clean training set is lower than the certified lower bound, for structural empirical risk minimization problem (shown in~\eqref{eq:erm} in the main paper), then target classifier will always have higher loss than another classifier and hence cannot be achieved. 

\begin{proof}
    We first show that for all models $\theta$, we can derive a lower bound on the number of poison points required to get $\theta_p$. Then since these lower bounds all hold, we can take the maximum over all of them and get a valid lower bound. We first show that for any model $\theta$, the minimum number of poisoning points cannot be lower than
    $$z(\theta)=\frac{L(\theta_{p};\mathcal{D}_{c})-L(\theta;\mathcal{D}_{c}) + NC_{R}(R(\theta_p)-R(\theta))}{\sup_{x,y}\big(l(\theta;x,y)-l(\theta_{p};x,y)\big) + C_{R}(R(\theta)-R(\theta_{p}))}.$$
    
    %The lower bound computes the lower bound on a pair of models $(\theta,\theta_p)$ and the takes the supremium overall the $\theta\in \Theta$. Therefore, we first prove the lower bound computed for a pair of models $(\theta,\theta_p)$ and then explain the requirement of taking the supremium. 
    
    Let us denote the point corresponding to the supremum of the loss difference between $\theta$ and $\theta_p$ as $(x^*,y^*)$~\footnote{In practice, the data space $\mathcal{X}$ is a closed convex set and hence, we can find $(x^*,y^*)$ using convex optimization. In other words, as we saw in experiments, calculating the lower bound is possible in practical scenarios.}. Namely, $l(\theta;x^*,y^*)-l(\theta_{p};x^*,y^*) = \sup_{x,y}\big(l(\theta;x,y)-l(\theta_{p};x,y)\big)$. Now suppose we can obtain $\theta_{p}$ with lower number of poisoning points $\underline{z} < z(\theta)$. Assume there is a poisoning set $\cD_p$ with size $\underline{z}$ such that when added to $\cD_c$ would result in $\theta_p$. We have 
    \begin{align*}
        \sup_{x,y} \big(l(\theta; x,y) - l(\theta_p; x,y)\big) \geq \frac{1}{|\cD_c\cup\cD_p|} L(\theta; \cD_c\cup \cD_p) - &\frac{1}{|\cD_c\cup \cD_p|} L(\theta_p; \cD_c\cup\cD_p)\\ 
        & > C_R\cdot\big(R(\theta_p) - R(\theta)\big),
    \end{align*}

 implying $\sup_{x,y} \big(l(\theta; x,y) - l(\theta_p; x,y)\big) + C_R\cdot(R(\theta)-R(\theta_p)) > 0$. Based on the assumption that $\underline{z} < z(\theta)$, and the fact that $\sup_{x,y} \big(l(\theta; x,y) - l(\theta_p; x,y)\big) + C_R\cdot(R(\theta)-R(\theta_p)) > 0$, we have  
    \begin{align*}
         \underline{z} \cdot\big(l(\theta;x^*,y^*)-l(\theta_{p};x^*,y^*)+C_{R}(R(\theta) - R(\theta_p))\big) & < z(\theta)\cdot\big(l(\theta;x^*,y^*)-l(\theta_{p};x^*,y^*) + C_{R}(R(\theta) - R(\theta_p))\big)\\ 
        & = L(\theta_{p};\mathcal{D}_{c})-L(\theta;\mathcal{D}_{c}) + NC_{R}(R(\theta_p) - R(\theta)).     
    \end{align*}
    where the equality is based on the definition of $z(\theta)$. On the other hand, by definition of $(x^*,y^*)$ for any $D_p$ of size $\underline{z}$, we have
    \begin{align*}
    & L(\theta;D_p) - L(\theta_p,D_p) + \underline{z}\cdot(C_R\cdot R(\theta) - C_R\cdot R(\theta_p))\leq \underline{z}\cdot\big(l(\theta;x^*,y^*)-l(\theta_{p};x^*,y^*)+C_{R}(R(\theta) - R(\theta_p))\big).    
    \end{align*}

    The above two inequalities imply that for any set $D_p$ with size $\underline{z}$ we have
    $$\frac{1}{|\cD_c\cup\cD_p|}L(\theta;\cD_c\cup\cD_p)+ C_R\cdot R(\theta)<\frac{1}{|\cD_c\cup\cD_p|}L(\theta_p;\cD_c\cup\cD_p) + C_R\cdot R(\theta_p).$$
    which indicates that adding $\cD_p$ poisoning points into the training set $\cD_c$, the model $\theta$ has lower loss compared to $\theta_{p}$, which is a contradiction to the assumption that $\theta_p$ has lowest loss on $\cD_c\cup\cD_p$ and can be achieved. Now, since $\theta_p$ needs to have lower loss on $\cD_c\cup\cD_p$ compared to any classifier $\theta\in\Theta$, the best lower bound is the supremum over all models in the model space $\Theta$.
    \label{proof:min_point}
\end{proof}

\shortsection{Corollary~\ref{cor:eps_lower_bound}} {\em 
If we further assume bi-directional closeness in the loss-based distance, we can also derive the lower bound on number of poisoning points needed to induce models that are $\epsilon$-close to the target model. More precisely, if $\theta_1$ being $\epsilon$-close to $\theta_2$ implies that $\theta_2$ is also $k\cdot \epsilon$ close to $\theta_1$, then we have, 
$$
\sup_{\theta}z^{'}(\theta) = \frac{L(\theta_{p};\mathcal{D}_{c})-L(\theta;\mathcal{D}_{c}) - NC_{R}\cdot R^{*} -Nk\epsilon}{\sup_{x,y}\big(l(\theta;x,y)-l(\theta_{p};x,y)\big) + C_{R}\cdot R^{*} + k\epsilon}.
$$
where $R^*$ is an upper bound on the nonnegative regularizer $R(\theta)$.
}

\begin{proof}[Proof of Corollary 4.2.1]
% The proof is straight forward. First recall that bidirectional-closeness says if $\theta_1$ is $\epsilon$-close to $\theta_2$, then $\theta_2$ is $O(\epsilon)$-close to $\theta_1$. 
%We omit the asymptotic notion and assume $\theta_2$ is $k\epsilon$-close to $\theta_2$, that is $\max_{(x,y) \in \cX\times \cY} l(\theta_2; x, y) - l(\theta_1; x,y)\leq k\epsilon$. 

The lower bound for all $\epsilon$-close models to the target classifier is given exactly as follows:
    $$\inf_{\|\theta^{'}-\theta_p\|_{\cD_{l,\cX,\cY}}\leq \epsilon}\sup_{\theta}\Bigg(z(\theta,\theta') = \frac{L(\theta^{'};\mathcal{D}_{c})-L(\theta;\mathcal{D}_{c}) + NC_{R}(R(\theta^{'})-R(\theta))}{\sup_{x,y}\big(l(\theta;x,y)-l(\theta^{'};x,y)\big) + C_{R}(R(\theta)-R(\theta^{'}))}\Bigg),$$

where $\inf_{\|\theta^{'}-\theta_p\|_{\cD_{l,\cX,\cY}}\leq \epsilon}$ denotes $\theta^{'}$ is $\epsilon$-close to $\theta_p$ in the loss-based distance. However, the formulation above is a min-max optimization problem and hard to analytically compute the lower bound (by plugging the lower bound formula into Algorithm~\ref{algorithm}. Therefore, we need to make several relaxations such that the lower bound is computable. For any model $\theta^{'}$ that is $\epsilon$-close to $\theta_p$, based on the bi-directional assumption, then $\theta_p$ is $k\epsilon$-close to $\theta^{'}$. Therefore we have, 
$$
L(\theta^{'};\cD_c)-L(\theta;\cD_c) = L(\theta^{'};\cD_c) - L(\theta_{p};\cD_c) + L(\theta_{p};\cD_c) - L(\theta;\cD_c) \geq -Nk\epsilon + L(\theta_{p};\cD_c) - L(\theta;\cD_c)
$$
and
\begin{align*}
    \sup_{x,y}\big(l(\theta;x,y) - l(\theta^{'},x,y)\big)  &=\sup_{x,y}\big(l(\theta;x,y) - l(\theta_{p},x,y)\big) + \sup_{x,y}\big(l(\theta_{p},x,y) - l(\theta^{'};x,y)\big)\\ 
    &\leq \sup_{x,y}\big(l(\theta;x,y) - l(\theta_{p},x,y) + k\epsilon
\end{align*}
and the inequalities are all based on the definition of $\theta_p$ being $k\epsilon$-close to $\theta^{'}.$

Plugging the above inequalities into the formula of $\sup_{\theta,\theta'}$ for model $\theta^{'}$, and with the assumption that $0 \leq R(\theta)\leq R^{*}, \forall \theta\in\Theta$, we immediately have 

\begin{align*}\sup_{\theta}z(\theta,\theta') &\geq \sup_{\theta}\frac{L(\theta_{p};\mathcal{D}_{c})-L(\theta;\mathcal{D}_{c}) - Nk\epsilon + NC_{R}(R(\theta^{'})-R(\theta))}{\sup_{x,y}\big(l(\theta;x,y)-l(\theta_{p};x,y)\big) -k\epsilon + C_{R}(R(\theta)-R(\theta^{'}))}\\
&\geq
 \sup_{\theta}\Bigg(\frac{L(\theta_{p};\mathcal{D}_{c})-L(\theta;\mathcal{D}_{c}) - Nk\epsilon - NC_{R}\cdot R^*}{\sup_{x,y}\big(l(\theta;x,y)-l(\theta_{p};x,y)\big) -k\epsilon + C_{R}\cdot R^*}=z^{'}(\theta)\Bigg).
\end{align*}

Since the inequality holds for any $\theta^{'}$, we have 
\begin{align*}\inf_{\|\theta^{'}-\theta_p\|_{\cD_{l,\cX,\cY}}\leq \epsilon}\sup_{\theta}z(\theta,\theta^{'}) \geq \sup_{\theta}z^{'}(\theta)
\end{align*}
and hence $z^{'}(\theta)$ is a valid lower bound. 

% However, in $z^{'}(\theta)$, there are terms related to $R(\theta^{'})$ and hence still not easy to compute the lower bound by plugging into Algorithm~\ref{algorithm}. If we further assume $0< R(\theta) \leq R^{*}$, we will have a lower bound that is independent from $R(\theta^{'})$ as follows:
% $$
% \sup_{\theta}z^{'}(\theta) = \frac{L(\theta_{p};\mathcal{D}_{c})-L(\theta;\mathcal{D}_{c}) - NC_{R}R^{*} -Nk\epsilon}{\sup_{x,y}\big(l(\theta;x,y)-l(\theta_{p};x,y)\big) + C_{R}R^{*} + k\epsilon}.
% $$

\end{proof}

\begin{remark}[Improving Results in Corollary~\ref{cor:eps_lower_bound}]
Assuming $0\leq R(\theta) \leq R^{*}$ is not a strong assumption and actually can be satisfied by many common convex models. For example, for SVM model with $\ell_2$-regularizer (in fact, applies to any regularizer $R(\theta)$ with $R(\mathbf{0})=0$), we have $R(\theta)\leq \frac{1}{C_R}$ and hence $R^{*}\leq \frac{1}{C_R}$. Moreover, we can further tighten the lower bound by better bounding the term $R(\theta')-R(\theta)$. Specifically, $R(\theta')-R(\theta) = R(\theta') - R(\theta_p) + R(\theta_p) - R(\theta)$ and we only need to have a tighter upper and lower bounds on $R(\theta') - R(\theta_p)$ utilizing some special properties of the loss functions. For the constant $k$ in the bi-directional closeness, we can also compute its value for some specific loss functions. For example, for Hinge loss, we can compute the value based on Corollary~\ref{cor:bidirection-close} in Appendix~\ref{sec:closeness}.
\end{remark}

\section{Relating closeness of loss-based distance to closeness of parameters}\label{sec:closeness}
\newcommand{\loss}[0]{\mathsf{loss}}
\newcommand{\inner}[2]{\langle#1,#2\rangle}
In theorem below, we show how one can relate the notion of $\epsilon$-closeness in Definition~\ref{def:eps_close} in the main paper to closeness of parameters in the specific setting of hinge loss. We use this just as an example to show that our notion of $\epsilon$-closeness can be tightly related to the closeness of the models.
\begin{theorem} \label{thm:hinge_to_norm}
Consider the hinge loss function $l(\theta; x,y) = \max(1-y\cdot\inner{x}{\theta},0)$ for $\theta\in \R^d$ and $x\in \R^d$ and $y\in\set{-1,+1}$. For $\theta, \theta' \in \R^d$ such that $\|\theta\|_1\leq r$ and $\|\theta'\|_1\leq r$, if $\theta$ is $\epsilon$-close to $\theta'$ in the loss-based distance, then,
$\|\theta-\theta'\|_1\leq r\cdot \epsilon$. 
%Furthermore, if $\|\theta-\theta^{'}\|_1\leq \epsilon$, then $\theta$ is $\frac{\epsilon}{r}$-close to $\theta^{'}$.
\end{theorem}
\begin{remark}
\label{rm:norm_upper}
In Theorem \ref{thm:hinge_to_norm} above with $\ell_2$-regularizer, an upper bound on the $\ell_1$-norm of $\theta$ and $\theta'$ is $\sqrt{d/C_R}$. however, the models that we care about in practice usually have smaller norms.
 \end{remark}
 
Remark~\ref{rm:norm_upper} can be obtained by plugging $\mathbf{0}\in\R^d$ and compare the resulting (regularized) optimization loss to the model $\theta^{*}$ that minimizes the model loss.
\begin{proof}[Proof of Theorem~\ref{thm:hinge_to_norm}]
We construct a point $x^*$ as follows:
\begin{align*}
    x^*_i = \begin{cases}
    -\frac{1}{r}, \text{ if } \theta_i > \theta_i', i\in[d]\\
    +\frac{1}{r} \text{ if } \theta_i \leq \theta_i', i\in[d]
    \end{cases}
\end{align*}
    Then we have 
    \begin{equation}\label{eq:norm_to_hinge}
        \inner{\theta-\theta'}{x^*} = \frac{1}{r}\cdot\|\theta-\theta'\|_1
    \end{equation}
    Since $\|\theta\|_1\leq r$ we have
    \begin{equation}\label{eq:0001}
        \inner{x^*}{\theta} \geq -1
    \end{equation}
    and similarly since $\|\theta'\|_1\leq r$
    we have 
    \begin{equation}\label{eq:0002}
        \inner{x^*}{\theta'} \geq -1.
    \end{equation}
    Therefore by Inequalities~\eqref{eq:0001} and~\eqref{eq:0002} we have
    \begin{align*}l(\theta; x^*,-1) - l(\theta';x^*,-1) &= \max(1 + \inner{x^*}{\theta},0) - \max(1 + \inner{x^*}{\theta'},0) = \inner{\theta-\theta'}{x^*} 
    \end{align*}
    which by ~\eqref{eq:norm_to_hinge} implies 
        \begin{equation}\label{eq:0003}
        l(\theta; x^*,-1) - l(\theta';x^*,-1) = \frac{1}{r}\cdot \|\theta-\theta'\|_1.
    \end{equation}
    Now since we know that, $\forall x\in \R^d$, the loss difference between $\theta$ and $\theta'$ is bounded by $\epsilon$, the bound should also hold for the point $(x^*, -1)$, meaning that
    $$\frac{1}{r}\cdot\|\theta-\theta'\|_1 \leq \epsilon.$$
    which completes the proof.
\end{proof}

\begin{theorem} \label{thm:norm_to_hinge}
Consider the hinge loss function $l(\theta; x,y) = \max(1-y\cdot\inner{x}{\theta},0)$ for $\theta\in \R^d$ and $x\in \R^d$ and $y\in\set{-1,+1}$. For $\cX = \set{x \in \R^d\colon \|x\|_1\leq q}$ and $\cY=\set{-1,+1}$,
%For $\theta, \theta' \in \R^d$ such that $\|\theta\|_1\leq r$ and $\|\theta'\|_1\leq r$ 
%for some $r\geq 1$, 
For any two models $\theta, \theta^{'}$ if $\|\theta-\theta^{'}\|_1\leq \epsilon$, then $\theta$ is $q\cdot\epsilon$-close to $\theta^{'}$ in the loss-based distance. Namely, $$D_{\ell, \cX, \cY}(\theta,\theta') \leq q\cdot \epsilon.$$
\end{theorem}

\begin{proof}
    For any given $\theta$ and $\theta^{'}$, by triangle inequality for maximum, we have $$
    l(\theta;x,y) - l(\theta^{'},x,y) = \max(1-y\cdot\inner{x}{\theta},0) - \max(1-y\cdot\inner{x}{\theta^{'}},0) \leq \max(0,\inner{yx}{\theta^{'}-\theta}).$$ Therefore, we have  $$\max_{(x,y) \in \cX\times \cY} l(\theta; x, y) - l(\theta^{'}; x,y) \leq \max_{(x,y) \in \cX\times \cY} \max(0,\inner{yx}{\theta^{'}-\theta}).$$ Our goal is then to obtain an upper bound of $O(\epsilon)$ for $\max_{(x,y) \in \cX\times \cY} \inner{yx}{\theta^{'}-\theta}$ when $\|\theta-\theta^{'}\|_1\leq \epsilon.$ To maximize $\inner{yx}{\theta^{'}-\theta}$ by choosing $x$ and $y$, we only need to ensure that $\sign yx_i = \sign \theta_i, i\in [d]$. Therefore, based on the assumption that $\frac{1}{q}\|x\|\leq 1$ (i.e.,$ \frac{1}{q}|x_i|\leq 1, i\in [d])$ we have 
    $$
    \max_{(x,y) \in \cX\times \cY}\frac{1}{q}\inner{yx}{\theta^{'}-\theta} =\sum_{i=1}^{d}\frac{1}{q}|x|_{i}|\theta_i-\theta^{'}_{i}|\leq \sum_{i=1}^{d}|\theta_i-\theta^{'}_{i}| = \|\theta - \theta^{'}\|_1 \leq \epsilon,
    $$
    which concludes the proof. 
\end{proof}

\begin{corollary}
\label{cor:bidirection-close}
    For Hinge loss, with Theorem~\ref{thm:hinge_to_norm} and Theorem~\ref{thm:norm_to_hinge}, if $\theta$ is $\epsilon$-close to $\theta^{'}$, then $\theta^{'}$ is $r\cdot q\cdot\epsilon$-close to $\theta$.   
\end{corollary}

\section{Instantiating Theorem~\ref{theorem:convergence_main} for the Case of SVM}\label{sec:compute_const}
Here we show how to instantiate Theorem \ref{theorem:convergence_main} for SVM with exact constants instead of the asymptotic notations. We need to calculate the constant $\gamma$ to get the exact constant. Imagine the feature domain is $\R^d$. Now we calculate the constant $C$ as follows. Let $i^*_\theta = \argmin_{i\in[d]}|\theta[i]/\theta_p[i]|$ and $\alpha_\theta=|\theta[i^*_\theta]/\theta_p[i^*_\theta]|$. Let $x_\theta^*\in \R^d$ be a point where is equal to 0 everywhere and is equal to $1/\theta_p[i_\theta^*]$ on the $i^*$ coordinate. We have, \begin{equation}\label{ineq:sup_lower} l(\theta,x_\theta^*,+1) - l(\theta_p,x_\theta^*,+1)=l(\theta,x_\theta^*,+1)\geq(1-\alpha_\theta).\end{equation}
Now we can calculate $C$ as follows
\begin{align*}
    C&=\left(\inf_{\substack{\theta\in F\\
\text{s.t. } R(\theta_p) > R(\theta) }}\sup_{x,y} (l(\theta;x,y)-l(\theta_p;x,y))/(R(\theta_p) - R(\theta))\right)\\
 &\geq \left(\inf_{\substack{\theta\in F\\
\text{s.t. } R(\theta_p) > R(\theta) }} (l(\theta,x_\theta^*,+1) - l(\theta_p,x_\theta^*,+1))/(R(\theta_p) - R(\theta))\right)\\
\text{(By Inequality \ref{ineq:sup_lower})~~~~}
  &\geq \inf_{\substack{\theta\in F\\
\text{s.t. } R(\theta_p) > R(\theta) }}\sup_{x,y} \frac{1-\alpha_\theta}{R(\theta_p)-R(\theta)}\\
\text{(By definition of $\alpha_\theta$)~~~~}
 &\geq \inf_{\substack{\theta\in F\\ \text{s.t. } R(\theta_p) > R(\theta) }} \frac{1-\alpha_\theta}{R(\theta_p)(1-\alpha_\theta^2)}\\
 &\geq \inf_{\substack{\theta\in F\\ \text{s.t. } R(\theta_p) > R(\theta) }} \frac{1-\alpha_\theta}{R(\theta_p)(1-\alpha_\theta^2)}\\
&\geq \frac{1}{2R(\theta_p)}
\end{align*}
Therefore $\gamma\geq 1-2\cdot C_R\cdot R(\theta_p)$. On the other hand, we can also calculate $\alpha(T)$ based on the exact form given in the proof of Theorem~\ref{theorem:convergence_main}. 
%For hinge loss, the value of $G$ is the maximum of the norm of data point $x\in \R^d$. 

%at most $\sqrt{d}$, which implies that the online regret is bounded by $G^2/(2\gamma) \leq d/(2-4 C_R \cdot R(\theta_p))$. 
% \begin{equation}\label{ineq:l2_to_l1}
%     |R(\theta_1) - R(\theta_2)| \leq \frac{2 \|\theta_1 - \theta_2\|_1}{\sqrt{C_R}}
% \end{equation}

% We know that the maximum loss difference between two SVM models $\theta_p$ and $\theta_2$ is at least $1$, because the . Therefore, combining Inequality \ref{ineq:l2_to_l1} with Theorem \ref{thm:hinge_to_norm} and by definition of $C$ we have 

% \begin{align*}
%     C&=\left(\inf_{\substack{(\theta_1,\theta_2)\in F^2\\
% \text{s.t. } R(\theta_1) > R(\theta_2) }}\sup_{x,y} (l(\theta_2;x,y)-l(\theta_1;x,y))/(R(\theta_1) - R(\theta_2))\right)\\
% \text{(By Theorem \ref{thm:hinge_to_norm})~~~~} &\geq \left(\inf_{\substack{(\theta_1,\theta_2)\in F^2\\
% \text{s.t. } R(\theta_1) > R(\theta_2) }} (\sqrt{d}\cdot\|\theta_1 - \theta_2\|_1)/(\sqrt{CR}\cdot(R(\theta_1) - R(\theta_2)))\right)\\
% \text{(By Inequality \ref{ineq:l2_to_l1})~~~~}
%  &\geq \left(\inf_{\substack{(\theta_1,\theta_2)\in F^2\\
%  \text{s.t. } R(\theta_1) > R(\theta_2) }} \frac{\sqrt{C_R}}{2}\right)=\frac{\sqrt{C_R}}{2}
% \end{align*}

\section{Additional Experimental Results}
\label{sec:additional_exps}
In this section, we provide more results in addition to the results in the main paper. In Section~\ref{sec:addi_svm_results}, we show the additional results on SVM model and more results on logistic regression model are given in Section~\ref{sec:addi_lr_results}. In Section~\ref{ssec:bettertarget}, we show results on improved target model generation process, which helps to validate the implications we made (below Theorem~\ref{theorem:convergence_main}) in the main paper.  

\subsection{More Results on SVM model}
\label{sec:addi_svm_results}
In this section, we first compare our attack to the KKT attack regarding the convergence to the target model. Then compare their attack success in achieving the attacker goals. Last, we provide the lower bound for inducing the model that are induced by our attack and the KKT attack. We use the exact same setup in Section~\ref{sec:experiments} in the main paper regarding the datasets and related models.

    \begin{figure*}[!tb]
        \centering
        \begin{subfigure}[b]{0.45\textwidth} 
            \centering 
            \includegraphics[width=\textwidth]{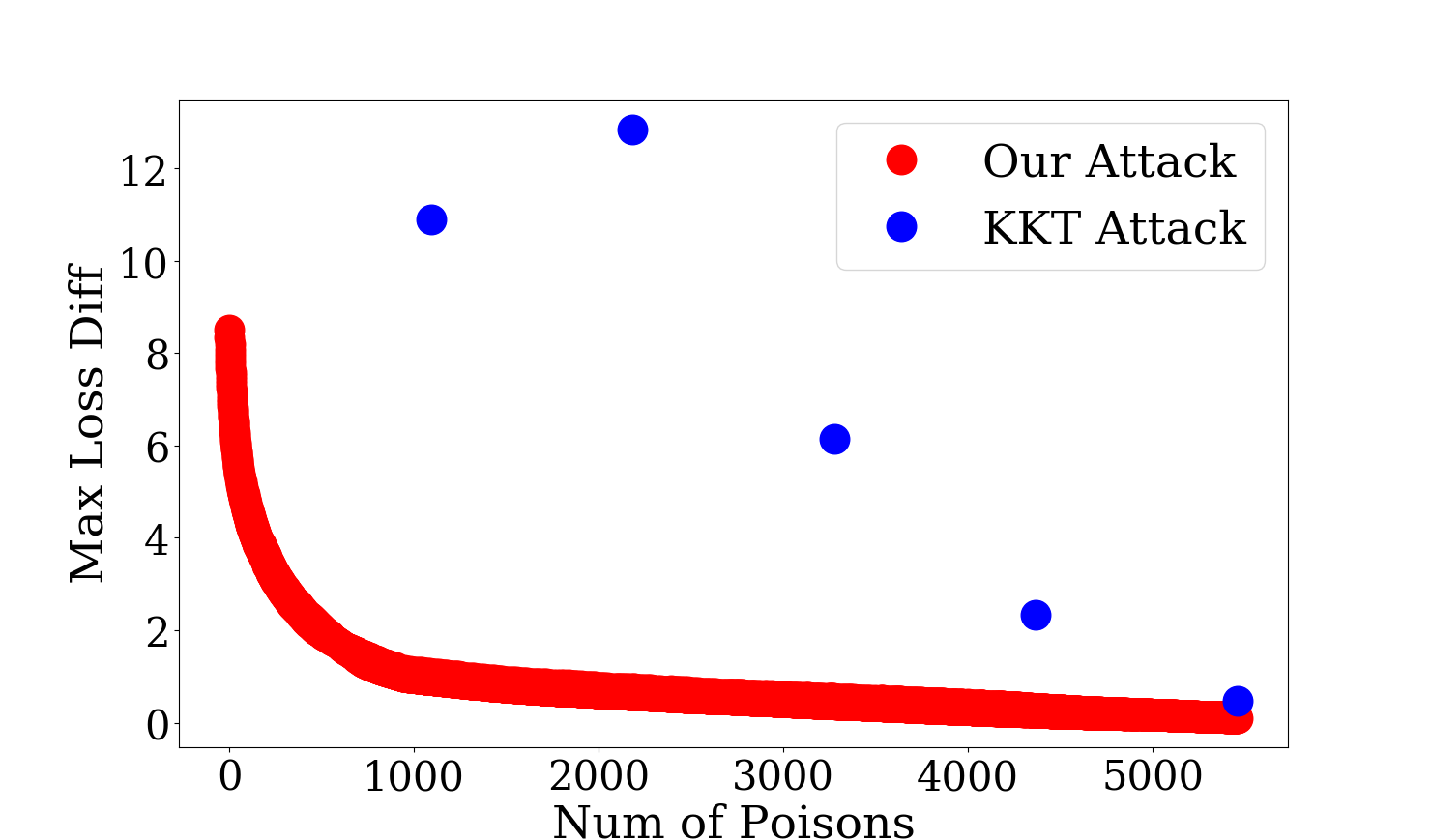}
            \caption[]%
            {Max Loss Difference}    
            \label{fig:mnist_svm_error_01_max_loss_diff}
        \end{subfigure}
        \begin{subfigure}[b]{0.45\textwidth}  
            \centering 
            \includegraphics[width=\textwidth]{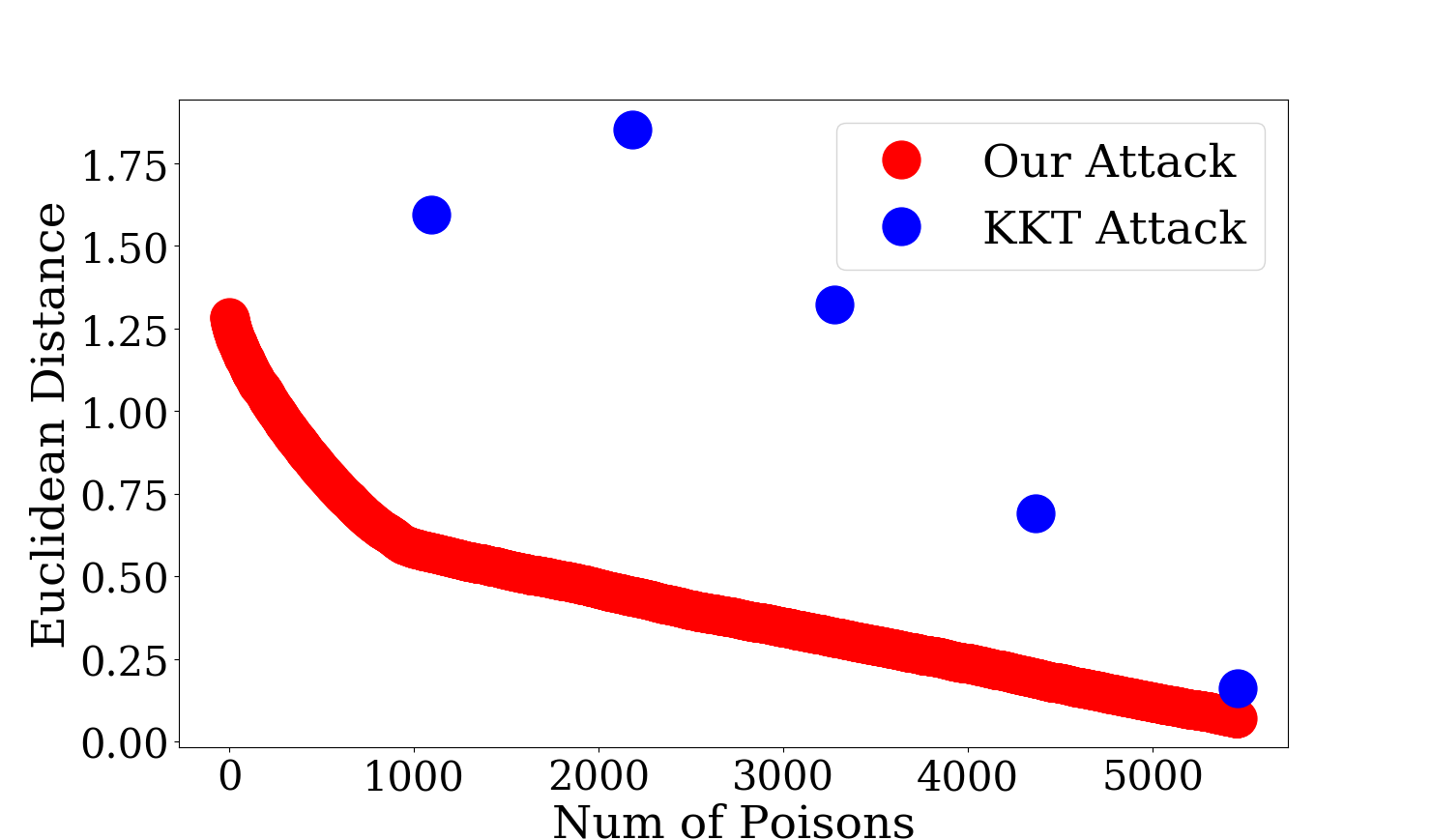}
            \caption[]%
            {Euclidean Distance}    
            \label{fig:mnist_svm_error_01_norm_diff}
        \end{subfigure}

        \caption[]
        {SVM on \MNIST\ dataset: attack convergence (results shown are for the target classifier of error rate 10\%). The maximum number of poisons is set using the $0.1$-close threshold to target classifier} 
        %\dnote{please fix the images so the axis captions are not cut off} 
        \label{fig:mnist_svm_error_01_convergence}
    \end{figure*}

    \begin{figure*}[!tb]
        \centering
        \begin{subfigure}[b]{0.45\textwidth}  
            \centering 
            \includegraphics[width=\textwidth]{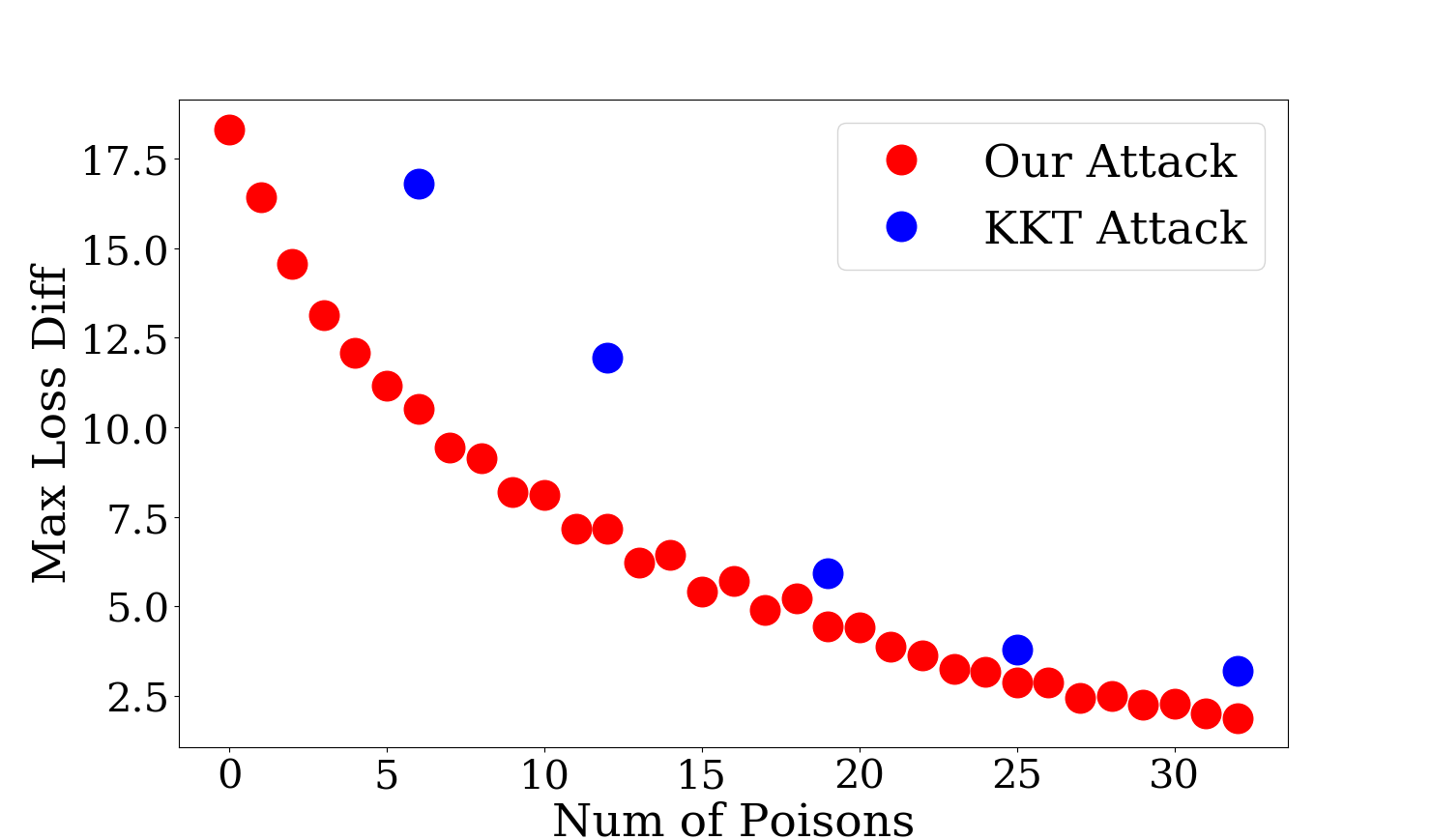}
            \caption[]%
            {Max Loss Difference}    
            \label{fig:dogfish_svm_error_01_max_loss_diff}
        \end{subfigure}
        \begin{subfigure}[b]{0.45\textwidth} 
            \centering 
            \includegraphics[width=\textwidth]{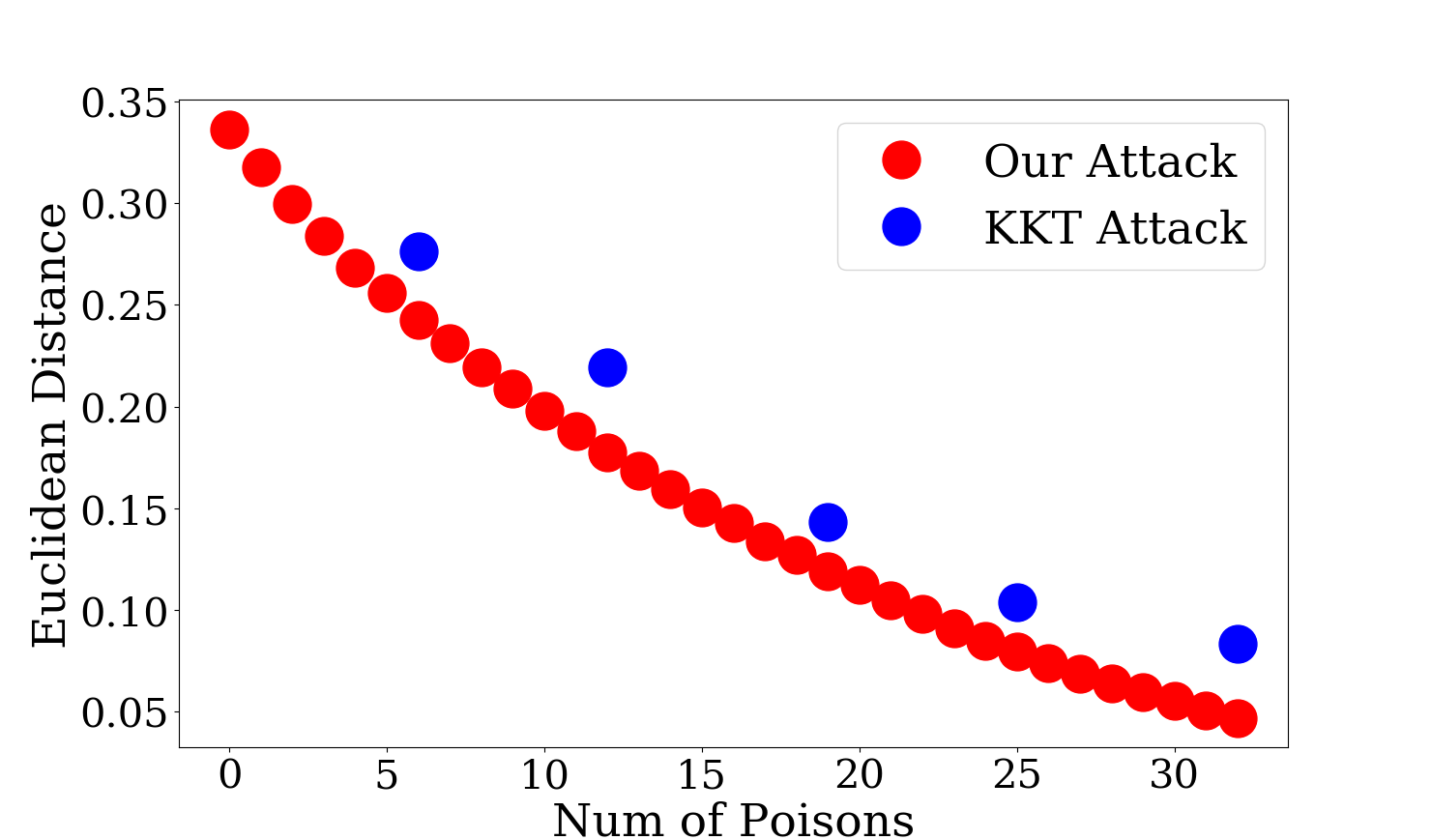}
            \caption[]%
            {Euclidean Distance}    
            \label{fig:dogfish_svm_error_01_norm_diff}
        \end{subfigure}

        \caption[]
        {SVM on Dogfish dataset: attack convergence (results shown are for the target classifier of error rate 10\%). The maximum number of poisons is set using the $2.0$-close threshold to target classifier} 
        %\dnote{please fix the images so the axis captions are not cut off} 
        \label{fig:dogfish_svm_error_01_convergence}
    \end{figure*}

% \begin{table}[tb!]
% \centering
% \begin{tabular}{c|rrc|rr}
% \toprule
% \multirow{2}{*}{Target Models} & \multicolumn{3}{c|}{Target Induced Models} & \multicolumn{2}{c}{Target Model} \\
%  & KKT Bound & Ours Bound & \# of Poisons & Bound & \# of Poisons \\ \midrule
% 5\% Error & $856$ & $874$ & $1737$ & $1035$ & $6510$ \\
% 10\% Error & $4058.4+1.4$ & $3850.4+0.8$ & $5458$ & $4231$ & $10825$ \\
% 15\% Error & $5031.4+4.8$ & $4904$ & $6192$ & $5214$ & $8648$ \\ \bottomrule
% \end{tabular}
% \vspace{0.5ex}
% \caption{Theoretical lower bounds on the number of poisoning points needed to achieve different models in indiscriminate poisoning. \# of poisons for the target model refers to the size of the poisoning set returned from the heuristic method of generating target models~\citep{koh2018stronger}. All results are averaged over 4 runs, integer value in the cell means we get exactly the same value for 4 runs and others are shown with the average and standard error.} 
% \label{tab:indiscrim_validate_lower_bound} 
% \end{table}

\shortsection{Convergence} 
We show the convergence of Algorithm~\ref{algorithm} by reporting the maximum loss difference and Euclidean distance between the classifier induced by the attack and the target classifier. Figures~\ref{fig:mnist_svm_error_01_convergence} summarizes the results on \MNIST\ dataset for the target classifier of 10\% error rate. The maximum number of poisoning points in the figure is obtained when the classifier from Algorithm~\ref{algorithm} is $0.1$-close to the target classifier in the loss-based distance. Figure~\ref{fig:dogfish_svm_error_01_convergence} shows the results on Dogfish dataset with the target classifier of 10\% error rate and the maximum number of poisoning points is obtained when the induced classifier is $2.0$-close to the target classifier. From the two figures, we observe that classifiers induced by our algorithm steadily converge to the target classifier both in the maximum loss difference and Euclidean distance, while the classifier induced by the KKT attack either cannot converge reliably (Figure~\ref{fig:mnist_svm_error_01_convergence}) or converges slower than our attack (Figure~\ref{fig:dogfish_svm_error_01_convergence}). We observe similar observations for other indiscriminate attack settings, and omitted those results here for clarity in presentation. 

%At the maximum number of points, the maximum loss difference of KKT-induced classifier to the target is 0.46, compared to 0.1 for the classifier induced by our attack. For the Euclidean distance, the KKT-induced classifier is 0.16 away, compared to 0.07 for the classifier induced by our attack.
% We then compared both attacks in terms of their loss on the clean training set, shown in Figure~\ref{fig:error_01_clean_loss}. Interestingly, although the KKT attack has larger gap to the target model compared to our attack for smaller number of poisons, the loss on clean training set becomes closer to the target model at the maximum number of poisons. This is an interesting finding and might also be able to partially explain why the KKT attack can have slightly higher attack success rate than our attack in next paragraph.  

    \begin{figure*}[!tb]
        \centering
        \begin{subfigure}[b]{0.33\textwidth}
            \centering
            \includegraphics[width=0.98\textwidth]{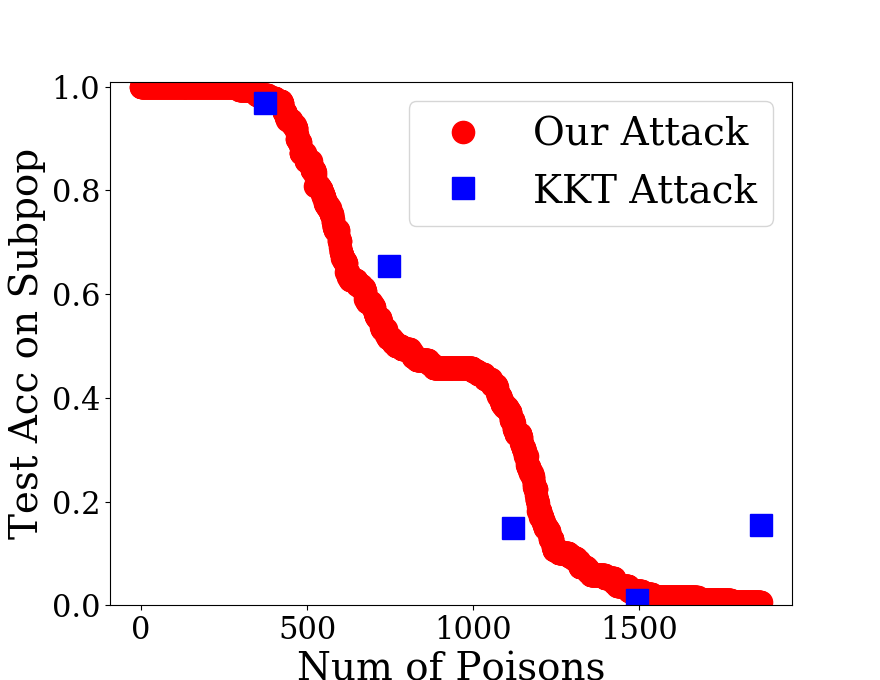}
           \caption[]
            {Cluster 0} 
            \label{fig:adult_svm_subpop0_acc_scores}
        \end{subfigure}
        %\hfill
        \begin{subfigure}[b]{0.33\textwidth}  
            \centering 
            \includegraphics[width=0.98\textwidth]{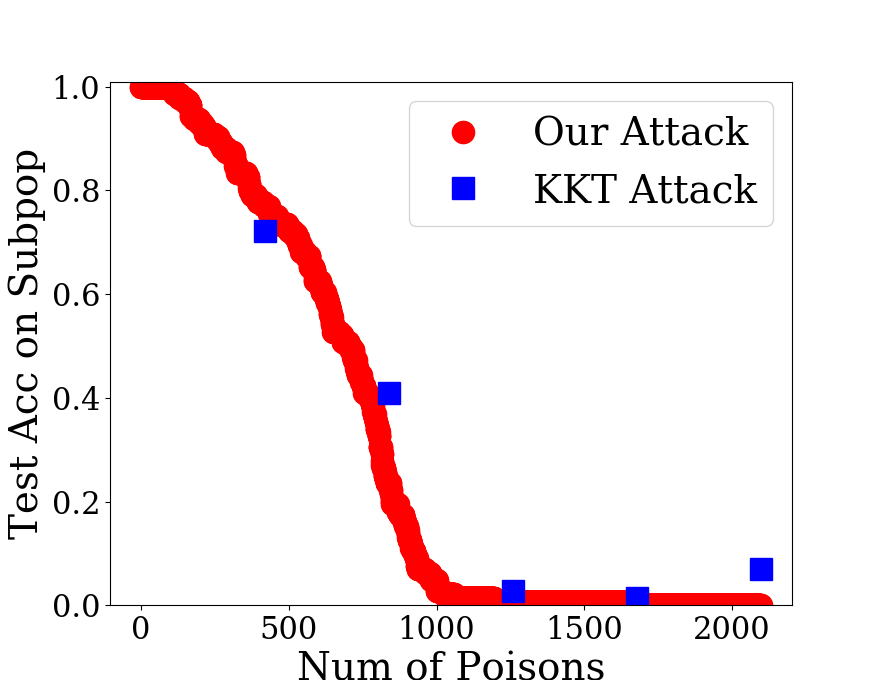}
            \caption[]%
            {Cluster 1}    
            \label{fig:adult_svm_subpop1_acc_scores}
        \end{subfigure}
        \begin{subfigure}[b]{0.33\textwidth} 
            \centering 
            \includegraphics[width=0.98\textwidth]{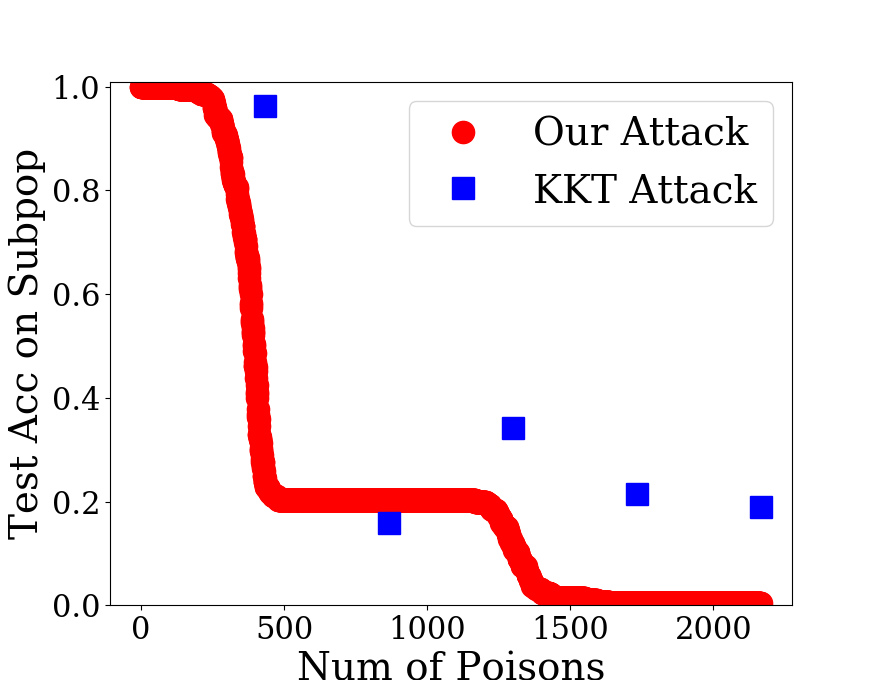}
            \caption[]%
            {Cluster 2}    
            \label{fig:adult_svm_subpop2_acc_scores}
        \end{subfigure}
        %\vskip\baselineskip
        %\quad
        \caption[]
        {SVM on Adult dataset: test accuracy for each target model of given error rate with classifiers induced by poisoning points obtained from our attack and the KKT attack.} 
        \label{fig:adult_svm_subpop_acc}
    \end{figure*}
    
    \begin{figure*}[!tb]
        \centering
        \begin{subfigure}[b]{0.33\textwidth}
            \centering
            \includegraphics[width=0.98\textwidth]{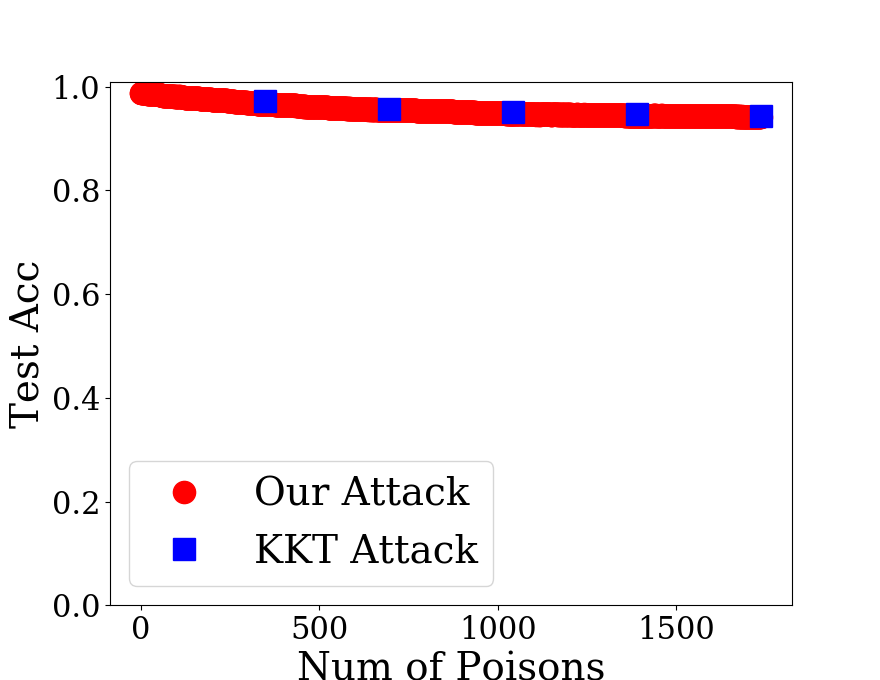}
           \caption[]
            {5\% Error Rate} 
            \label{fig:mnist_svm_error_005_acc_scores}
        \end{subfigure}
        %\hfill
        \begin{subfigure}[b]{0.33\textwidth}  
            \centering 
            \includegraphics[width=0.98\textwidth]{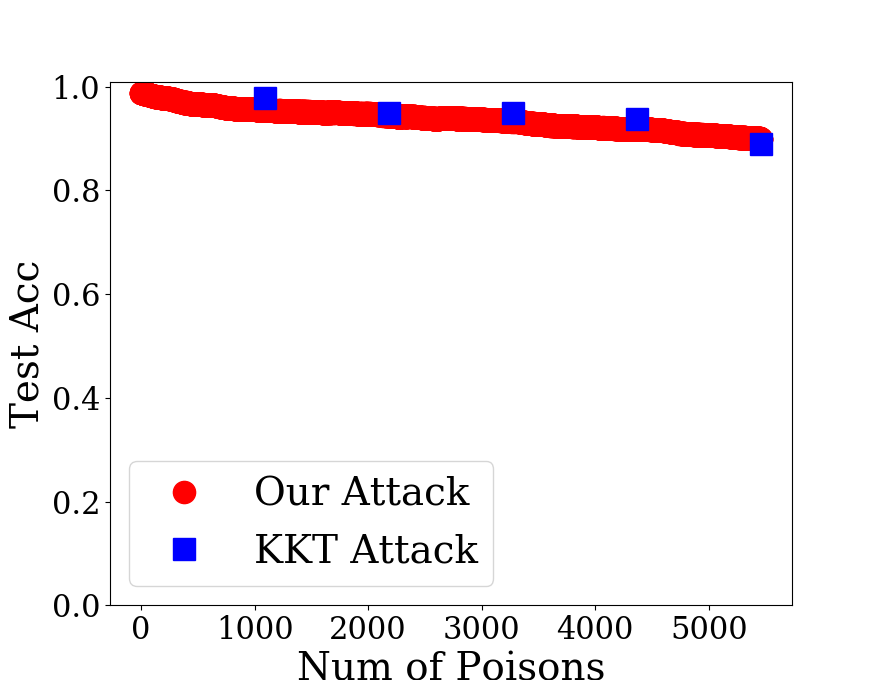}
            \caption[]%
            {10\% Error Rate}    
            \label{fig:mnist_svm_error_01_acc_scores}
        \end{subfigure}
        \begin{subfigure}[b]{0.33\textwidth} 
            \centering 
            \includegraphics[width=0.98\textwidth]{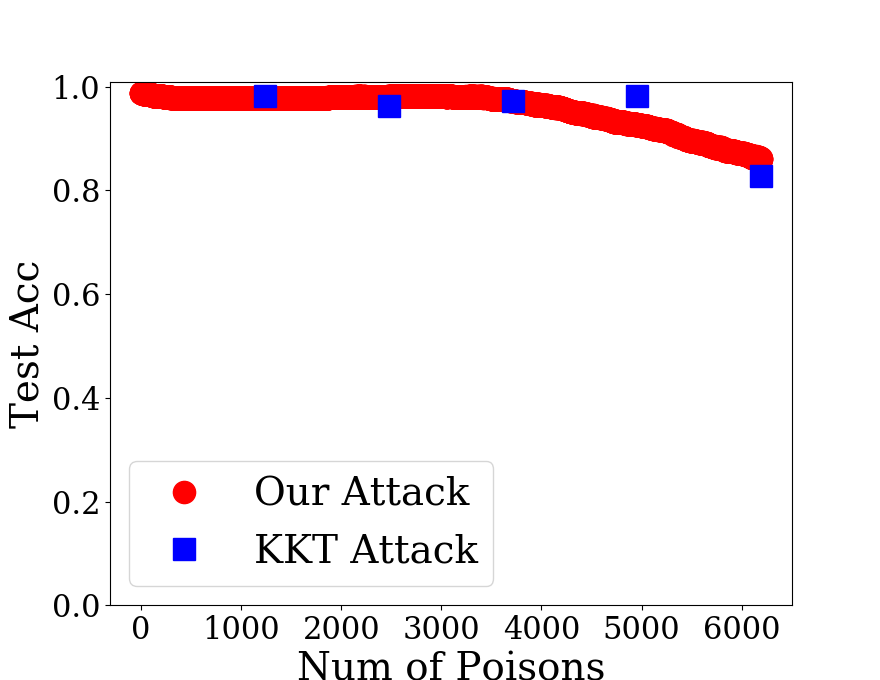}
            \caption[]%
            {15\% Error Rate}    
            \label{fig:mnist_svm_error_015_acc_scores}
        \end{subfigure}
        %\vskip\baselineskip
        %\quad
        \caption[]
        {SVM on \MNIST\ dataset: test accuracy for each target model of given error rate with classifiers induced by poisoning points obtained from our attack and the KKT attack.} 
        \label{fig:mnist_svm_indiscriminate_acc}
    \end{figure*}

    \begin{figure*}[!tb]
        \centering
        \begin{subfigure}[b]{0.33\textwidth}
            \centering
            \includegraphics[width=0.98\textwidth]{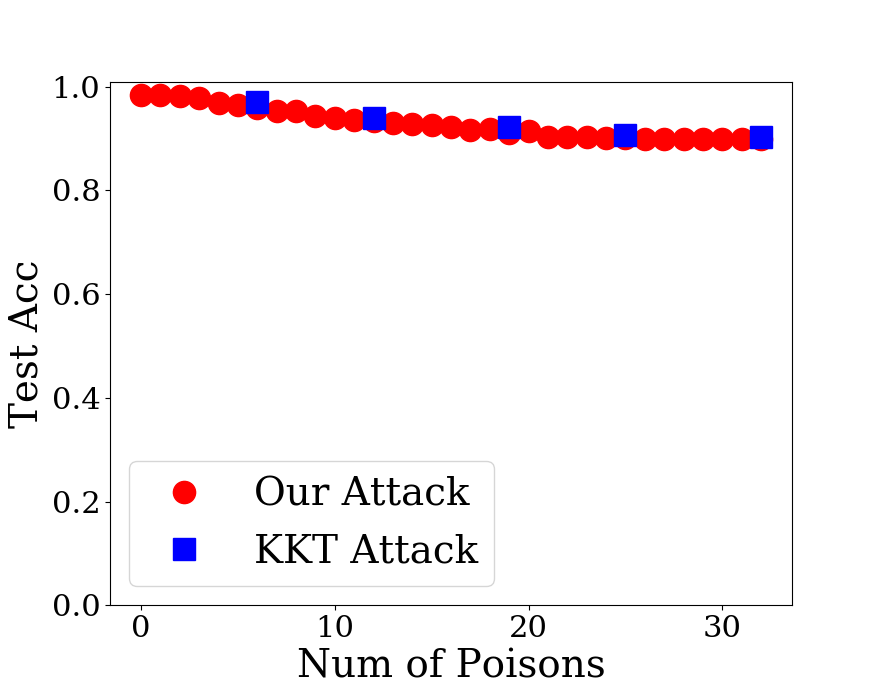}
           \caption[]
            {10\% Error Rate} 
            \label{fig:dogfish_svm_01_acc}
        \end{subfigure}
        %\hfill
        \begin{subfigure}[b]{0.33\textwidth} 
            \centering 
            \includegraphics[width=0.98\textwidth]{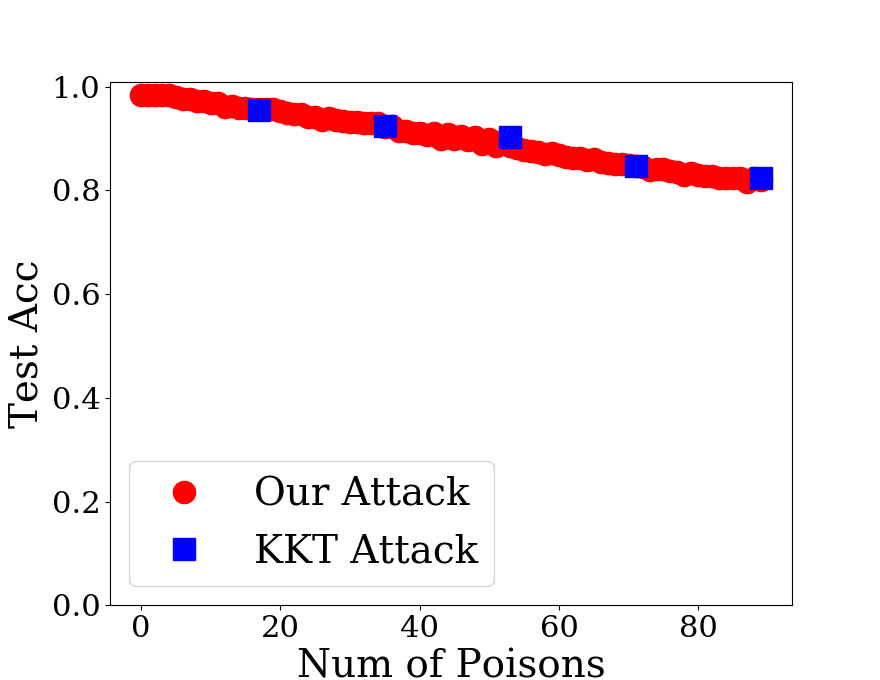}
            \caption[]%
            {20\% Error Rate}    
            \label{fig:dogfish_svm_02_acc}
        \end{subfigure}
        \begin{subfigure}[b]{0.33\textwidth} 
            \centering 
            \includegraphics[width=0.98\textwidth]{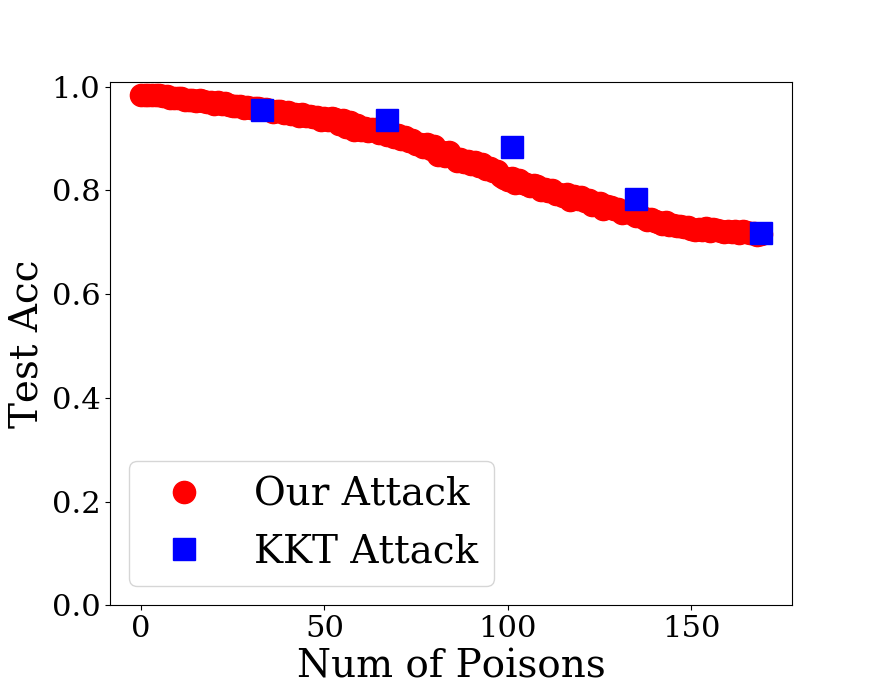}
            \caption[]%
            {30\% Error Rate}    
            \label{fig:dogfish_svm_03_acc}
        \end{subfigure}
        %\vskip\baselineskip
        %\quad
        \caption[]
        {SVM on Dogfish dataset: test accuracy of each target model of given error rate with classifiers induced by poisoning points obtained from our attack and the KKT attack.} 
        \label{fig:dogfish_svm_acc}
    \end{figure*}

\shortsection{Attack Success} In Figure~\ref{fig:adult_svm_subpop_acc} - Figure~\ref{fig:dogfish_svm_acc}, we show the attack success of our attack as the number of poisoning points gradually increases. 
%We also evaluate the KKT attack at $0.2n_p, 0.4n_p, 0.6n_p, 0.8n_p$ and $n_p$ number of poisoning points, and compare to our attack. $n_p$ is the maximum number of poisoning points and is obtained by running our attack till the induced model is $\epsilon$-close to the target model. We set $\epsilon$ as 0.01 for Adult, 0.1 for \MNIST\ and 2.0 Dogfish datasets. In fact, 
These figures present Table~\ref{tab:model-targeted-comparison-subpop} and Table ~\ref{tab:model-targeted-comparison-indis} (in the main paper) in the form of figures. The main purpose of these figures is to highlight the online nature of our attack -- in contrast to the KKT attack, our attack does not require the number of poisoning points in advance and the attack performance in each iteration can be easily tracked. Besides the online and incremental property, the conclusion from the figures is the same as the conclusion for SVM model in Table~\ref{tab:model-targeted-comparison-subpop} and Table ~\ref{tab:model-targeted-comparison-indis} --  our attack has better attack success than the KKT attack in subpopulation setting and has comparable performance in the indiscriminate setting.

    \begin{figure*}[!tb]
        \centering
        \begin{subfigure}[b]{0.33\textwidth}
            \centering
            \includegraphics[width=0.98\textwidth]{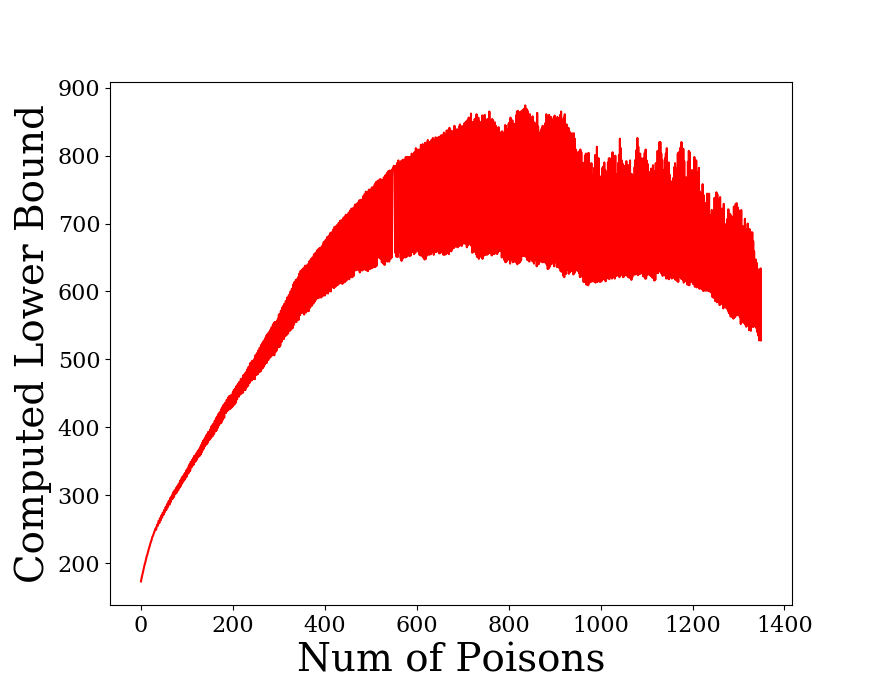}
           \caption[]
            {5\% Error Rate} 
            \label{fig:mnist_svm_005_lower_bound_ol}
        \end{subfigure}
        %\hfill
        \begin{subfigure}[b]{0.33\textwidth} 
            \centering 
            \includegraphics[width=0.98\textwidth]{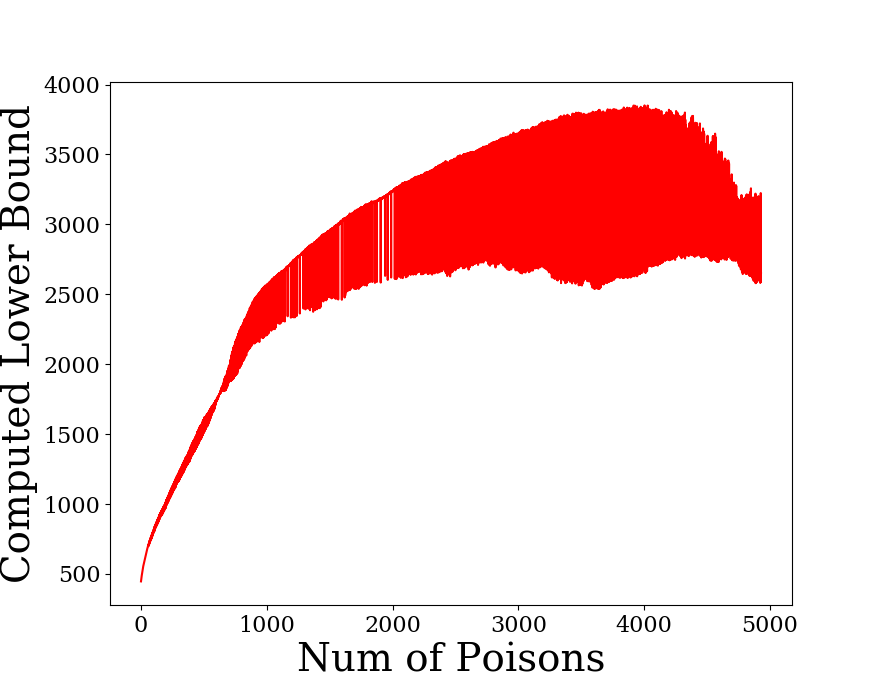}
            \caption[]%
            {5\% Error Rate}    
            \label{fig:mnist_svm_01_lower_bound_ol}
        \end{subfigure}
        \begin{subfigure}[b]{0.33\textwidth}  
            \centering 
            \includegraphics[width=0.98\textwidth]{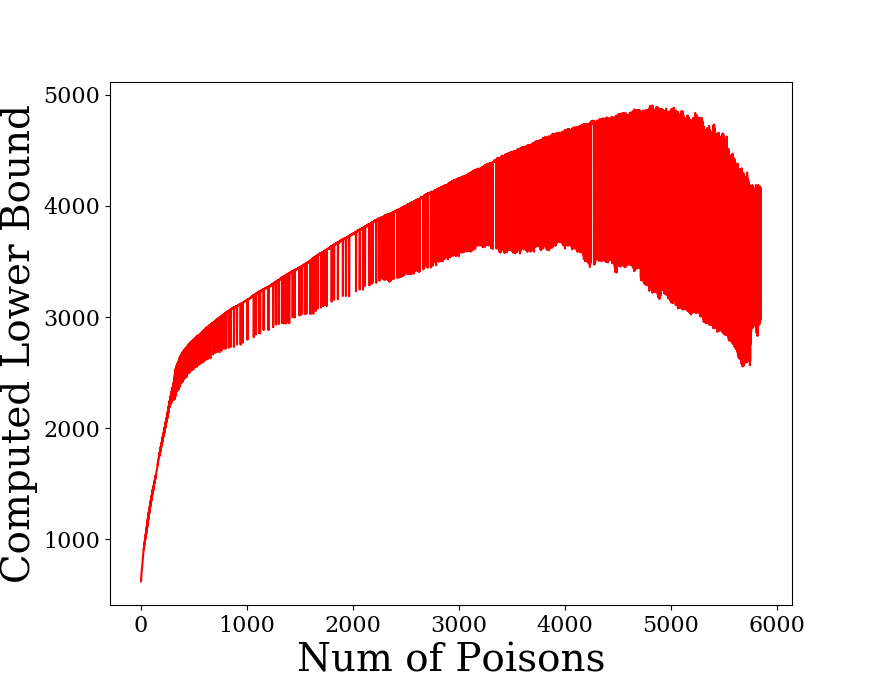}
            \caption[]%
            {15\% Error Rate}    
            \label{fig:mnist_svm_015_lower_bound_ol}
        \end{subfigure}
        %\vskip\baselineskip
        %\quad
        \caption[]
        {SVM on \MNIST: lower bound computed in each iteration of running algorithm~\ref{algorithm}. The target classifier of the algorithm is the classifier induced from our Attack. The maximum number of poisons is obtained when the induced classifier is $0.1$-close to the target classifier.} 
        \label{fig:mnist_svm_indiscriminate_lower_bound_ol}
    \end{figure*}    
  
    \begin{figure*}[!tb]
        \centering
        \begin{subfigure}[b]{0.33\textwidth}
            \centering
            \includegraphics[width=0.98\textwidth]{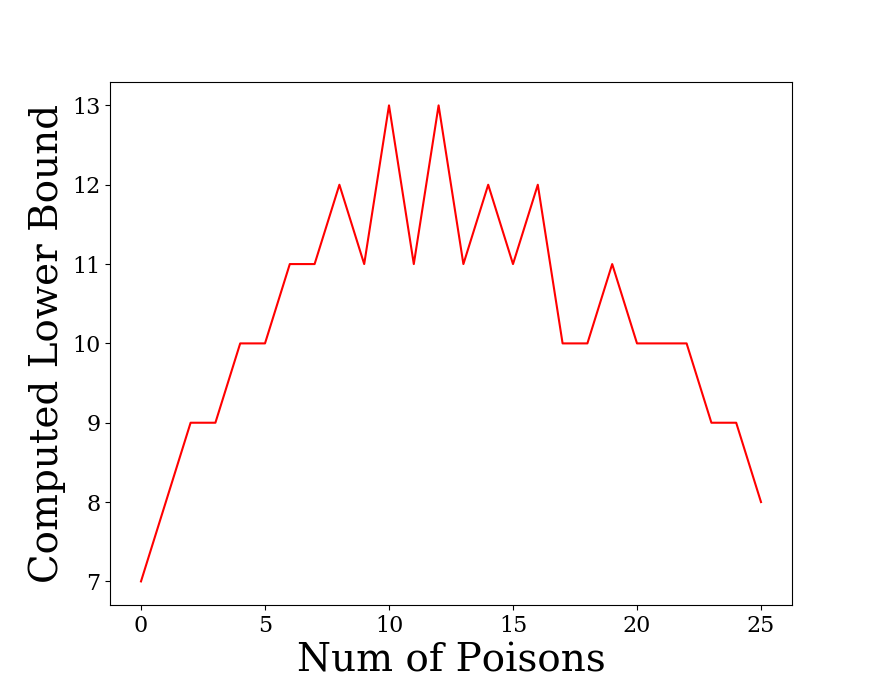}
           \caption[]
            {10\% Error Rate} 
            \label{fig:dogfish_svm_01_lower_bound_ol}
        \end{subfigure}
        %\hfill
        \begin{subfigure}[b]{0.33\textwidth} 
            \centering 
            \includegraphics[width=0.98\textwidth]{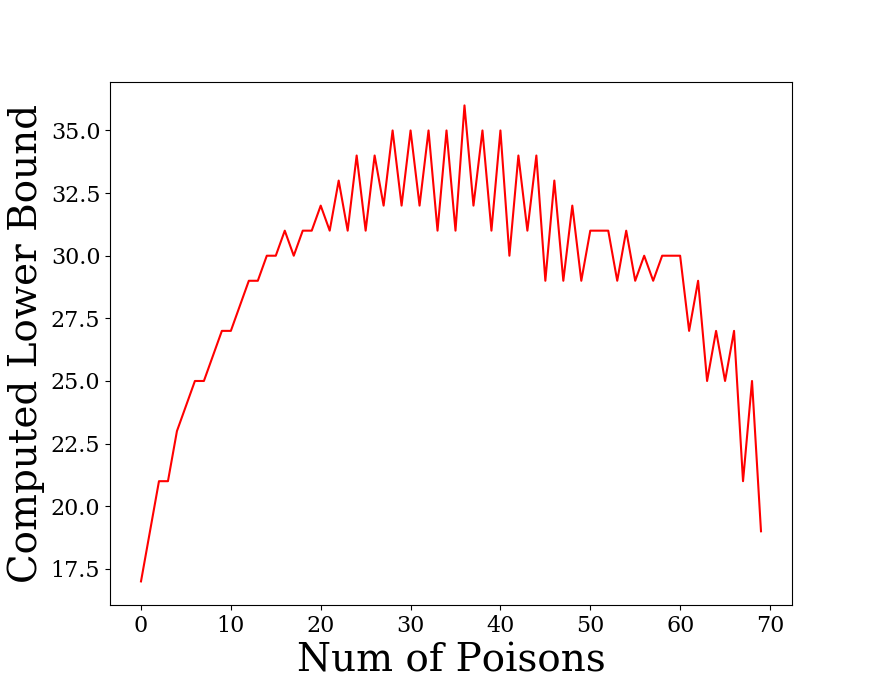}
            \caption[]%
            {20\% Error Rate}    
            \label{fig:dogfish_svm_02_lower_bound_ol}
        \end{subfigure}
        \begin{subfigure}[b]{0.33\textwidth}  
            \centering 
            \includegraphics[width=0.98\textwidth]{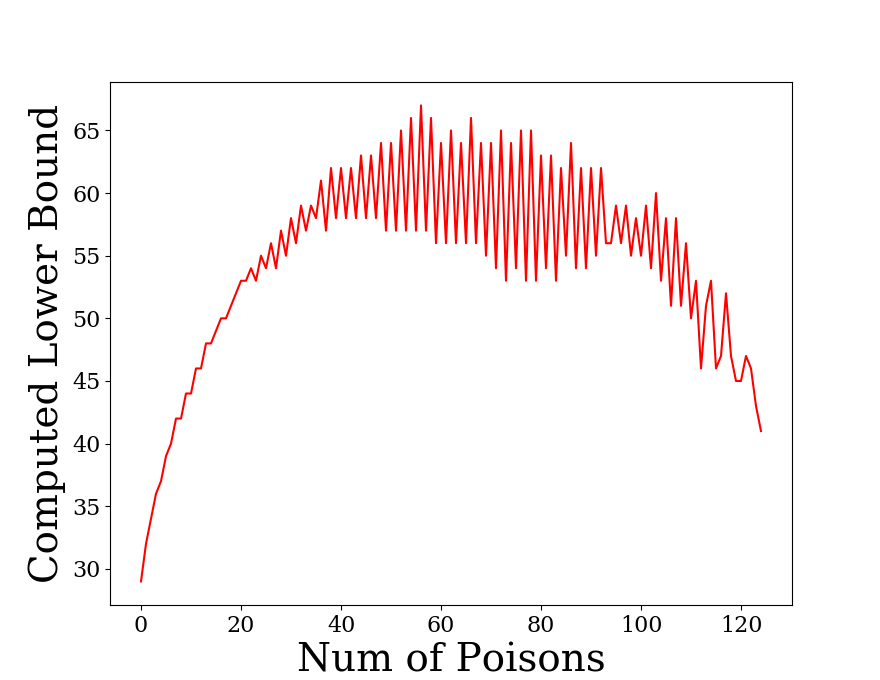}
            \caption[]%
            {30\% Error Rate}    
            \label{fig:dogfish_svm_03_lower_bound_ol}
        \end{subfigure}
        %\vskip\baselineskip
        %\quad
        \caption[]
        {SVM on Dogfish: lower bound computed in each iteration when running algorithm~\ref{algorithm}. The target classifier of for the algorithm is the classifier induced from our Attack. The maximum number of poisons is obtained when the induced classifier is $2.0$-close to the target classifier.} 
        \label{fig:dogfish_svm_indiscriminate_lower_bound_ol}
    \end{figure*}

\shortsection{Lower Bound on Number of Poisons} The lower bounds for SVM in Table~\ref{tab:model-targeted-comparison-subpop} and Table~\ref{tab:model-targeted-comparison-indis} in the main paper is obtained by running Algorithm~\ref{algorithm} and using the intermediate classifier $\theta_t$ to compute the lower bound (with Theorem~\ref{theorem:lower_bound}) in each iteration, and returning the highest lower bound computed across all iterations. In this section, we directly plot the computed lower bound in each iteration to show the trend of the lower bound as more number of poisoning points are added. Figure~\ref{fig:mnist_svm_indiscriminate_lower_bound_ol} and Figure~\ref{fig:dogfish_svm_indiscriminate_lower_bound_ol} shows the results on \MNIST\ and Dogfish datasets. From the figures, we can easily observe that the peak value of the lower bound is obtained in the middle of the attack process. Therefore, it might be the case that the computed lower bound is already very tight, as we cannot improve the highest lower bound by running the attack for more iterations. This implies that, it is more likely that our attack is not very optimal on these two datasets and we should seek for more efficient data poisoning attacks. We did not show the curves for Adult dataset because the gap between the lower bound and the number of poisoning points used by our attack is small, indicating our attack is nearly optimal. 

% We next check the optimality of our attack in the indiscriminate attack scenario. Similar to the subpopulation attack setting, we still use Theorem~\ref{thm:lowerbounds} to compute the lower bound of the induced classifier from our attack by using it as the input to Algorithm~\ref{algorithm}, and terminating when the induced classifier is $0.1$-close to the given target model. In Table~\ref{tab:mnist_svm_indiscrim_validate_lower_bound_ol}, our calculated lower bound shows that there exists a relatively large gap between the number of poisoning points, especially for the induced classifier from our attack for the target model of 5\% error rate, where the lower bound is only 50\% of the actual number of poisoning points used. For the induced classifier for the target model of 15\% error rate, the gap between the number of poisoning points and the lower bound is smallest, with the lower bound taking 79\% of the number of poisoning points. The relatively large gap indicates that either the estimated lower bound is not tight or the attack itself is not close to optimal. To gain more insights into this problem, we further show the computed lower bound at each iteration when running Algorithm~\ref{algorithm} and Figure~\ref{fig:mnist_svm_indiscriminate_lower_bound_ol} summarizes the results. From the Figure, we see that, the peak value of the curve (i.e., highest lower bound) always appears before the termination of the algorithm, indicating that the computed lower bound is likely to be tight and we may need to further improve the attack algorithm. 
For completeness, we also repeat the same experiment, but now with the model induced from the KKT attack as the target model for our attack to compute its lower bound. In Table~\ref{tab:mnist_svm_indiscrim_validate_lower_bound_kkt} and Figure~\ref{fig:mnist_svm_indiscriminate_lower_bound_kkt}, we report the lower bound results on \MNIST\ dataset. Table~\ref{tab:mnist_svm_indiscrim_validate_lower_bound_kkt} shows the highest computed lower bound and Figure~\ref{fig:mnist_svm_indiscriminate_lower_bound_kkt} plots the lower bound computed in each iteration. The conclusion is still the same as our attack -- there still exists a large gap between the lower bound and the number of poisoning points used by the KKT attack, which indicates that the KKT attack is also not very efficient. We have similar observations on the Dogfish dataset using the KKT attack. 

    \begin{table}[!tb]
    \centering
    \begin{tabular}{cccc}
    \toprule
     & 5\% Error & 10\% Error & 15\% Error \\ \midrule
    \# of Poisons & 1737 & 5458 & 6192 \\
    Lower Bound & 856 & 4058.4$\pm$1.4 & 5031.4$\pm$4.8 \\ \bottomrule
    \end{tabular}
    \vspace{0.5em}
    \caption[]{SVM on \MNIST: poisoning points needed to achieve target classifiers induced from the KKT attack. Top row means number of poisoning points used by the KKT attack. Bottom row means the lower bound computed from Theorem~\ref{theorem:lower_bound} for the target classifier, which is the model induced by the KKT attack. All results are averaged over 4 runs, integer value in the cell means we get exactly same value for 4 runs and others are shown with the average and standard error.
    }
    \label{tab:mnist_svm_indiscrim_validate_lower_bound_kkt} 
    \end{table}

    \begin{figure*}[!tb]
        \centering
        \begin{subfigure}[b]{0.33\textwidth}
            \centering
            \includegraphics[width=0.98\textwidth]{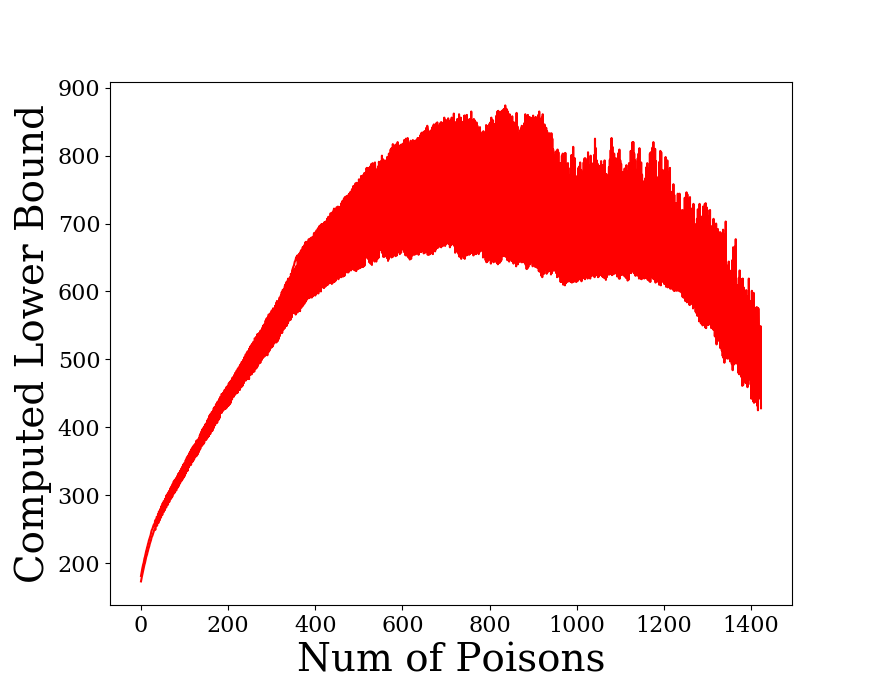}
           \caption[]
            {5\% Error Rate} 
            \label{fig:mnist_svm_005_lower_bound_kkt}
        \end{subfigure}
        %\hfill
        \begin{subfigure}[b]{0.33\textwidth}  
            \centering 
            \includegraphics[width=0.98\textwidth]{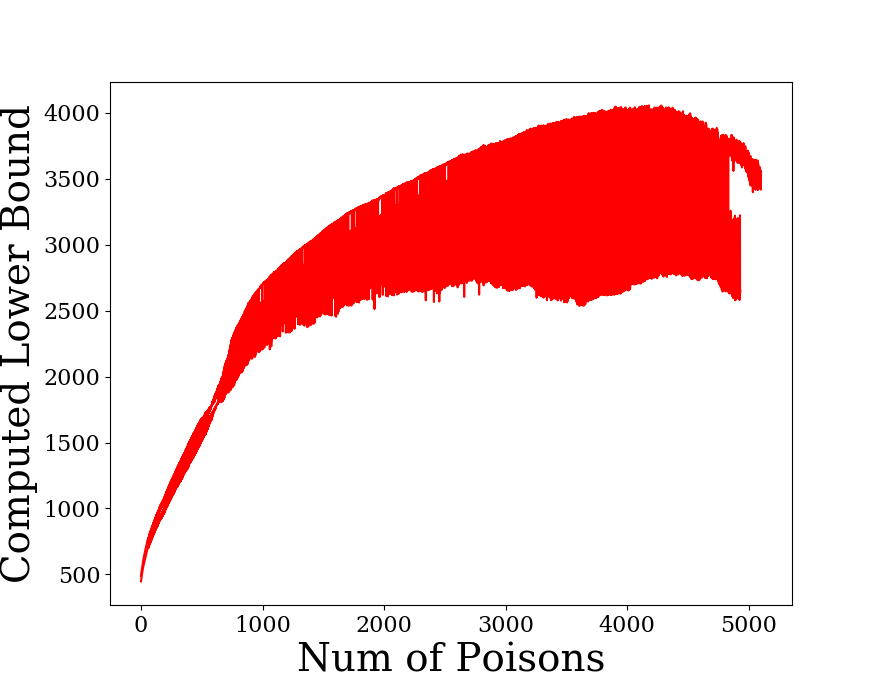}
            \caption[]%
            {5\% Error Rate}    
            \label{fig:mnist_svm_01_lower_bound_kkt}
        \end{subfigure}
        \begin{subfigure}[b]{0.33\textwidth} 
            \centering 
            \includegraphics[width=0.98\textwidth]{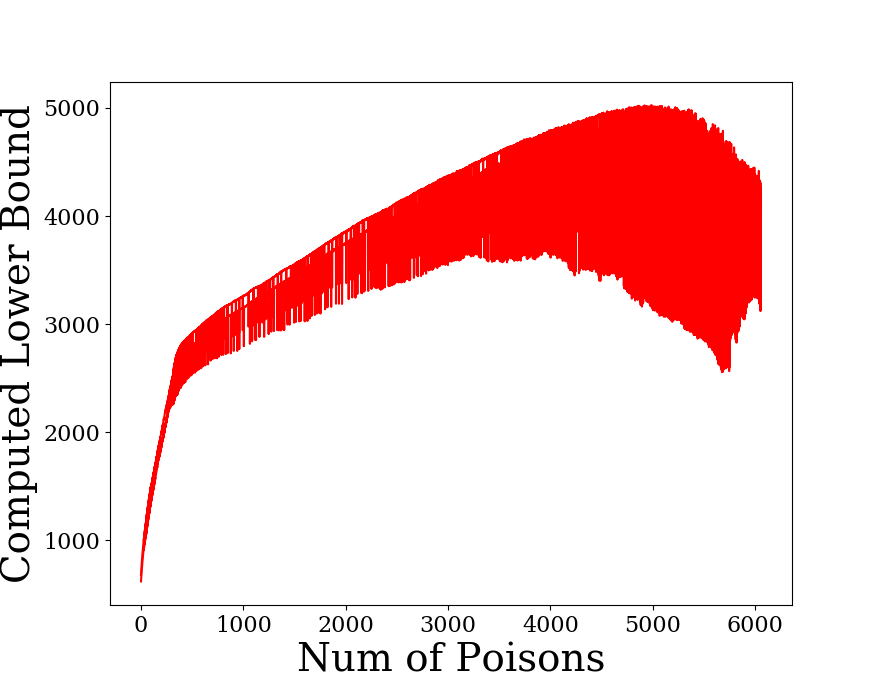}
            \caption[]%
            {15\% Error Rate}    
            \label{fig:mnist_svm_015_lower_bound_kkt}
        \end{subfigure}
        %\vskip\baselineskip
        %\quad
        
        \caption[]
        {SVM on \MNIST: lower bound computed in each iteration of running algorithm~\ref{algorithm} when the target classifier of the algorithm is the classifier induced from the KKT Attack. The maximum number of poisons is set using the $0.1$-close threshold to KKT induced classifier.} 
        \label{fig:mnist_svm_indiscriminate_lower_bound_kkt}
    \end{figure*}

\subsection{More Results on Logistic Regression}
\label{sec:addi_lr_results}
In this section, we provide additional results on the logistic regression model. The experiment setup is as the same as in Section~\ref{sec:addi_svm_results}. Compared to SVM, we do not report the lower bound results for logistic regression because the maximum loss difference found for logistic regression is an approximate solution and hence the lower bound can be invalid. In what follows, we first discuss the impact of approximate maximum loss difference and then show results on the attack convergence and attack success. 

\shortsection{Approximate Maximum Loss Difference} The convergence guarantee in the paper also holds for logistic regression model (more generally, holds for any Lipschitz and convex function with strongly convex regularizer). However, for logistic regression, we may not be able to efficiently search for the globally optimal point with maximum loss difference (Line 4 in Algorithm~\ref{algorithm}) because the difference of two logistic losses is not concave. Therefore, we adopt gradient descent strategy, using the Adam optimizer~\citep{kingma2014adam} to search for the point that (approximately) maximizes the loss difference. This is in contrast to the SVM model, where the difference of Hinge loss is piece-wise linear and we can deploy general (convex) solvers to search for the globally optimal point in each linear segment~\citep{diamond2016cvxpy,gurobi2020manual}. 
However, as will be demonstrated next, poisoning points with approximate maximum loss difference can still be very effective. More formally, if the approximate maximum loss difference $\hat{l}$ found from local optimization techniques is within a constant factor from the globally optimal value $l^{*}$ (i.e., $\hat{l}\geq\alpha l^{*}, 0 <\alpha<1$), then we still enjoy similar convergence guarantees. A similar issue of global optimality also applies to the KKT attack~\citep{koh2018stronger}, where the attack objective function is no longer convex for logistic regression models, and therefore, we also utilize gradient based technique to (approximately) solve the optimization problem and present the results below. 

    \begin{figure*}[tbp]
%    \vspace{-1em}
        \centering
        \begin{subfigure}[b]{0.45\textwidth} 
            \centering 
            \includegraphics[width=\textwidth]{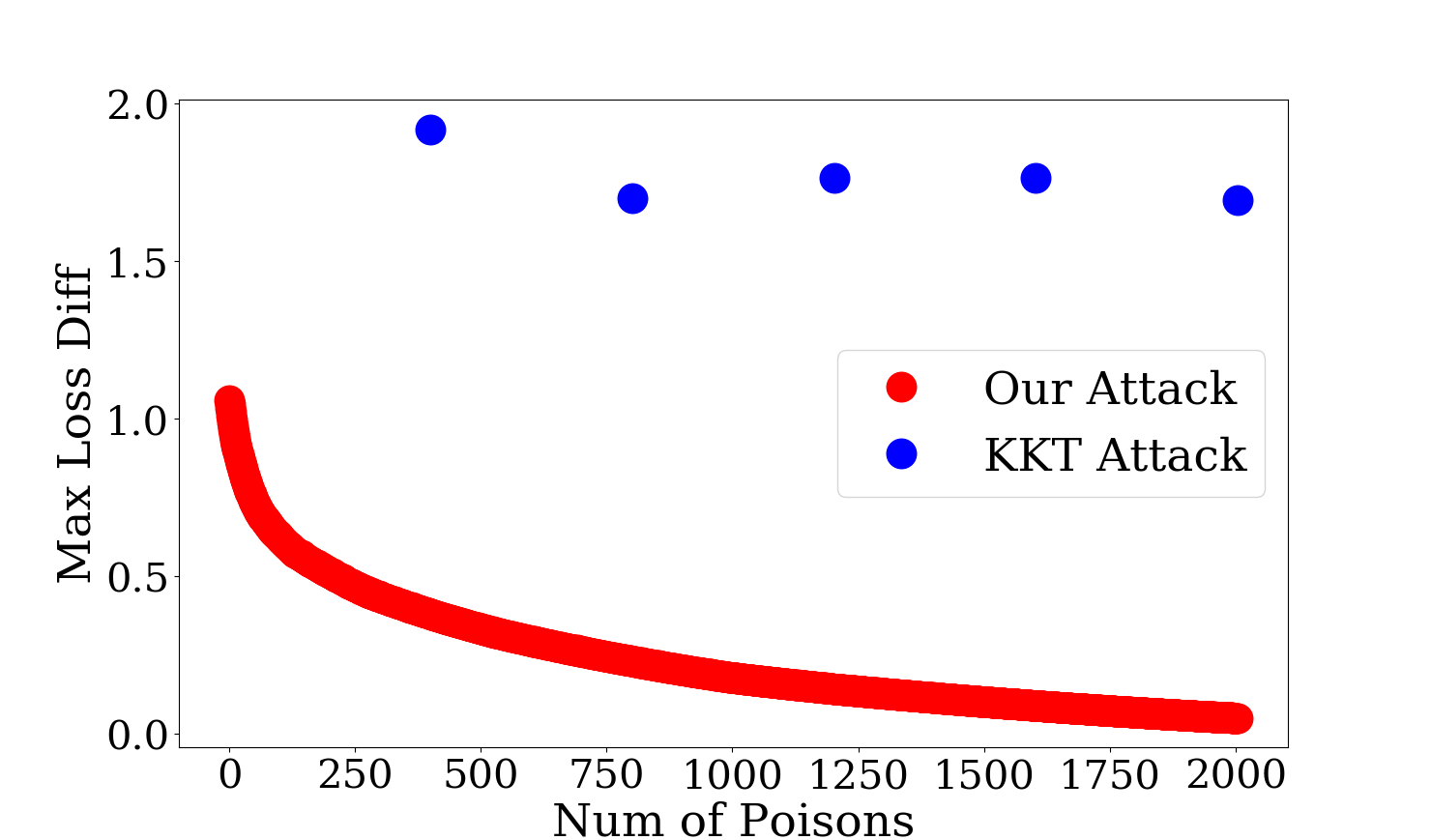}
            \caption[]%
            {Max Loss Difference}    
            \label{fig:adult_lr_subpop0_max_loss_diff}
        \end{subfigure}
        \begin{subfigure}[b]{0.45\textwidth} 
            \centering 
        \includegraphics[width=\textwidth]{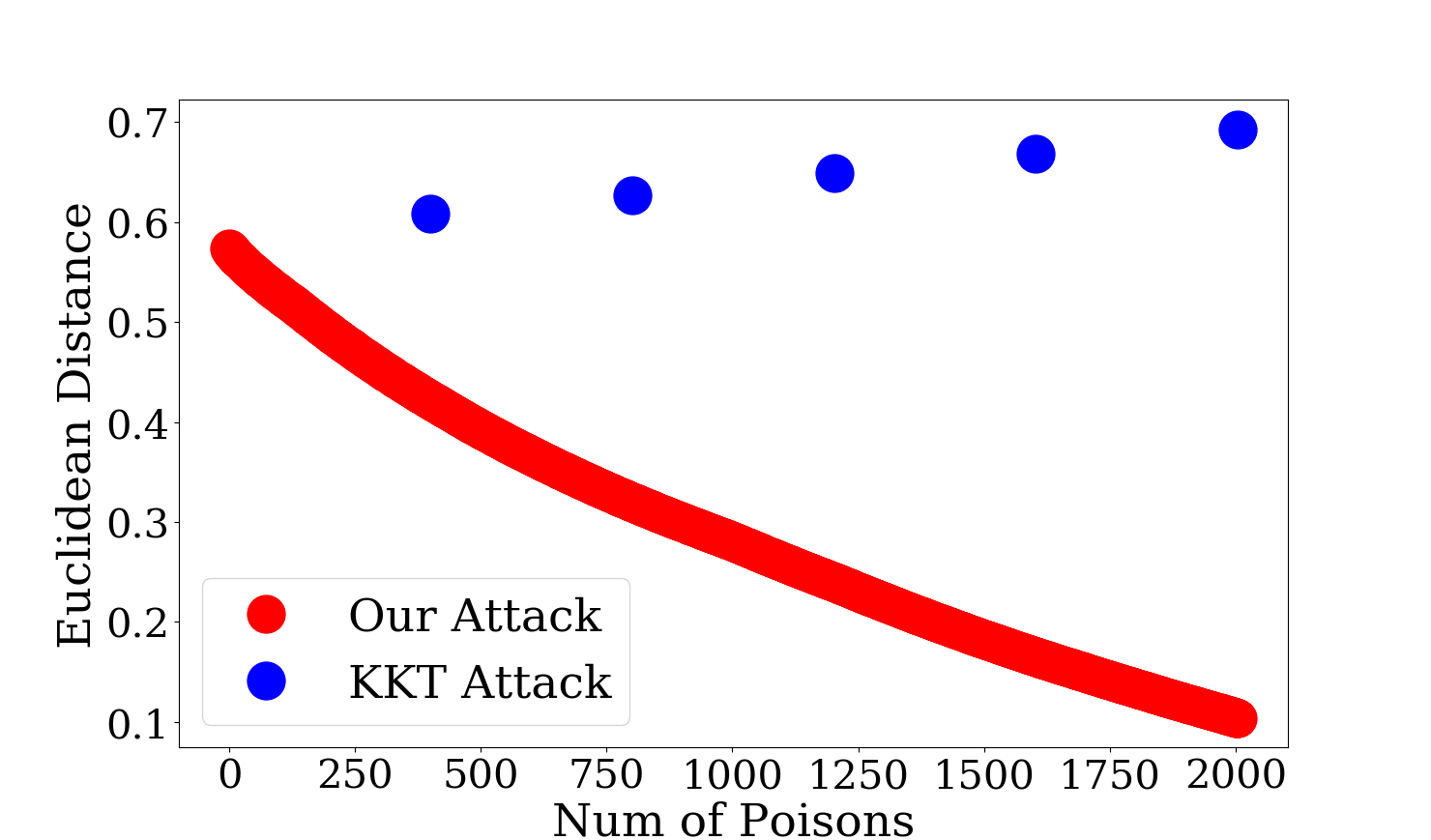}
            \caption[]%
            {Euclidean Distance}    
            \label{fig:adult_lr_subpop0_norm_diff}
        \end{subfigure}
        \caption[]
        {Logistic regression model on Adult: attack convergence (results shown are for the first subpopulation, Cluster 0). The maximum number of poisons is set using the $0.05$-close threshold to target classifier.}
        \label{fig:adult_lr_subpop0_convergence}
    \end{figure*}
    
    \begin{figure*}[tbp]
        \centering
        \begin{subfigure}[b]{0.45\textwidth}  
            \centering 
            \includegraphics[width=\textwidth]{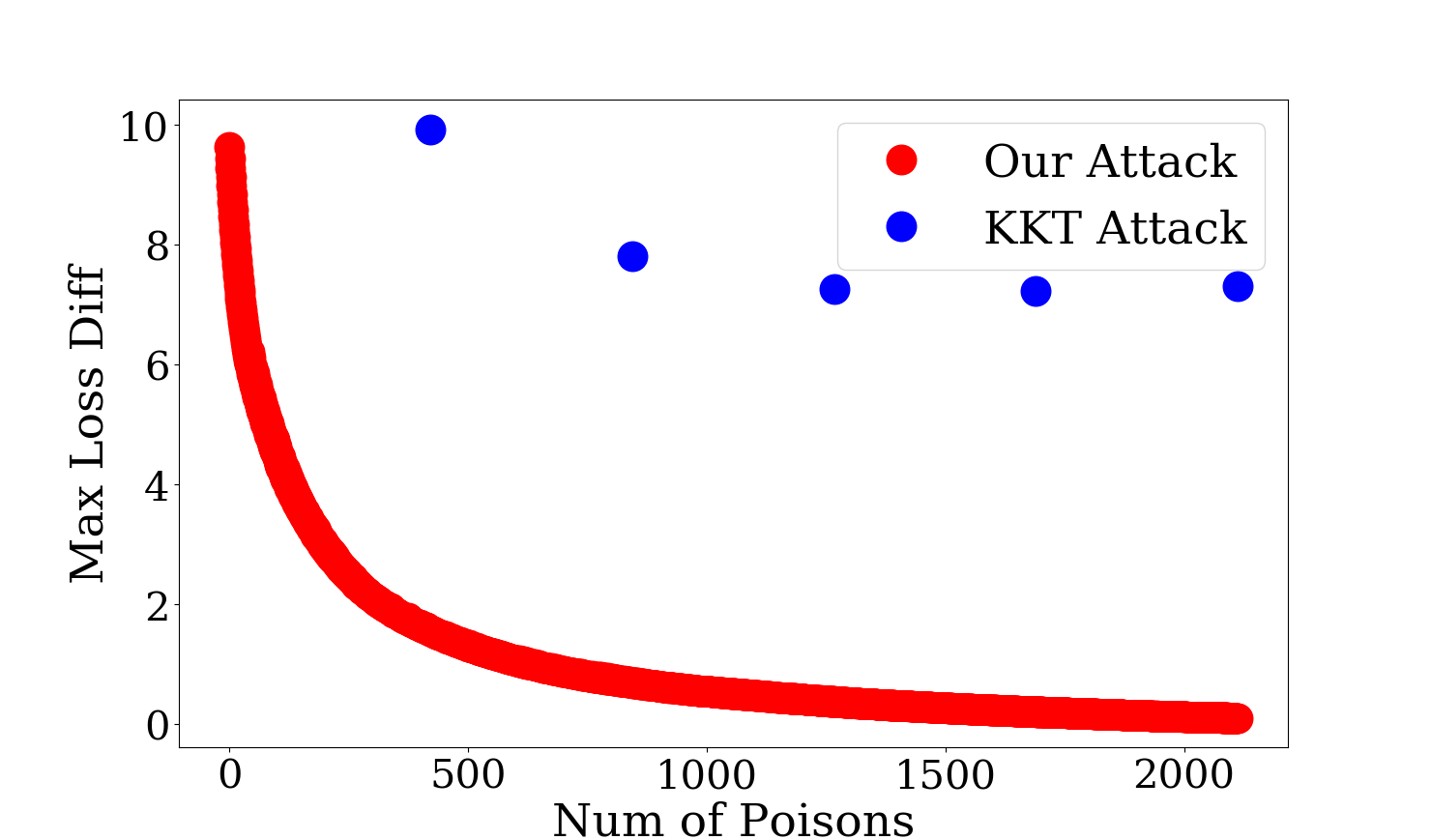}
            \caption[]%
            {Max Loss Difference}    
            \label{fig:mnist_lr_error_01_max_loss_diff}
        \end{subfigure}
        \begin{subfigure}[b]{0.45\textwidth} 
            \centering 
            \includegraphics[width=\textwidth]{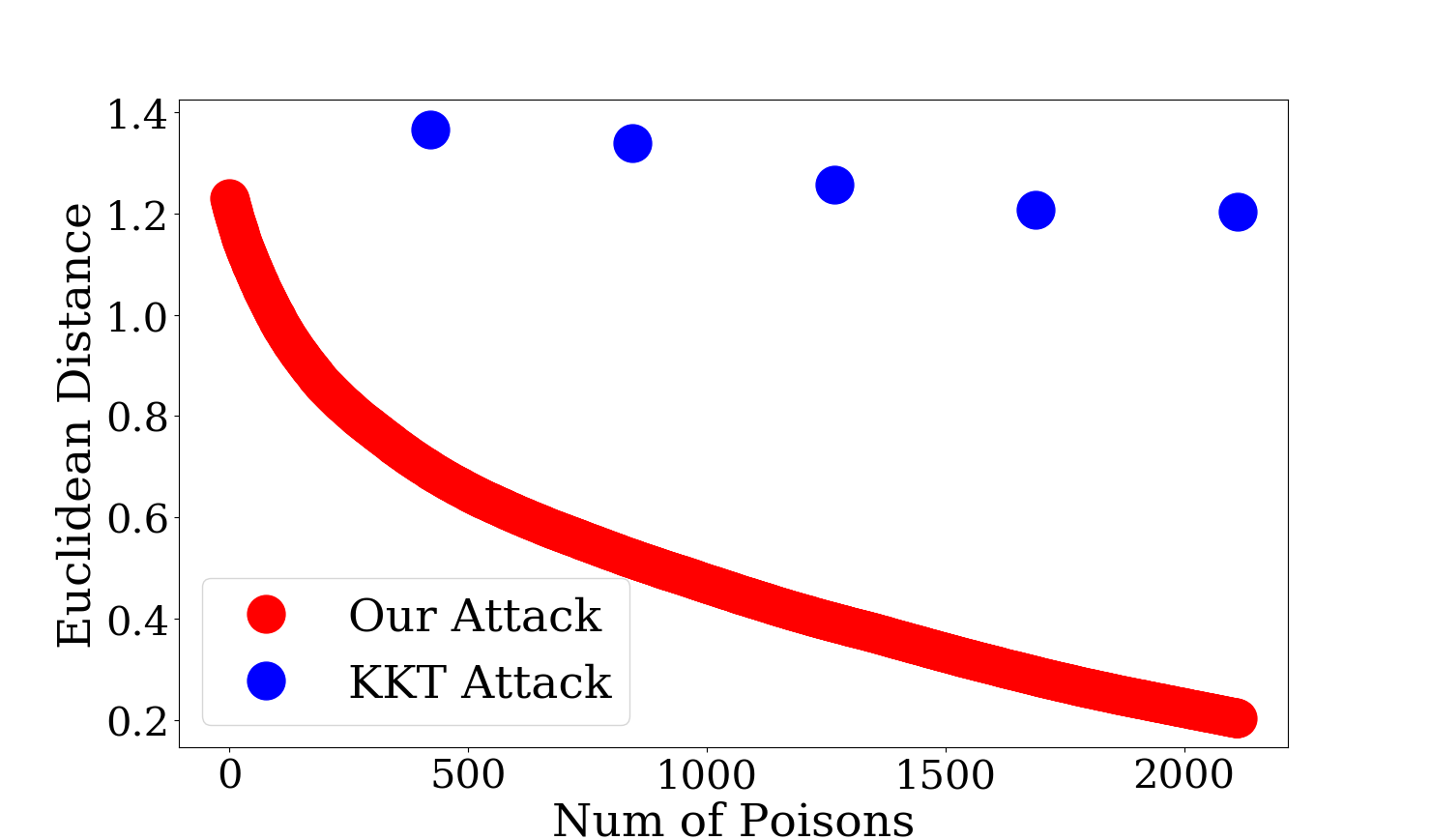}
            \caption[]%
            {Euclidean Distance}    
            \label{fig:mnist_lr_error_01_norm_diff}
        \end{subfigure}

        \caption[]
        {Logistic regression model on \MNIST\ dataset: attack convergence (results shown are for the target classifier of error rate 10\%). The maximum number of poisons is set using the $0.1$-close threshold to target classifier.} 
        \label{fig:mnist_lr_error_01_convergence}
    \end{figure*}

    \begin{figure*}[tbp]
        \centering
        \begin{subfigure}[b]{0.45\textwidth}  
            \centering 
            \includegraphics[width=\textwidth]{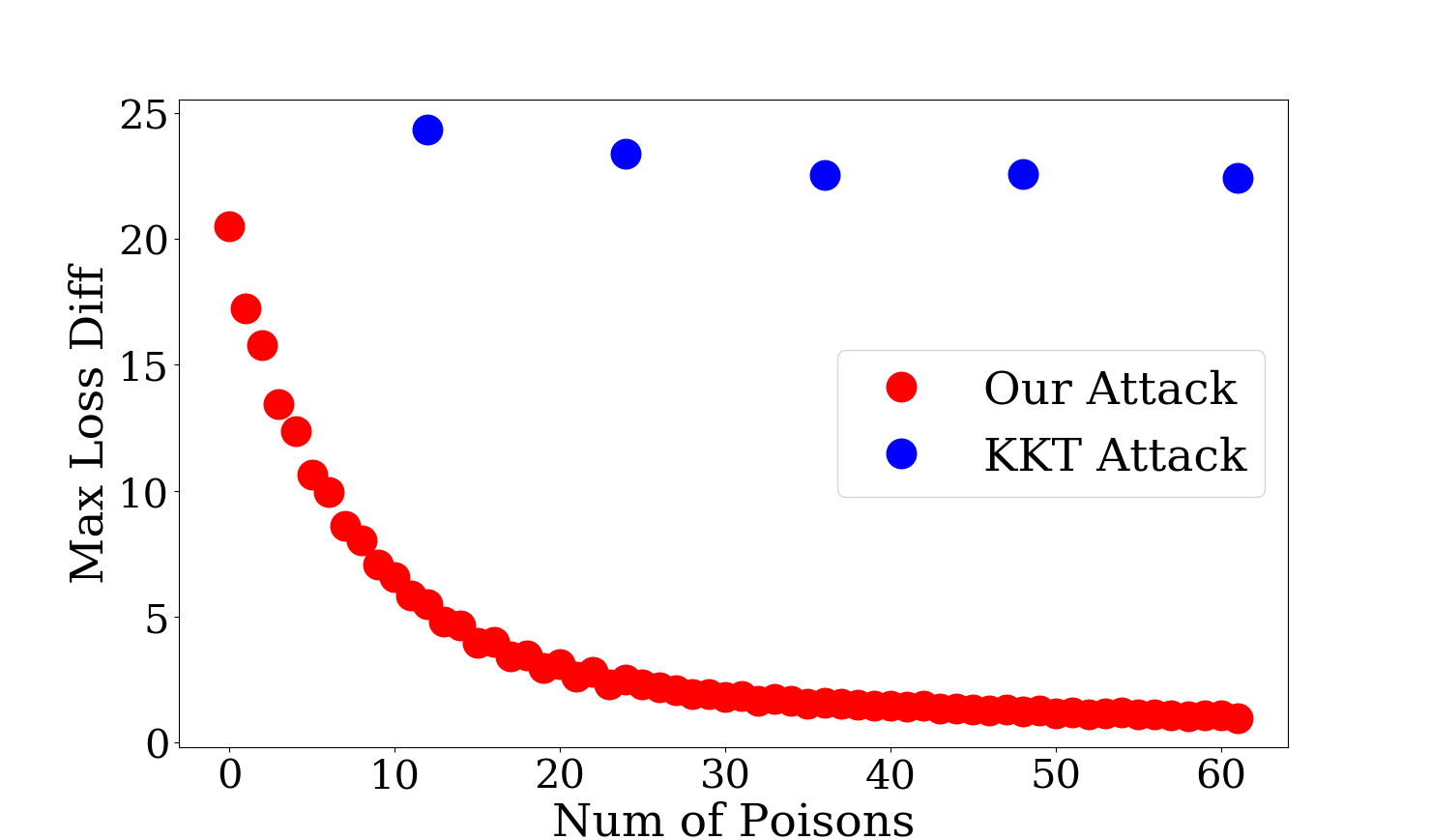}
            \caption[]%
            {Max Loss Difference}    
            \label{fig:dogfish_lr_error_01_max_loss_diff}
        \end{subfigure}
        \begin{subfigure}[b]{0.45\textwidth} 
            \centering 
            \includegraphics[width=\textwidth]{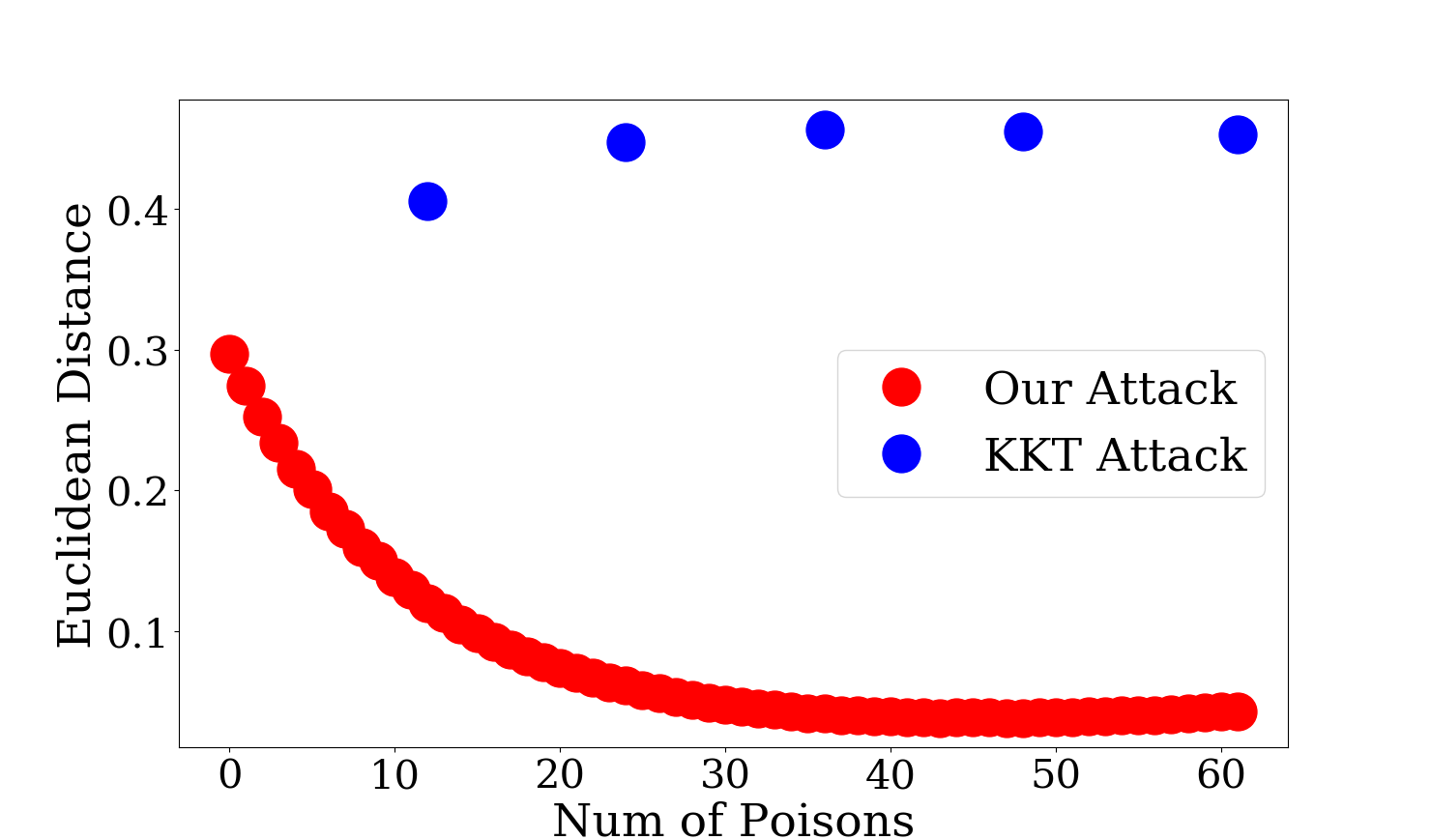}
            \caption[]%
            {Euclidean Distance}    
            \label{fig:dogfish_lr_error_01_norm_diff}
        \end{subfigure}

        \caption[]
        {Logistic regression model on Dogfish: attack convergence (results shown are for the target classifier of error rate 10\%). The maximum number of poisons is set using the $1.0$-close threshold to target classifier.} 
        \label{fig:dogfish_lr_error_01_convergence}
    \end{figure*}

\shortsection{Convergence} The results results for logistic regression on Adult, \MNIST\ and Dogfish datasets are show in Figure~\ref{fig:adult_lr_subpop0_convergence}, Figure~\ref{fig:mnist_lr_error_01_convergence} and Figure~\ref{fig:dogfish_lr_error_01_convergence} respectively. For the Adult dataset, we show the convergence on the first subpopulation (cluster 0). For \MNIST\ and Dogfish, similar to Section~\ref{sec:addi_svm_results}, we show the convergence on the target models of 10\% error rates. All results show that, our attack steadily converges to the target model while the KKT attack fails to have a reliable convergence. Similar observations are also found in other settings (i.e., different clusters for the subpopulation setting and different target models in the indiscriminate settings). 

\shortsection{Attack Success}
The attack success results on Adult, \MNIST\ and Dogfish datasets are show in Figure~\ref{fig:adult_lr_subpop_acc}, Figure~\ref{fig:mnist_lr_indiscriminate_acc} and Figure \ref{fig:dogfish_lr_acc} respectively. These figures present the logistic regression results in Table~\ref{tab:model-targeted-comparison-subpop} and Table~\ref{tab:model-targeted-comparison-indis} (in the main paper) in the form of figures. All the results show that our attack is much more effective than the KKT attack on logistic regression models, and in fact, the KKT attack cannot effectively poison the models in most cases. In addition, our attack runs in an online fashion and we can easily track the attack performance in each iteration.

% \shortsection{Results on Adult} Figure~\ref{fig:adult_lr_subpop0_convergence} shows the effectiveness of our attack on logistic regression models trained on the Adult dataset, using the loss-based distance and the actual model distance in $\ell_2$-norm. The maximum number of poisons for the experiments is obtained when the classifier from Algorithm 1 is $0.05$-close to the target classifier. Our attack steadily converges to the target model while the KKT attack fails to have a reliable convergence. Similar observations are also found in other attack settings, as shown in Figure~\ref{fig:adult_lr_subpop_acc}. Our attack is much more successful than the KKT attack, especially for the attack on Cluster 1. 

%As mentioned at the beginning of Section~\ref{sec:additional_exps}, we did not compute the lower bound on the model induced from our attack because the maximum loss difference found through gradient ascend strategy may not be the globally optimal value, and hence the lower bound might be invalid. 

    \begin{figure*}[tbh!]
%!    \vspace{-1em}
        \centering
        \begin{subfigure}[b]{0.33\textwidth}
            \centering
            \includegraphics[width=0.98\textwidth]{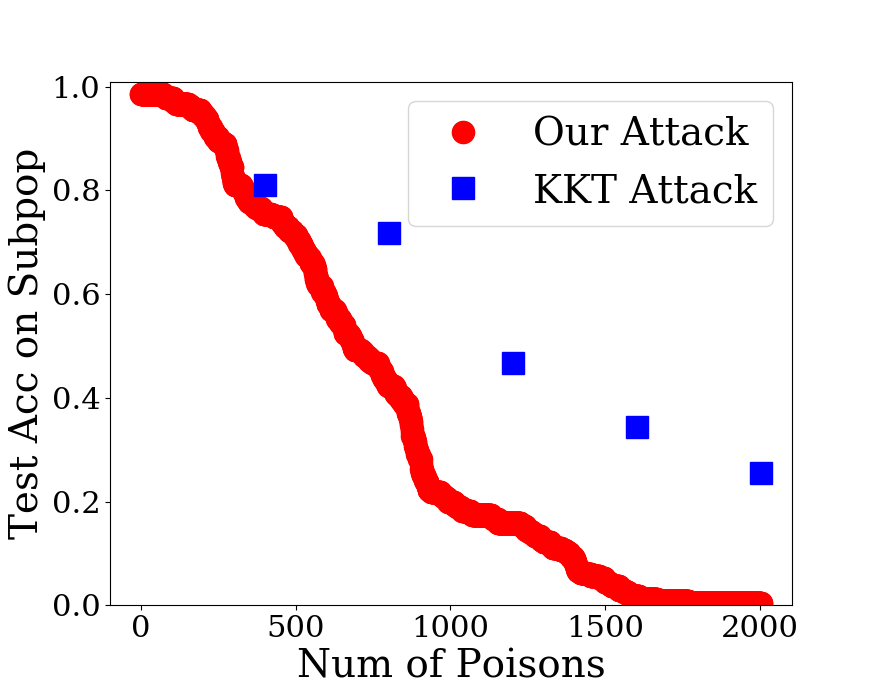}
            \caption[]
            {Cluster 0} 
            \label{fig:adult_lr_subpop0_acc_scores}
        \end{subfigure}
        %\hfill
        \begin{subfigure}[b]{0.33\textwidth}  
            \centering 
            \includegraphics[width=0.98\textwidth]{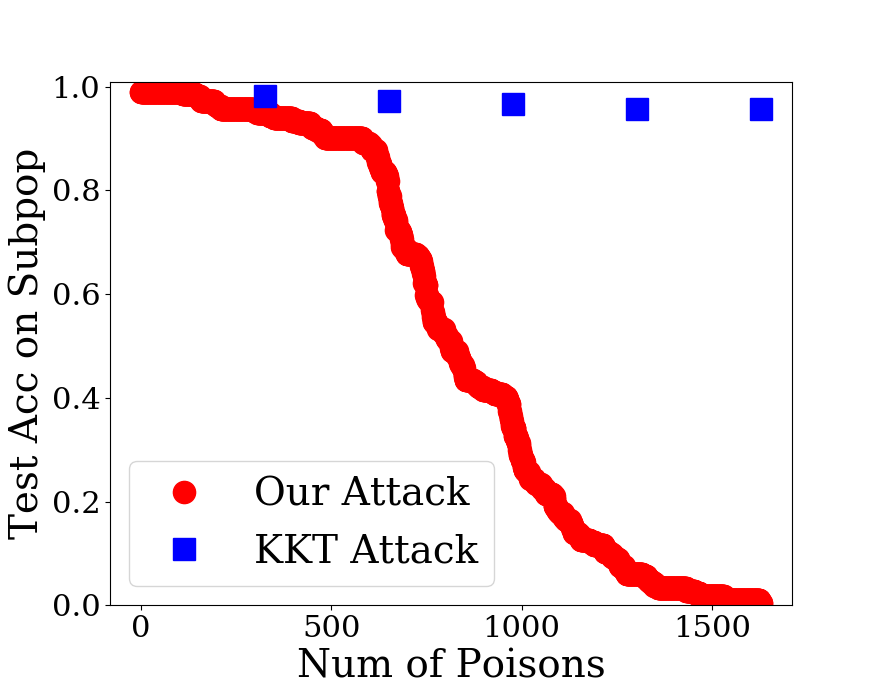}
            \caption[]%
            {Cluster 1}    
            \label{fig:adult_lr_subpop1_acc_scores}
        \end{subfigure}
        \begin{subfigure}[b]{0.33\textwidth}  
            \centering 
            \includegraphics[width=0.98\textwidth]{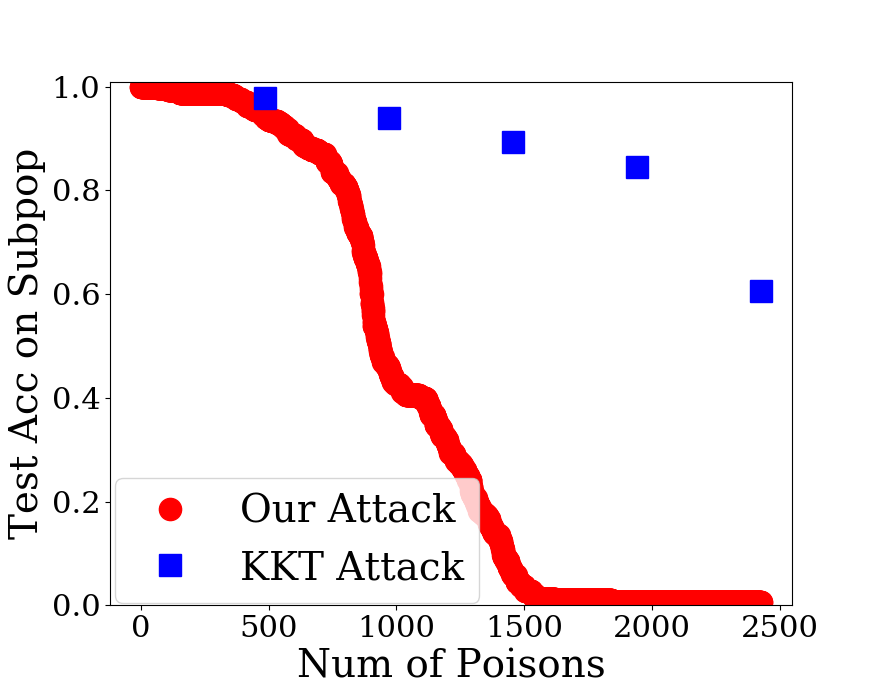}
            \caption[]%
            {Cluster 2}    
            \label{fig:adult_lr_subpop2_acc_scores}
        \end{subfigure}
        %\vskip\baselineskip
        %\quad
        \caption[]
        {Logistic regression model on Adult: test accuracy for each subpopulation with classifiers induced by poisoning points obtained from our attack and the KKT attack.     
        } 
        \label{fig:adult_lr_subpop_acc}
    \end{figure*}
    
    \begin{figure*}[tbh!]
        \centering
        \begin{subfigure}[b]{0.33\textwidth}
            \centering
            \includegraphics[width=0.98\textwidth]{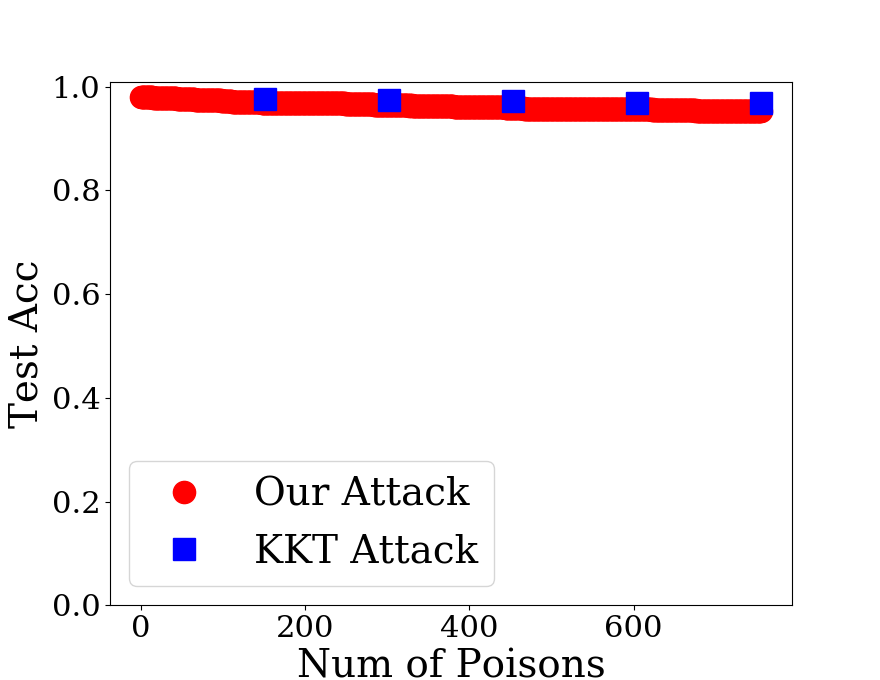}
           \caption[]
            {5\% Error Rate} 
            \label{fig:mnist_lr_error_005_acc_scores}
        \end{subfigure}
        %\hfill
        \begin{subfigure}[b]{0.33\textwidth}  
            \centering 
            \includegraphics[width=0.98\textwidth]{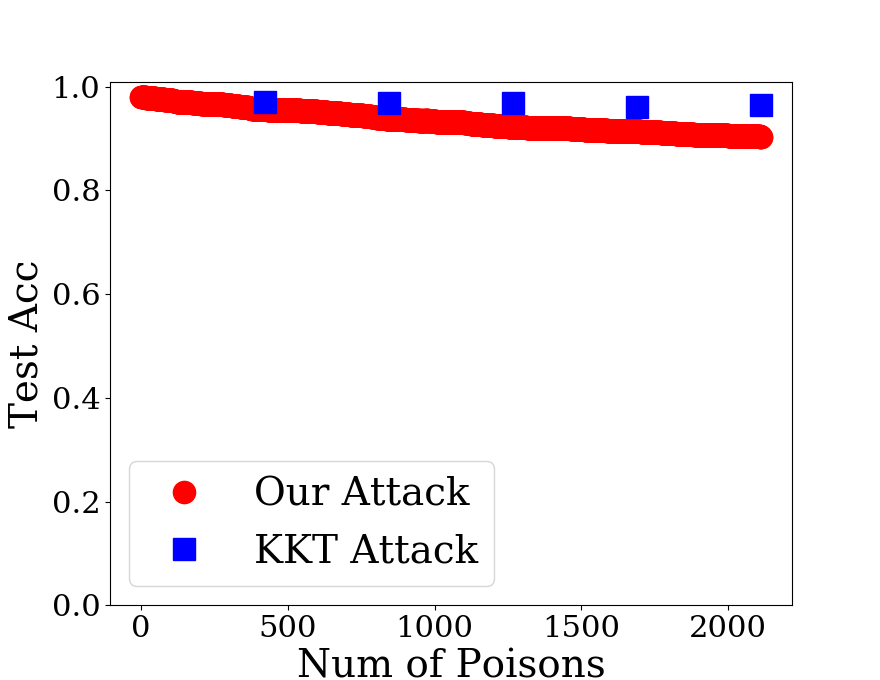}
            \caption[]%
            {10\% Error Rate}    
            \label{fig:mnist_lr_error_01_acc_scores}
        \end{subfigure}
        \begin{subfigure}[b]{0.33\textwidth}  
            \centering 
            \includegraphics[width=0.98\textwidth]{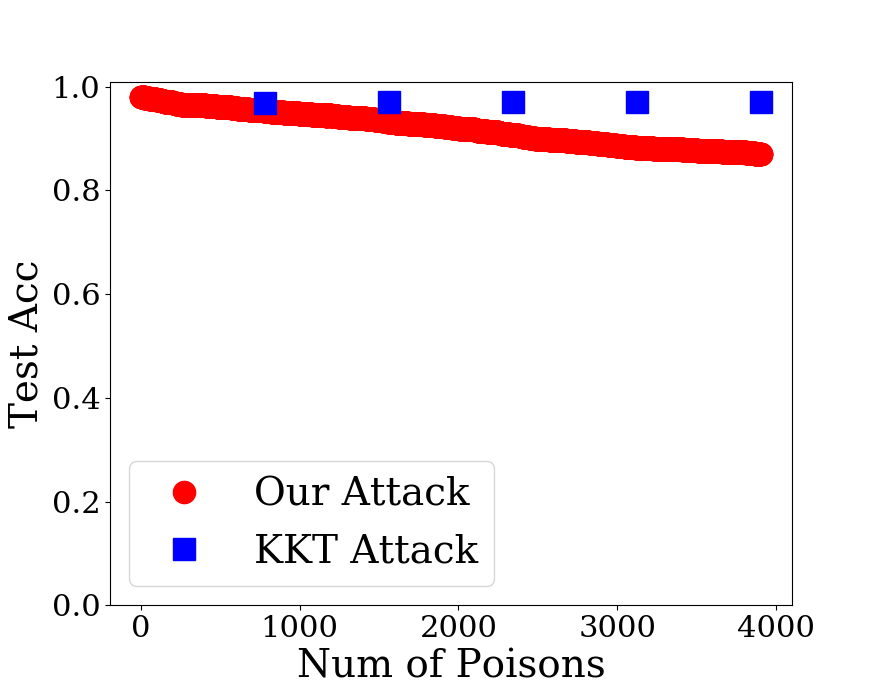}
            \caption[]%
            {15\% Error Rate}    
            \label{fig:mnist_lr_error_015_acc_scores}
        \end{subfigure}
        %\vskip\baselineskip
        %\quad
        \caption[]
        {Logistic regression model on \MNIST: test accuracy for each target model of given error rate with classifiers induced by poisoning points obtained from our attack and the KKT attack.} 
        \label{fig:mnist_lr_indiscriminate_acc}
    \end{figure*}
 
    \begin{figure*}[tbp]
        \centering
        \begin{subfigure}[b]{0.33\textwidth}
            \centering
            \includegraphics[width=0.98\textwidth]{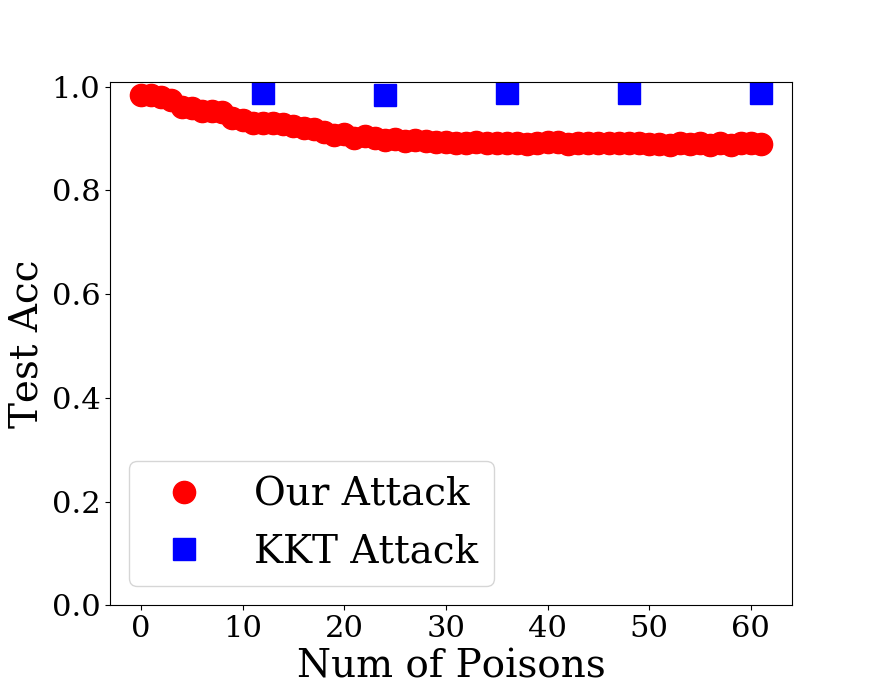}
           \caption[]
            {10\% Error Rate} 
            \label{fig:dogfish_lr_01_acc}
        \end{subfigure}
        %\hfill
        \begin{subfigure}[b]{0.33\textwidth} 
            \centering 
            \includegraphics[width=0.98\textwidth]{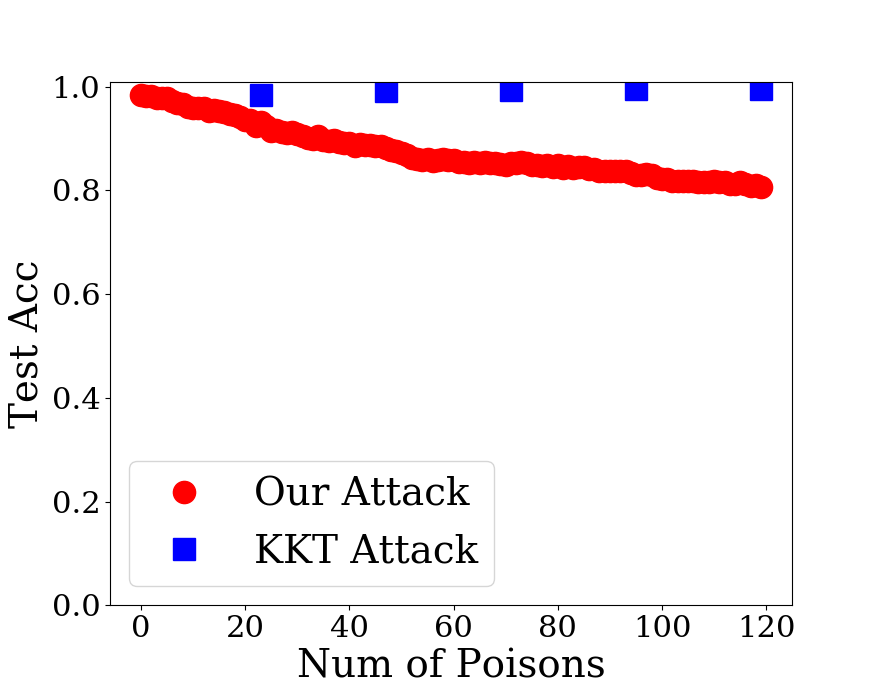}
            \caption[]%
            {20\% Error Rate}    
            \label{fig:dogfish_lr_02_acc}
        \end{subfigure}
        \begin{subfigure}[b]{0.33\textwidth}  
            \centering 
            \includegraphics[width=0.98\textwidth]{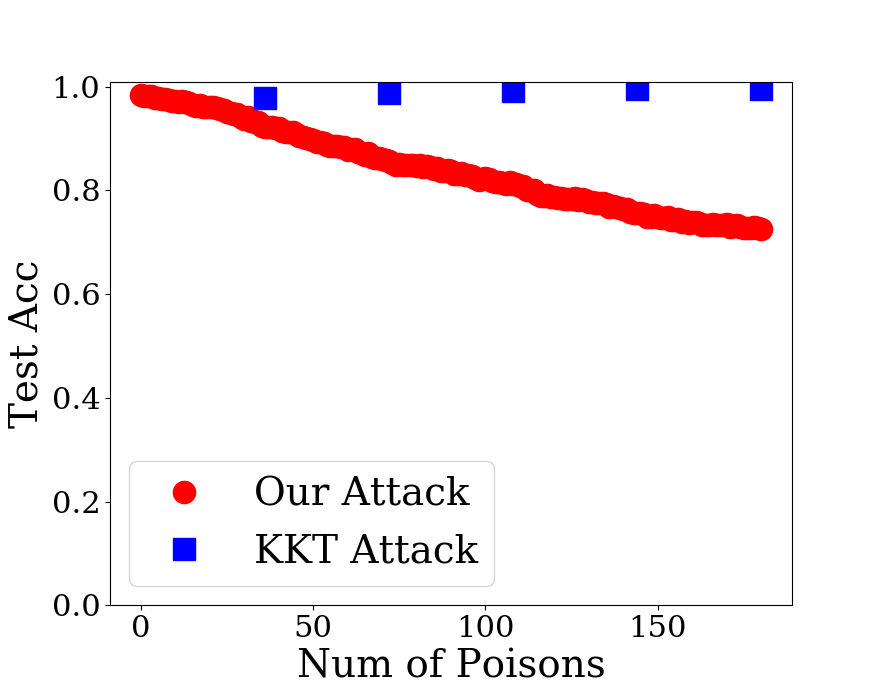}
            \caption[]%
            {30\% Error Rate}    
            \label{fig:dogfish_lr_03_acc}
        \end{subfigure}
        %\vskip\baselineskip
        %\quad
        \caption[]
        {Logistic regression model on Dogfish: test accuracy of each target model of given error rate with classifiers induced by poisoning points obtained from our attack and the KKT attack.} 
        \label{fig:dogfish_lr_acc}
    \end{figure*}

\subsection{Improved Target Generation Process}\label{ssec:bettertarget} 
The original heuristic approach in~\citet{koh2018stronger} works by finding different quantiles of training points that have higher loss on the clean model, flipping their labels, repeating those points for multiple copies, and adding them to the clean training set. We find that, in the process of trying different quantiles and copies of high loss points, if we also adaptively update the model where the high loss points are found (instead of just always fixing it to be the clean model), we can generate a target classifier that still satisfies the attack objective but with much lower loss on the clean training. Such an improved generation process can significantly reduce the number of poisoning points needed to reach the same $\epsilon$-closeness (with respect to the loss-based distance) to the target classifier, consistent with the claims in Theorem~\ref{theorem:convergence_main} in the main paper. In addition, we find that, if we compare our attack with improved generation process to the KKT attack with the original generation process~\citep{koh2018stronger}, we can also reach the desired target error rate much faster using our attack. 

    \begin{table}[tbp]
    \centering
    \begin{tabular}{c|cc|cc|cc}
    \toprule
    \multirow{2}{*}{Target Models} & \multicolumn{2}{c|}{Test Acc (\%)} & \multicolumn{2}{c|}{Loss on Clean Set} & \multicolumn{2}{c}{\# of Poisons} \\
     & Original & Improved & Original & Improved & Original & Improved \\ \midrule
    5\% Error & 94.0 & 94.9 & 2254.6 & 1767.1 & 2170 & 1340 \\ 
    10\% Error & 88.8 & 88.9 & 4941.0 & 3233.1 & 5810 & 2432 \\ 
    15\% Error & 83.3 & 84.5 & 5428.4 & 4641.6 & 6762 & 3206 \\ \bottomrule
    \end{tabular}
    \vspace{0.5em}
    \caption[]
    {SVM on \MNIST: comparison of two target generation methods on number of poisoning points used to reach $0.1$-closeness to the target. {\it Original} indicates the original target generation process from~\citet{koh2018stronger}. {\it Improved} denotes our improved target generation process with adaptive model updating.}
    \label{tab:target_gen_compare}
    \end{table}
    
    \begin{figure*}[!tb]
        \centering
        \begin{subfigure}[b]{0.33\textwidth}
            \centering
            \includegraphics[width=0.98\textwidth]{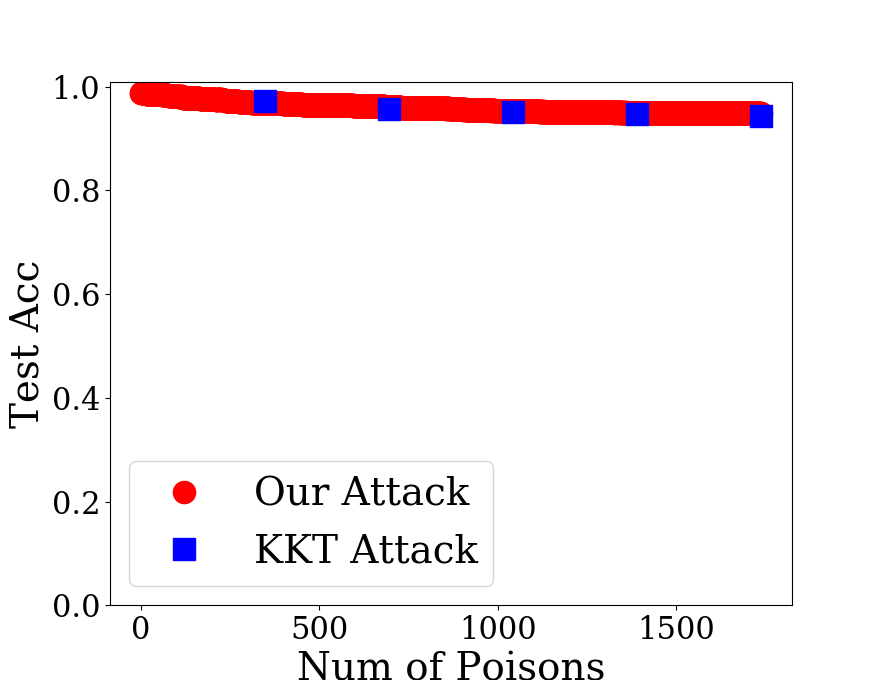}
           \caption[]
            {5\% Error Rate} 
            \label{fig:error_005_acc_scores_compare}
        \end{subfigure}
        %\hfill
        \begin{subfigure}[b]{0.33\textwidth}  
            \centering 
            \includegraphics[width=0.98\textwidth]{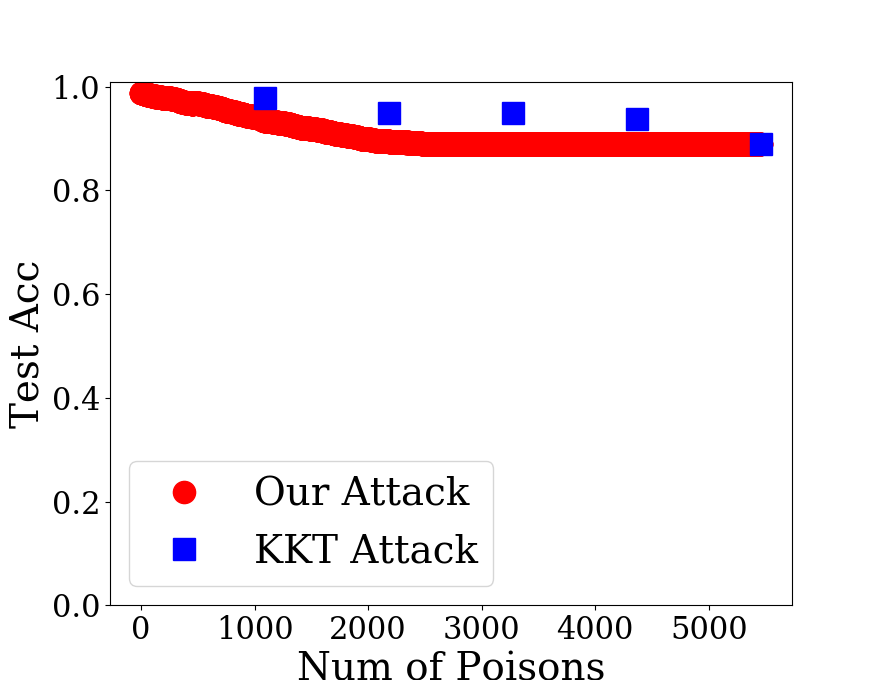}
            \caption[]%
            {10\% Error Rate}    
            \label{fig:error_01_acc_scores_compare}
        \end{subfigure}
        \begin{subfigure}[b]{0.33\textwidth}  
            \centering 
            \includegraphics[width=0.98\textwidth]{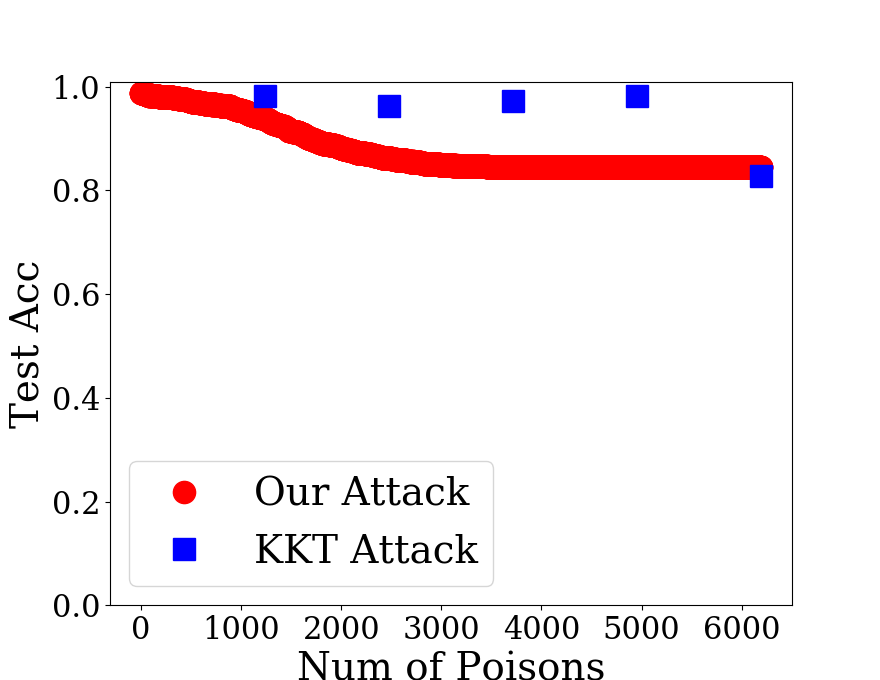}
            \caption[]%
            {15\% Error Rate}    
            \label{fig:error_015_acc_scores_compare}
        \end{subfigure}

        \caption[]
        {SVM on \MNIST: test accuracy with classifiers obtained from our attack and KKT attack. Target model for KKT attack is generated from the original generation process and target model for our attack is generated from the improved generation process. Maximum number of poisoning points is obtained by running our attack with target model generated from the original process and resultant classifier is $0.1$-close to the target.} 
        \label{fig:indiscriminate_acc_compare}
    \end{figure*}

\shortsection{Implication of Theorem~\ref{theorem:convergence_main}} 
We first empirically validate the implication of Theorem~\ref{theorem:convergence_main} in the main paper: to obtain the same $\epsilon$-closeness in loss-based distance, a target classifier with lower loss on the clean training set $\cD_c$ requires fewer poisoning points. Therefore, when adversaries have multiple target classifiers that satisfy the attack goal, the one with lower loss on clean training set is preferred. 
%This also reduces the risk of detection by the model owner.

%Our improved target generation method can produce target classifiers with similar error rate with smaller loss on the clean training set compared to the original method from~\citet{koh2018stronger}.  
We run experiments on the SVM and the \MNIST\ dataset. For both the original and improved target generation methods, we generate three target classifiers with error rates of 5\%, 10\% and 15\%. The original target classifier generation method returns classifiers with test accuracy of 94.0\%, 88.8\% and 82.3\% respectively (also used in the previous experiments on indiscriminate attack). The improved target generation process returns target classifiers with approximately the same test accuracy (94.9\%, 88.9\% and 84.5\%). However, for classifiers of same error rate returned from the two target generation processes, the improved generation method produces classifiers with significantly lower loss compared to the original one. 

Table~\ref{tab:target_gen_compare} compares the two target generation approaches by showing the number of poisoning points needed to get $0.1$-close to the corresponding target model of same error rate. For example, for target models of 15\% error rate, the model from the original approach has a total clean loss of 5428.4 while our improved method reduces it to 4641.6. With the reduced clean loss, getting $0.1$-close to the target model generated from our improved process only requires 3206 poisoning points, while reaching the same distance from the target model produced by the original method would require 6762 poisoning points, a more than 50\% reduction. 

\shortsection{End-to-End Comparison} Figure~\ref{fig:indiscriminate_acc_compare} compares the two attacks in an end-to-end manner in terms of their attack success (show as the overall test accuracy after poisoning). With the improved target generation process, our attack can achieve the desired error rate much faster than the KKT attack with the original process. For the KKT attack with target model generated from the original process, we determine the target number of poisoning points by running our attack with $0.1$-closeness as the stopping criteria and the model generated from the original process as the target classifier. To run our attack with improved generation process, we terminate the algorithm when the size of the poisoning points is same as the number of poisoning points used by the KKT attack with original process. Such a termination criteria helps us to ensure that both attacks use same number of poisoning points and can be compared easily. We also evaluate the KKT attack on fractions of the maximum target number of poisoning points (0.2, 0.4, 0.6, and 0.8), as in the previous experiments. The accuracy plot shows that our attack (with improved target model) can achieve the desired error rate (e.g., 10\% and 15\%) much faster than the KKT attack (with original target model). For example, for the attacker objective of having 15\% error rate, with target classifier of error rate of 15\% error, our attack can achieve the attacker goal much faster than the KKT attack. 

\section{Comparison of Model-Targeted and Objective-Driven Attacks}
\label{sec:model_vs_objective}

Although model-targeted attacks work to induce the given target classifiers by generating poisoning points, the end goal is still to achieve the attacker objectives encoded in the target models. In terms of the comparison to the objective-driven attacks, we first demonstrate that objective-driven attacks can be used to generate a target model, which can then be used as the target for a model-targeted attack, resulting in an attack that achieves the desired attacker objective with fewer poisoning points. Then, we show that to have competitive performance against state-of-the-art objective-driven attacks (e.g., the min-max attack~\citep{steinhardt2017certified}), the target classifiers should be generated carefully, such that the attacker objectives of the target classifiers can be achieved efficiently with model-targeted attacks using fewer poisoning points. Although the investigation of a systematic approach to generate such ``desired'' classifiers is out of the scope of this paper, in the indiscriminate setting, we have some empirical evidence. Specifically, we find that target classifiers with a lower loss on the clean training set and higher error rates (higher than what are desired in the attacker objectives) often require fewer poisoning points to achieve the attacker objectives. The following experiments are conducted on the \MNIST\ dataset.

\shortsection{Target Models Generated from Objective-driven Label-Flipping Attacks} In our experiments, the target classifiers are generated from the label-flipping based objective-driven attacks that are effective but need too many poisoning points to achieve their objective. Then, our attacks are deployed to achieve the same objective with fewer poisoning points. Table~\ref{tab:label-flip-results} shows the number of poisoning points used by the label-flipping attack described in \citet{koh2018stronger} and our model-targeted attack, to achieve desired attack objectives of increasing the test error to a certain amount. We can see that using our attack, the number of poisoning points used by label-flipping attacks can be saved up to 73\%. 

\begin{table}[t]
\centering
\begin{tabular}{cccc}
\toprule
Attacker Objectives & 5\% Error & 10\% Error & 15\% Error \\ \midrule
Label-flipping Attack & 6,510 & 8,648 & 10,825 \\ % \hline
Our Attack & 1,737 & 5,458 & 6,192\\
\bottomrule
\end{tabular}
\caption{Generate target classifiers using objective-driven label-flipping attacks and achieve similar attacker objectives using our attack with fewer poisoning points. The attacker objectives are to increase the test error to certain amounts (i.e., 5\%, 10\% and 15\%) and the target classifiers to our attack are generated by running the label-flipping attacks with given attacker objectives.}
\label{tab:label-flip-results}
\end{table}

\shortsection{Comparison to Objective-driven Attacks}
Still using target classifiers generated from label-flipping attacks, we show that our attack can outperform existing objective-driven attacks (including the state-of-the-art min-max attack~\citep{steinhardt2017certified}) at reducing the overall test accuracy, under the same amount of poisoning points. We still experiment on the SVM model and the \MNIST\ dataset.
%The original min-max attack is evaluated by reporting the overall test accuracies at different poisoning ratios. Therefore, we compare our attack to the min-max attack by reporting the resulting test accuracy at 5\%, 15\% and 30\% poisoning ratio. 
Since we aim to produce target classifiers with lower loss on clean training set and higher error rates, we adopt the improved target model generation process described in Section~\ref{ssec:bettertarget} (helps to reduce the loss on clean training set) and generate a classifier of 15\% error rate. With the target model, we terminate our attack when a fixed number of poisoning points are generated, and then compare the attack effectiveness to existing objective-driven attacks under same number of poisoning points. We compare the test accuracies of all attacks at poisoning ratios of 5\%, 15\% and 30\%. We also modified the baseline objective-driven attacks slightly for a fair comparison:
\begin{enumerate}
    \item The min-max attack~\citep{steinhardt2017certified} and the gradient attack~\citep{koh2017understanding} consider evading defenses during the attack process, which degrades their effectiveness. We simply remove those defenses in our evaluation.
    \item Since the generated poisoning points should be valid normalized images in [0,1] range (need not be semantically meaningful), we clip their generated poisoning points into the [0,1] range.
    \item The attacks by \citet{biggio2011support,demontis2019adversarial} use validation data to compute gradients. However, our splits only contain training and testing data, To avoid leaking test-data information or using gradients from data already used to train the model, we create a 70:30 train-validation split using the original training data: this new 70\% of the training data is used while the adversary trains its models, and the remaining 30\% is used as validation data for gradient computations. The victim then trains the model on the mixture of the original (100\%) training data and the generated poisoning points.
\end{enumerate}
% 1) for the min-max attack~\citep{steinhardt2017certified} and the gradient attack by~\citet{koh2017understanding}, both attacks consider evading defenses during the attack process, which degrades their effectiveness and we simply remove those defenses in evaluation; 2) since the generated poisoning points should be valid normalized images in [0,1] range (need not be semantically meaningful), for all attacks, we clip their generated poisoning points into the [0,1] range if they are initially some invalid images.
We note that the gradient attacks in~\citet{biggio2011support,demontis2019adversarial} are extremely slow to run on \MNIST~dataset (when we use the full training set) because the poisoning points are generated sequentially and the computational cost in each step of generation is very high. Therefore, we choose to improve the attack efficiency by repeating each generated poisoning point $N$ times and produce the desired number of poisoning points faster. We set $N=10$ for the attack on Logistic regression by \citet{demontis2019adversarial} and $N=100$ for the attack on SVM by \citet{biggio2011support} (still took 3 days to finish on the linear SVM model). For the attack by \citet{demontis2019adversarial}, we compared the setting of $N=10$ to the default setting of $N=1$ for poisoning ratios of 5\% and 10\% \footnote{We did not compare $N=1$ and $N=10$ for larger poisoning ratios because the $N=1$ case will take too long to finish.}, and did not find a significant degradation in the attack effectiveness when $N=10$ (the attack effectiveness drops by 0.3\% at most). This might be explained by the size of the training dataset: the impact of just one data point in nearly 13,000 (and even more, once the poison data generation starts) might not vary significantly across iterations, which results in adding multiple copies of the same poison data to be a fair approximation while giving a significant speedup in the runtime. We did not repeat this comparison for the attack from \citet{biggio2011support} because running it for the case of $N=1$ is simply infeasible. 

The results are summarized in Table~\ref{tab:min-max-results}. From the table, we observe that, compared to the existing objective-driven attacks, our attack reduces more on the test accuracy under the same poisoning budget and the gap becomes larger when the poisoning ratio increases. At 5\% of poisoning ratio, our attack outperforms all baseline attacks by at least 0.8\% in terms of the reduced test accuracy, and this gap increases to at least 8.6\% at 30\% of poisoning ratio.

\begin{table}[tb]
\centering
\begin{tabular}{cccc}
\toprule
Attack/Model & 5\% Poison Ratio & 15\% Poison Ratio & 30\% Poison Ratio \\ \midrule
Min-Max Attack~\cite{steinhardt2017certified}/SVM & 97.0\% & 93.9\% & 92.9\% \\ % \hline
\citet{biggio2011support}/SVM & 98.7\% & 98.2\%  & 96.8\%  \\
% 98.7\%, 98.2\%, 96.8\%
%98.9\% & 98.7\%  & 98.4\%
\citet{koh2017understanding}/SVM & 98.7\% & 98.0\%  & 97.2\%  \\ 
\textbf{Our Attack/SVM} & $\mathbf{96.2\%}$ & $\mathbf{88.6\%}$ & $\mathbf{84.3\%}$ \\ \midrule
\citet{demontis2019adversarial}/LR & 98.2\% & 97.6\%  & 95.7\% \\
\textbf{Our Attack /LR} & $\mathbf{96.5\%}$ & $\mathbf{89.1\%}$  & $\mathbf{83.1\%}$  \\
\bottomrule
\end{tabular}
\caption{Comparison of our attack to objective-driven attacks with different poisoning ratios. The target model of our attack is of 15\% error rate. The poisoning ratio is with respect to the full training set size of 13,007. Each cell in the table denotes the test accuracy of the classifier after poisoning. The clean test accuracies of SVM and LR models are 98.9\% and 99.1\% respectively.}
\label{tab:min-max-results}
\end{table}

% \begin{tabular}{l|ccc}
% \toprule
% Attack/Model   & 5\%Acc & 15\%Acc & 30\%Acc \\ \midrule
% \textbf{Ours/SVM}        & $\mathbf{96.2\%}$ & $\mathbf{88.6\%}$  & $\mathbf{84.3\%}$  \\
% Biggio/SVM     & 98.7\% & 98.2\%  & 96.8\%  \\
% Koh\&Liang/SVM & 98.7\% & 98.0\%  & 97.2\%  \\ \midrule
% \textbf{Ours/LR}         & $\mathbf{96.5\%}$ & $\mathbf{89.1\%}$  & $\mathbf{83.1\%}$  \\
% Demontis/LR    & 98.2\% & 97.6\%  & 95.7\% \\
% \bottomrule
% \end{tabular}

\begin{table}[bt]
\centering
\begin{tabular}{llllrrr}
\toprule
 & \multicolumn{3}{c}{SVM} & \multicolumn{3}{c}{Logistic Regression} \\
 & Cluster 0 & Cluster 1 & Cluster 2 & Cluster 0 & Cluster 1 & Cluster 2 \\\midrule
\textbf{Our Attack} & $\mathbf{0.0\%}$ & $\mathbf{0.0\%}$ & $\mathbf{0.0\%}$ & $\mathbf{0.0\%}$ & $\mathbf{0.0\%}$ & $\mathbf{0.0\%}$ \\
Label-Fipping & 31.4\% & 2.8\% & 15.5\% & 15.9\% & 14.0\% & 19.1\% \\
\bottomrule
\end{tabular}
\caption{Comparison of our attack to the label-flipping based subpopulation attack. The table compares the test accuracy on subpoplation of Adult dataset under same number of poisning points. The number of poisons are determined when our attack achieves 0\% test accuracy on the subpopulation. Cluster 0-3 in the logistic regression and SVM models denote different clusters. For logistic regression, number of poisoning points for Cluster 0-3 are 1,575, 1,336 and 1,649 respectively. For SVM, number of poisoning points for Cluster 0-3 are 1,252, 1,268 and 1,179 respectively.}
\label{tab:subpop-baseline-results}
\end{table}

\shortsection{Comparison to the Label-Flipping Subpopulation Attack} We also compare our attack to the label-flipping subpopulation attack from \citet{jagielski2019subpop}. This attack works by randomly sampling fixed number (constrained by the poisoning budget) of instances from the training data of the subpopulation, flipping their labels and then injecting them into to the original training set. Although this attack is very simple, it shows relatively high attack success when the goal is to cause misclassification on the selected subpopulation~\citep{jagielski2019subpop}. 

To be consistent with our experiments in Section~\ref{sec:experiments}, we assume the attacker objectives are still to induce a model that has 0\% accuracy on a selected subpopulation. For each of the SVM and logistic regression models, we selected the three subpopulations with highest test accuracy (all end up having have 100\% accuracy). 
%We then briefly introduce our target model selection process that aims to induce the given attacker objective with fewer poisoning points. 
In indiscriminate setting, we already observed that models with lower loss on clean training set and larger overall error rates can achieve attacker objectives of smaller error rates faster. However, to leverage this observation into our subpopulation experiments, one challenge is the attacker objective is to have 100\% test error on the subpopulation, but no classifiers can have test errors larger than 100\%. To tackle this, we select models with larger loss on training samples from the subpopulation, with a hope that this process is ``equivalent'' to selecting target models with larger error rates (on subpopulation) than 100\%. To this end, we heuristically select targeted models that satisfy the attacker objective, have larger loss on the training data from the subpopulation, and have relatively low loss on the entire clean training set. Empirically, this selection strategy works better than the original target generation process (as done in Section~\ref{sec:experiments}) in achieving the attacker objectives. A more detailed and systematic investigation of the target model search process is left as the future work. 

To check the effectiveness of achieving the attacker objectives, we first run our attack and terminate when our attack achieves the attacker objective to have 0\% accuracy on the selected subpopulation, and record the number of poisoning points used. Then, we run the random label-flipping attack with the same number of poisoning points. For both attacks, we report the final test accuracies of the resulting models on the subpopulations. 

The attack comparisons on different subpopulation clusters and models are given in Table~\ref{tab:subpop-baseline-results}. Results in the table compare our attack and the label-flipping attack over the three distinct subpopulation clusters for the SVM and logistic regression models. Across all settings, our attack is considerably more successful. The number of poisoning points needed to reach the 0\% accuracy goal is small compared to the entire training set size (e.g., the maximum poisoning ratio is only 10.5\%). The gap between our attack and the label-flipping attack is fairly small.  For example, for Cluster 1 in the SVM experiment, the label-flipping attack is also quite successful and reduces the test accuracy to 2.8\% (our attack achieves 0\% accuracy). We believe the success of label-flipping attack is due to the following two reasons. First, label-flipping in the subpopulation setting can be successful because smaller subpopulations show some degree of locality and hence, injecting points (from the subpopulation) with flipped labels can have a strong impact on the selected subpopulation. This is confirmed by empirical evidence that increasing the subpopulation size (i.e., reducing its locality) gradually reduces the label-flipping effectiveness and the attack becomes much less effective in the indiscriminate setting (i.e., subpopulation is the entire population). Second, the Adult dataset only contains 57 features, where 53 of them are binary features with additional constraints. Therefore, the benefit from optimizing the feature values is less significant as the optimization search space of our attack is fairly limited. 

% We believe our attack can outperform the label-flipping attack more in the subpopulation setting by choosing better target models for the attacker objectives and also testing on datasets with larger search space for the attack optimization.

% moved to AF_DNNs.tex
\section{Attacks on Deep Neural Networks}
\label{sec:dnn_results}

The theoretical guarantees of our proposed algorithm require convexity of the model loss. They do not hold for non-convex models such as deep neural networks (DNN). However, we hypothesize that our method of picking poisoning points incrementally might still perform well on non-convex models. Here, we report some preliminary results attacking DNNs.

Several poisoning attacks have been proposed for DNNs. However, all attacks with publicly available source code focus on causing misclassification for a given single instance~\citep{shafahi2018poison,zhu2019transferable,huang2020metapoison,geiping2020witches}, while we are more interested in the practical sub-population setting. Therefore, we use a random label-flipping attack as our baseline in the sub-population setting. The label-flipping attack selects a random image from the dataset in the targeted sub-population class without replacement and changes its label.

For these experiments, we use the \MNIST\ dataset and conduct sub-population poisoning attacks where the targeted sub-population is the class {\sf 1} (so, the adversary's goal is to have test images that would be correctly classified as {\sf 1} digits, classified as {\sf 7}s).  We compare the poisoning effectiveness in reducing the classification accuracy for the {\sf 1} class of our algorithm to the random label-flipping attack.
We implement our attack for DNNs with the cross-entropy loss as our loss function. We conduct experiments poisoning a non-convex three-layer neural network with non-linear activation functions, a multilayer perceptron (MLP).\footnote{We also tested our attack on convolutional neural networks with all modifications to our attack (described below) that work well on the MLP models. However, its performance is unstable and exhibits erratic accuracy curves despite the smooth loss convergence. We leave exploring loss functions and tuning the attack to make it work for convolutional neural networks as part of future work.}

We observe that direct implementations of these poisoning attacks do not work. Via experiments to understand the kind of randomness introduced by varying hyper-parameters like batch-size and weight-initialization, we take steps to remove these sources of variation (Section~\ref{sec:simplify_dnn_setting}). Then, we describe modifications to our attack to make it work on DNNs (Section~\ref{sec:modify-to-dnn}). Finally, we show results with the modified attack and how it performs well when one of these assumptions is relaxed (Section~\ref{sec:dnn-specific-results}), giving us some hope for the possibility of relaxing other assumptions as well.

\begin{figure*}[tbh!]
    \centering
    \begin{subfigure}[b]{0.45\textwidth}
        \centering
        \includegraphics[width=\textwidth]{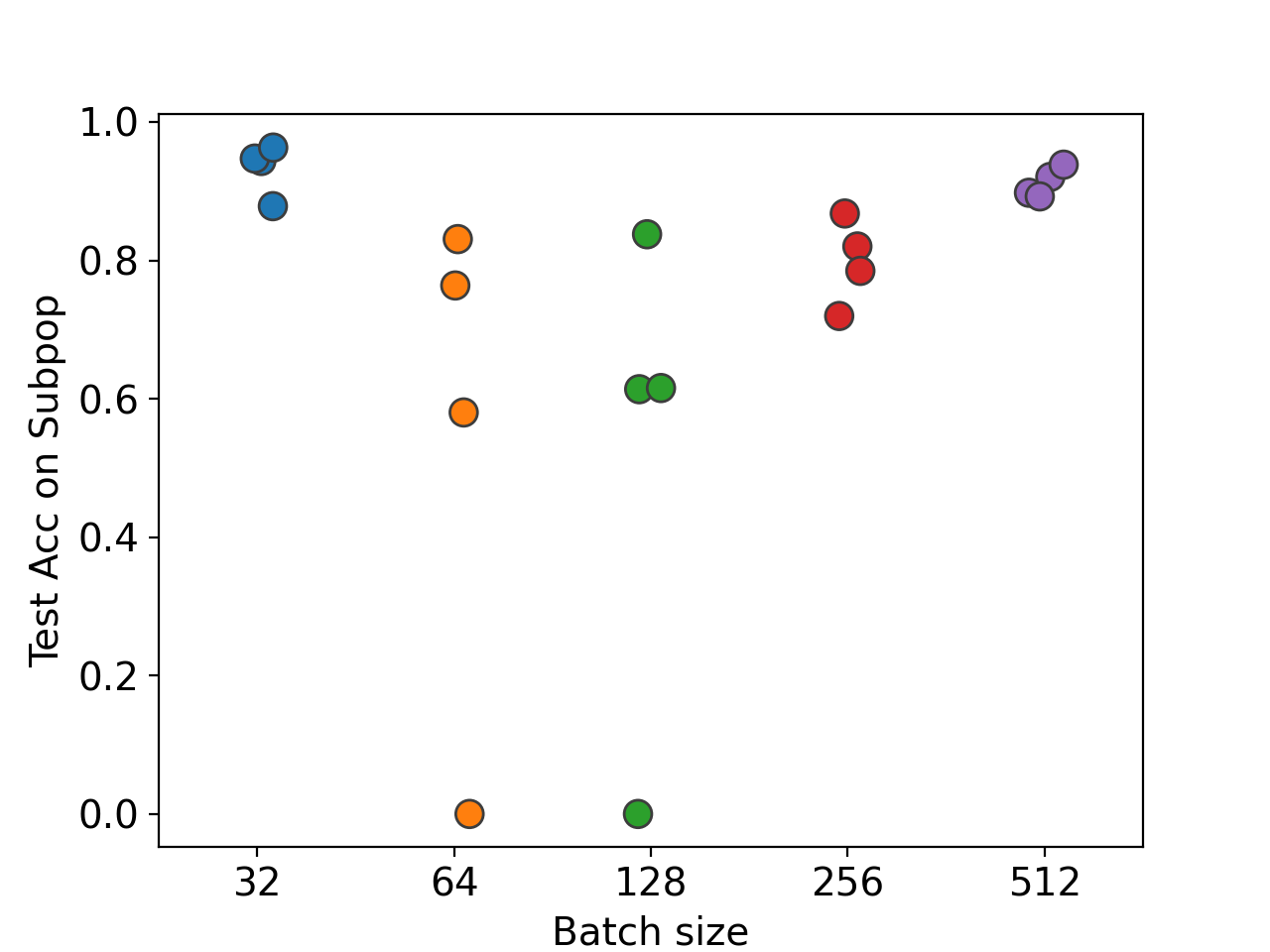}
          \caption[]
        {Fixed Model Initialization, Varying Batch-sizes.}
        \label{fig:vary_batch_size}
    \end{subfigure}
    \begin{subfigure}[b]{0.45\textwidth} 
        \centering 
        \includegraphics[width=\textwidth]{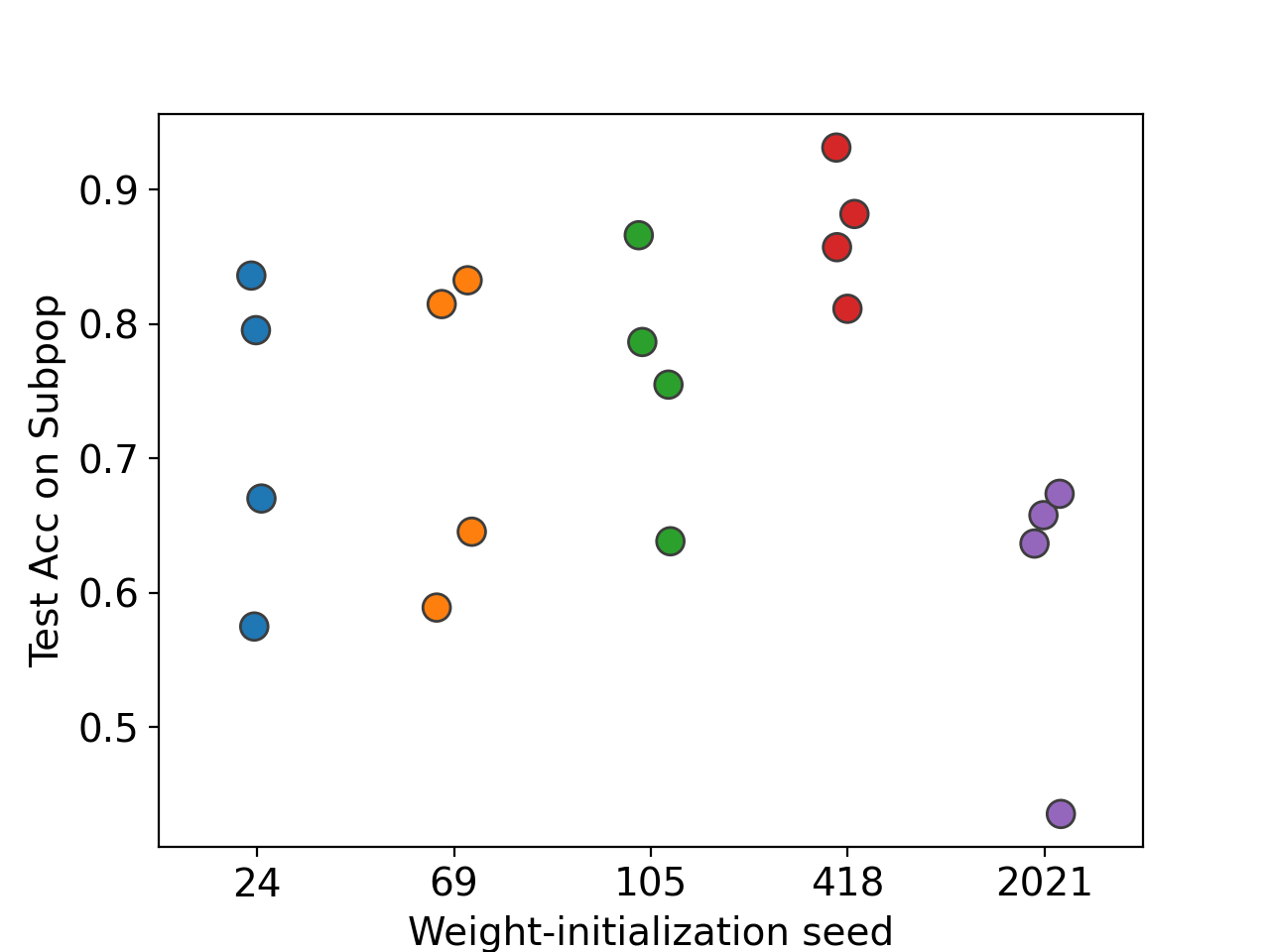}
        \caption[]%
        {Fixed Batch-size, Varying Model Initialization.}
        \label{fig:vary_model_init}
    \end{subfigure}
    \caption[]
    {Variance of Poisoning Attacks. Each figure shows the test accuracy on sub-population for the label-flipping attack (same poisoning ratio: 0.5) on \MNIST, with varying hyper-parameters of batch-size and model-weight initializations. Each setting is repeated four times, and each dot shows the result for one run. The attack is highly unstable, with large variations even when everything other than either the batch size or the weight initializations is changed.} 
    \label{fig:label_flip_unstable}
\end{figure*} 

\begin{figure*}[tbh!]
    \centering
    \begin{subfigure}[b]{0.45\textwidth}
        \centering
        \includegraphics[width=\textwidth]{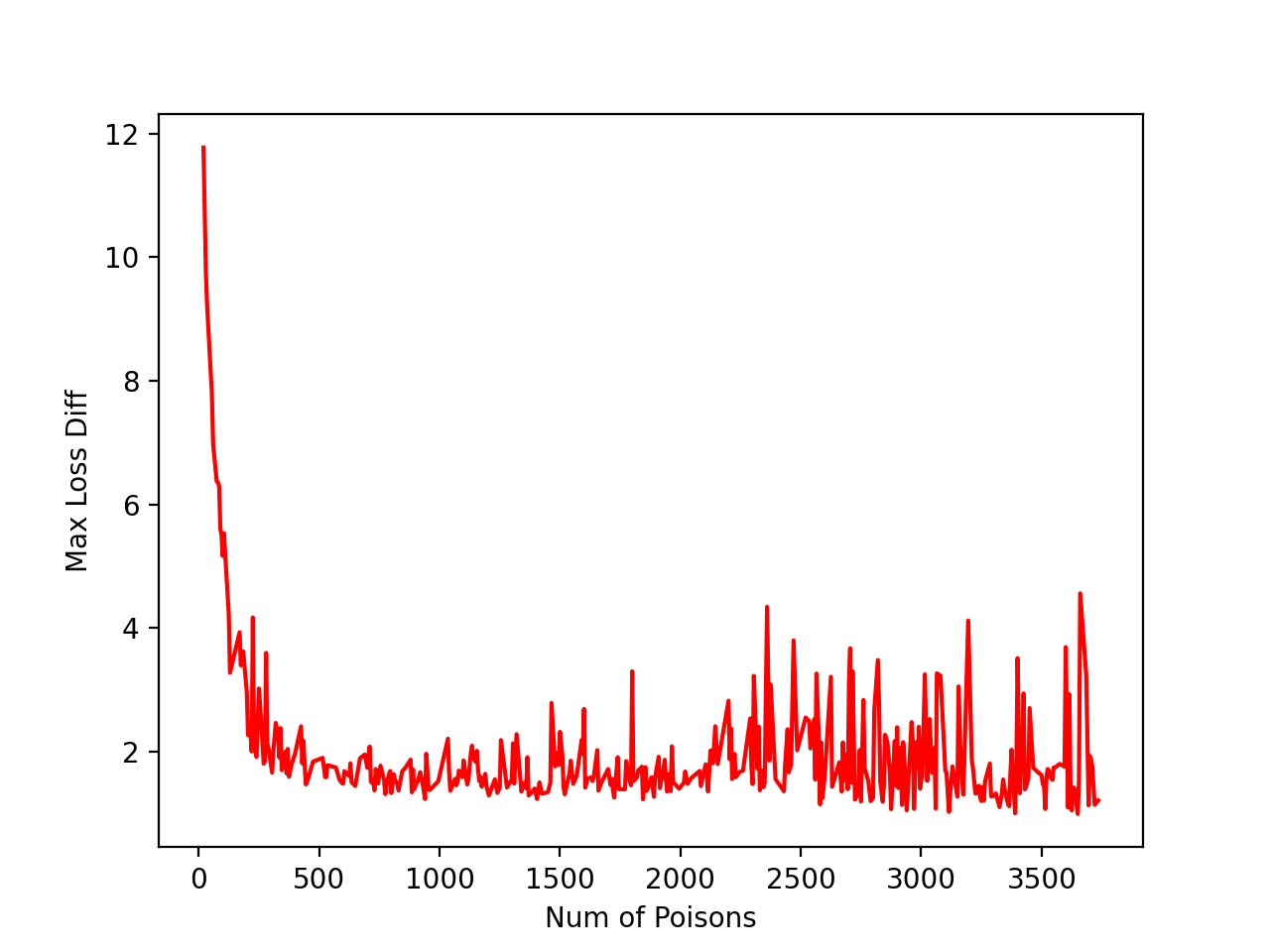}
          \caption[]
        {Maximum Loss Difference}
    \end{subfigure}
    %\hfill
    \begin{subfigure}[b]{0.45\textwidth} 
        \centering 
        \includegraphics[width=\textwidth]{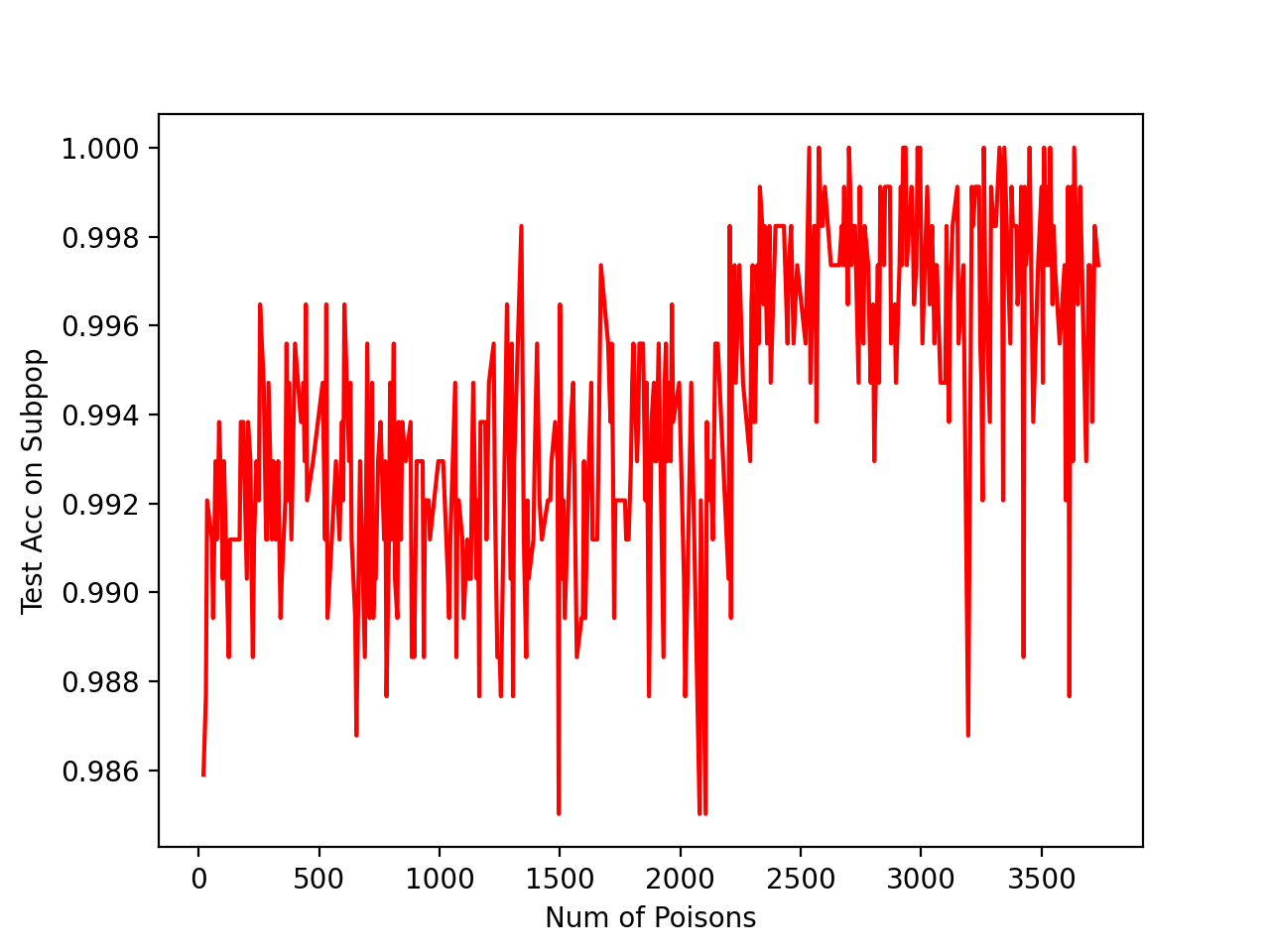}
        \caption[]%
        {Test Accuracy on Target Sub-population}
    \end{subfigure}
    \caption[]
    {Maximum loss difference and test accuracy on target sub-population across iterations for our algorithm. Data is not batched, and same weight-initializations for $\theta_t, \theta_p$ are used. The loss drops sharply within the first few iterations, but the accuracy fluctuates within a very small window, even when $|\cD_p| \sim 0.5|\cD_c|$ is added.}
    \label{fig:mtp_sameseed}
\end{figure*} 

\subsection{Simplifying the Setting} \label{sec:simplify_dnn_setting}

Both our attack and the label-flipping attack on DNNs are observed to be highly sensitive to hyper-parameters like batch-size, weight-initialization, and even randomness induced by the ordering of batches across epochs. As shown in Figure~\ref{fig:vary_batch_size}, for the same initialization for model weights and batch-size, different runs of the label-flipping attack with the same poisoning ratio lead to wildly varying error rates. Figure~\ref{fig:vary_model_init} shows that, even when we only vary the model weight-initialization and keep other hyper-parameters fixed, attack effectiveness fluctuates significantly across different random weight initializations.

To better compare our attack with the label-flipping baseline reliably, we design our experiments by not batching the data (\textit{i.e.} batch-size is the same as dataset size). The target model $\theta_p$ is trained with the label-flipping attack with fixed weight-initialization and no batching. Additionally, we ensure that the weight-initialization used to generate the intermediate model $\theta_t$ in each iteration of our attack is the same as the weight-initialization to train the target model $\theta_p$. Using a different weight initialization in each round of retraining interferes with model convergence and leads to unstable results. This way, we can substantially eliminate randomness introduced by batching data and different model weight initializations.

\subsection{Modifying Attack for DNNs}
\label{sec:modify-to-dnn}

%\shortsection{Constrained Optimization}
Despite removing batching and setting the weight-initialization for $\theta_p$ and $\theta_t$ to be the same, we observe that our attack still fails. Even though the loss difference seems to converge, the model preserves its accuracy on the target sub-population; dropping by less than 2\% across the iterations even up to a poisoning rate of 0.55 (Figure~\ref{fig:mtp_sameseed}).
Although we do not understand what causes this behavior, we speculate that it is due to a disconnect in the attacker's objective and the loss-function used.

To mitigate this problem, we modify the algorithm to constrain the search space of possible poisoning points to a predefined set of candidates. By iterating over all the candidate points, the algorithm picks the most promising poisoning point (i.e., with maximum loss difference between $\theta_t$ and $\theta_p$) from this candidate set. To define the candidate set, we construct two non-overlapping, equal-sized stratified splits of the dataset. The first one is used for training purposes ($\cD_c$), while the second one is used as the candidate set for $(x^*, y^*)$ optimization.
We add an additional constraint on the candidate set that enforces the selection of points from the target sub-population but are assigned an incorrect label (i.e., the candidate set consists of digit 7, and the assigned labels are 1.).

% Adding this constraint helps with the stability of metrics like loss and weight-norm difference, but still fails to outperform the label-flipping attack. We find that for many iterations of the algorithm, poisoning points are selected where $\theta_t$ and $\theta_p$ already agree in their predictions, although their loss difference is highest. The motivation behind finding points with maximum loss difference with current training loss is to find inputs where $\theta_t$ and $\theta_p$ have different predictions, and $\theta_p$ is correct and confident prediction for the given $(x^*, y^*)$. Thus, we add an additional constraint on this selection process that enforces the selection of points that are from the target sub-population but are assigned an incorrect label (i.e., the candidate set consists of digit 7 and the assigned labels are 1.). In future, we will try different loss functions than cross entropy and check if the attack effectiveness can be improved.

% \dnote{okay - this seems more complex than necessary - if it is just a matter of eventually getting to the simple solution of always flipping the label from the true label, then I don't think we need this two-step story. If it was the case that even when the label is always flipped, it matters to pick points that the original model labels correctly, then that would be interesting.}

\subsection{Results}
\label{sec:dnn-specific-results}

% In the settings above, we assumed same weight-initialization while training $\theta_p$, $\theta_t$ in our attacks, as well as when the victim trains models on $D_p$. Although the first assumption is reasonable in practice since both $\theta_t$, $\theta_p$ are under the adversary's control, it is unlikely for the adversary to know the victim's weight-initialization in any realistic attack scenario. So, we consider a modification of the attack that uses models with different initializations to generate poisoning points that are more robust to initialization settings.

% \shortsection{Known Initialization}
We start with the case where the weight-initialization used by the victim to train its models is known to the adversary. This setting is unrealistic, but shows how effective the attack could be when the adversary has full knowledge of everything about the victim's training process, including the random seeds used. With the constraints described in Section~\ref{sec:modify-to-dnn}, our attack consistently outperforms the label-flip attack by a large margin, as shown in Figure~\ref{fig:mtp_final}. For these experiments, we observe similar convergence and attack success rates between adding just one copy of $(x^*, y^*)$ per iteration and adding as many as ten copies, and with at most ten copies, we can reduce attack execution time by nearly 90\%.

\begin{figure*}[tbp]
    \centering
    \begin{subfigure}[b]{0.33\textwidth}
        \centering
        \includegraphics[width=0.98\textwidth]{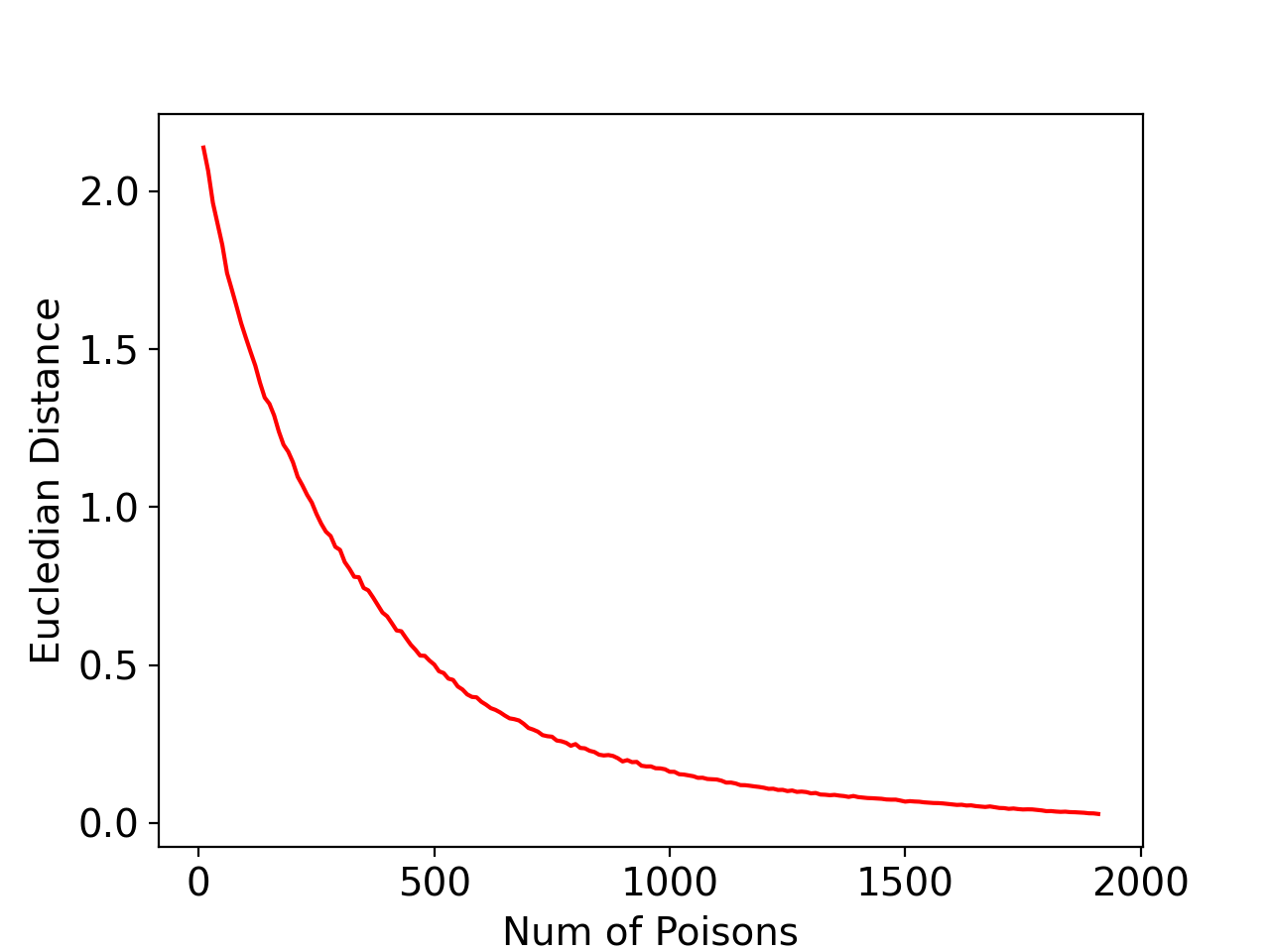}
          \caption[]
        {Maximum Loss Difference}
    \end{subfigure}
    \begin{subfigure}[b]{0.33\textwidth} 
        \centering 
        \includegraphics[width=0.98\textwidth]{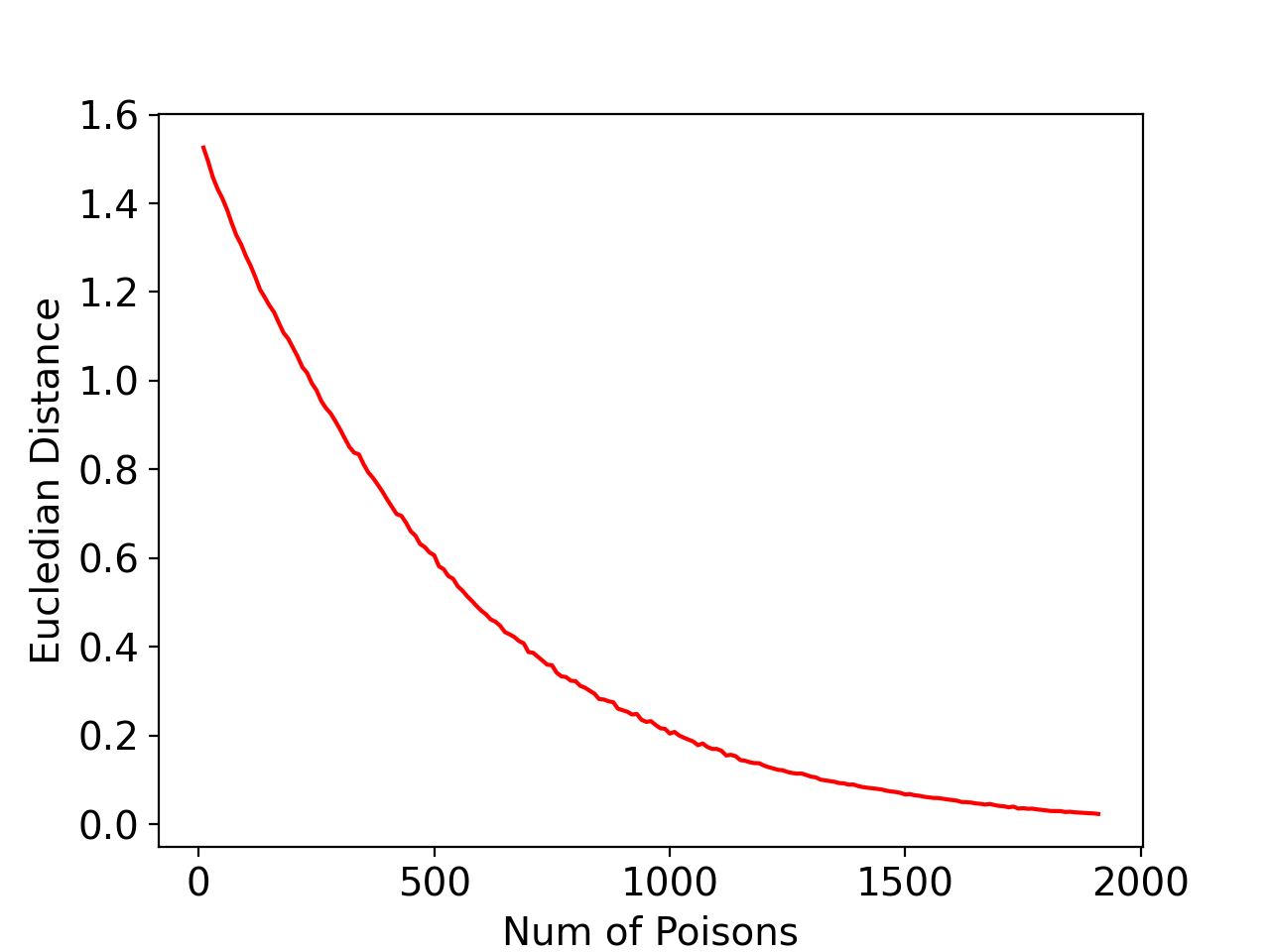}
        \caption[]%
        {Euclidean Distance}
    \end{subfigure}
    \begin{subfigure}[b]{0.33\textwidth} 
        \centering 
        \includegraphics[width=0.98\textwidth]{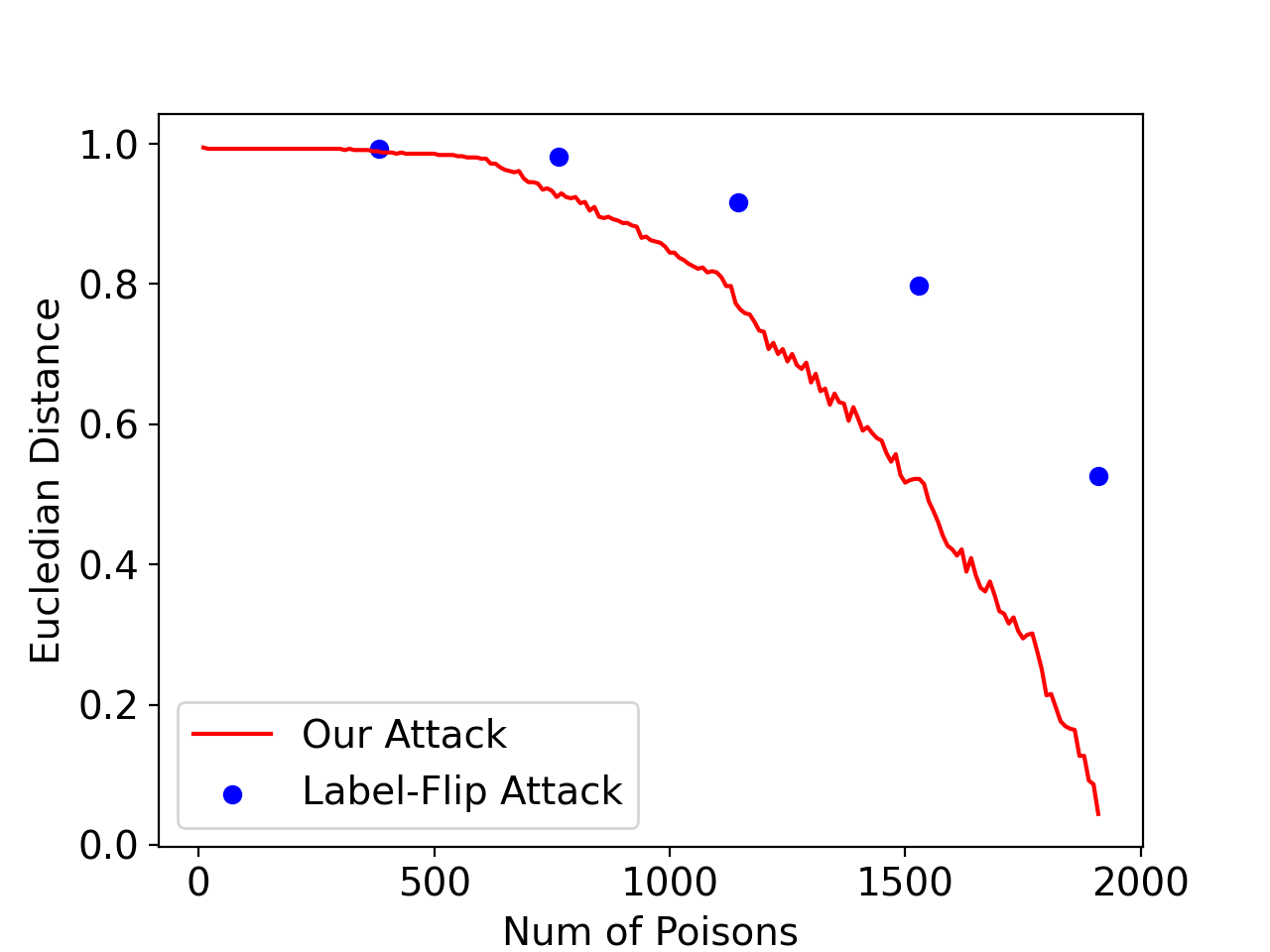}
        \caption[]%
        {Test Accuracy on Target Sub-population}
    \end{subfigure}
    %\vskip\baselineskip
    %\quad
    \caption[]
    {Maximum loss difference and test accuracy on target sub-population across iterations for our algorithm. The optimization process is constrained to select points from the candidate set. As visible, both the loss and Euclidean distance converges to zero smoothly. Additionally, our attack outperforms label-flip attack by a significant margin. We observed consistent results across several seeds.}
    \label{fig:mtp_final}
\end{figure*}

% \shortsection{Different Weight-Initializations} \label{sec:dnn-diff-seeds}

% Although both the label-flipping attack and model targeted poisoning perform great (with our method outperforming label-flipping attack) with known weight initialization, it is unrealistic to assume the initial model weights are known to adversaries in practice.
% Thus, it is crucial to see how well these generated $\cD_p$ will work when used to train models with different weight-initializations.

Next, we evaluate the attack in a more realistic setting where the adversary does not know the weight-initializations used in training the victim model. As shown in Figure~\ref{fig:mtp_unseen_seeds}, we observe a large variation in the performance of trained models across different initial model weights, and the attack is not as effective as it can be when the initialization is known. The variance in attack performance is because these models are unstable to varying weight-initializations (Figure~\ref{fig:vary_model_init}) --- some initial weights are biased towards having larger errors on the target sub-population, making it possible to poison these models with fewer points.
% For example, with known initial weights, the $\cD_p$ from our attack can induce a model with an error rate of more than 95\% on the target sub-population, but the error rate with the same $\cD_p$ varies between 70\% to 100\% when trained with different weight initialization.
Even in this setting, our model-targeted poisoning attack consistently outperforms the label-flipping attack.
% and is also relatively effective at increasing the error rates on the target subpopulation.

%\begin{figure*}[htb]
\begin{figure*}[tbp]
    \centering
    \includegraphics[width=0.45\textwidth]{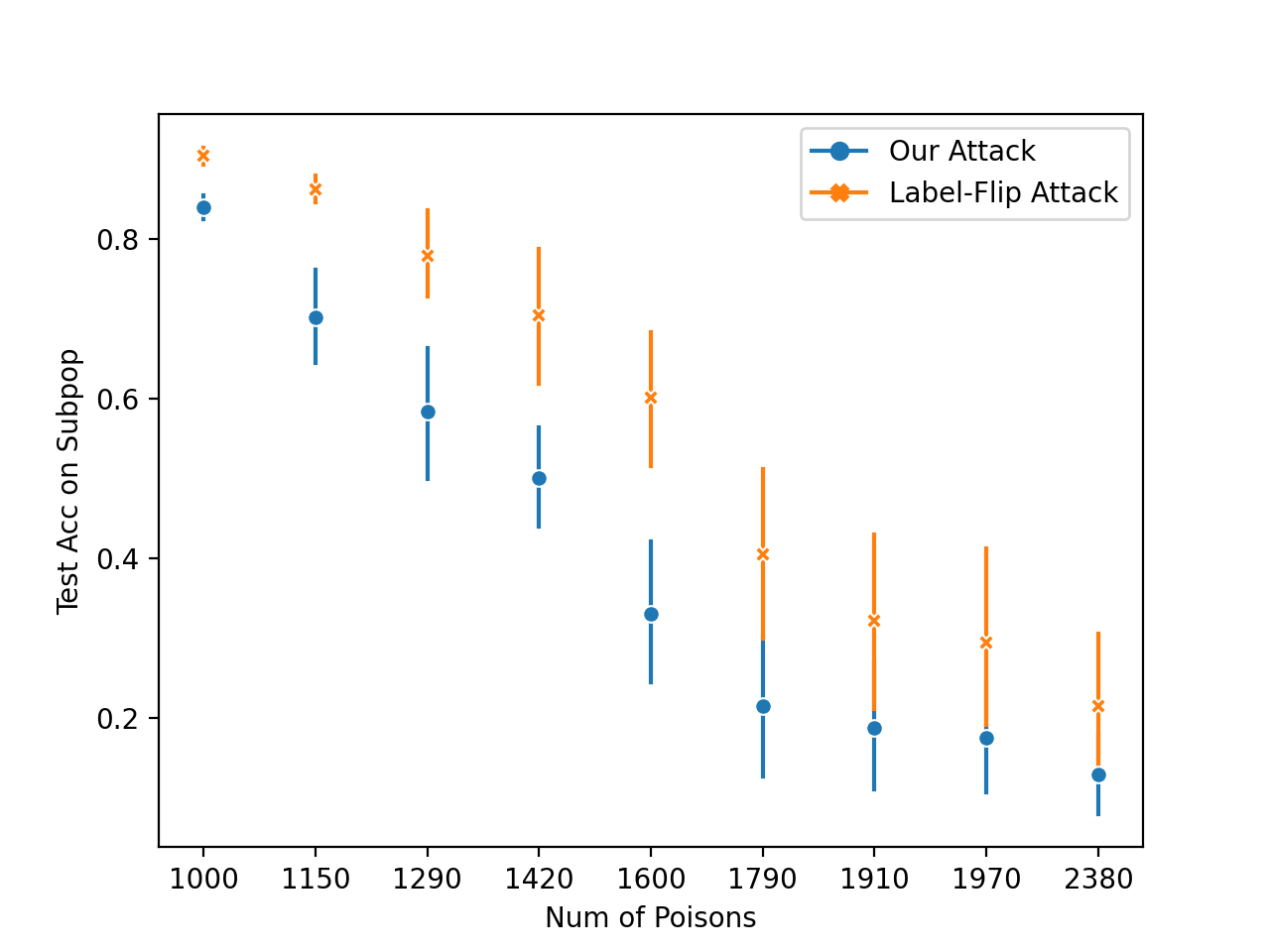}
    \caption[]
    {Poisoning attack effectiveness when the adversary does not know the victim's weight initializations (ten different seeds tried per experiment). The error bars show 68\% confidence intervals (standard error).}
    \label{fig:mtp_unseen_seeds}
\end{figure*}

\end{document}